\documentclass{article} % For LaTeX2e
\usepackage{iclr2027_conference,times}

\usepackage{amsmath,amsfonts,bm}

\def\eqref#1{equation~\ref{#1}}
\def\1{\bm{1}}

\DeclareMathAlphabet{\mathsfit}{\encodingdefault}{\sfdefault}{m}{sl}
\SetMathAlphabet{\mathsfit}{bold}{\encodingdefault}{\sfdefault}{bx}{n}

\usepackage{hyperref}
\usepackage{url}
\usepackage[T1]{fontenc}
\usepackage{booktabs}
\usepackage{amsfonts}
\usepackage{nicefrac}
\usepackage{microtype}
\usepackage{xcolor}
\usepackage{soul}
\usepackage{comment}
\usepackage{amsmath}
\usepackage{amsthm, amssymb}

\usepackage{enumitem}
\usepackage[framemethod=TikZ]{mdframed}
\usepackage{multirow}
\usepackage[most]{tcolorbox}
\usepackage{adjustbox}

\usepackage{thmtools}
\usepackage{thm-restate}

\newtheorem{definition}{Definition}
\newtheorem{theorem}{Theorem}
\newtheorem{lemma}{Lemma}
\newtheorem{corollary}{Corollary}
\newtheorem{remark}{Remark}

\title{On the effectiveness of reward functions in reinforcement learning for confidence calibration of large language models}

\author{Chee Heng Tan \\
School of Computing \\
National University of Singapore \\
\texttt{e0764286@u.nus.edu} \\
\And 
Zhuoyi Lin \\
Institute of Advanced Intelligence and Computing \\
Agency for Science, Technology and Research \\
\texttt{lin\_zhuoyi@a-star.edu.sg} \\
\And
Mehul Motani \\
Department of Electrical \& Computer Engineering \\ 
National University of Singapore  \\
\texttt{motani@nus.edu.sg} \\
\And
Wee Sun Lee \\
School of Computing \\
National University of Singapore \\
\texttt{leews@comp.nus.edu.sg}\\
}

\iclrpreprintcopy

\begin{document}
\begingroup

\setlength{\tabcolsep}{5pt}
\renewcommand{\arraystretch}{0.9}

\maketitle

\begin{abstract}
\begin{comment}
The abstract paragraph should be indented 1/2~inch (3~picas) on both left and
right-hand margins. Use 10~point type, with a vertical spacing of 11~points.
The word \textsc{Abstract} must be centered, in small caps, and in point size 12. Two
line spaces precede the abstract. The abstract must be limited to one
paragraph.
\end{comment}
In this paper, we consider the setting where large language models (LLMs) are trained using reinforcement learning (RL) to simultaneously improve reasoning accuracy and verbalize their confidence. Our reward scheme uses two functions for rewarding confidence verbalized by the LLM: one for correct answers and the other for incorrect answers. If poorly designed, such a scheme may incentivize an LLM to answer incorrectly in order for its confidence to be calibrated, a phenomenon we term confidence reward hacking. We introduce the notion of non-hackable confidence reward schemes and provide methods for constructing them. We show that selective confidence reward hacking can arise in practical datasets under hackable reward schemes while non-hackable reward schemes are resistant to hacking. Finally, we place some of these schemes along an overconfidence-underconfidence spectrum for RL-based confidence calibration and demonstrate experimentally that they tend to exhibit the corresponding calibration biases relative to other schemes in the spectrum. The code of our experiments is available in \url{https://anonymous.4open.science/r/rl-confidence-calibration-9ED4/README.md}.
\end{abstract}

\section{Introduction}

Large language models (LLM)s have made significant advances in reasoning, especially when using reinforcement learning (RL) \citep{zhang2025surveyreinforcementlearninglarge}. However, much of the RL literature has mainly focused on the RL with verifiable rewards setting (RLVR) \citep{wang2025arbitraryentropypolicyoptimization, wu2026takestwogrposecretly}, where the LLMs are usually rewarded for correctness. While useful, maximizing answer correctness can incentivize guessing, leading to high hallucination rates, as pointed out in \cite{kalai2025languagemodelshallucinate}.

Besides providing an answer, LLMs can also be asked to verbalize their confidence \citep{tian-etal-2023-just,xiong2024can, kapoor2024large}. Verbalized confidence can be used to decide whether the LLM is hallucinating, as well as in other applications such as event forecasting \citep{turtel2025outcomebased, murphy2026agentic} and determining whether the user is trying to factually deceive the LLM \citep{li-etal-2025-firm}. However, pretrained LLMs \citep{damani2026beyond, li2026consistencylargereasoningmodels} and LLMs trained with RL \citep{xiaohu2026knowyourewrongaligning, openai2024gpt4technicalreport} have been demonstrated to be overconfident. To be useful, the verbalized confidence should accurately match the correctness rate, i.e. be calibrated. 

Confidence calibration using RL shows potential as RL usually exhibits better out-of-distribution generalization capabilities \citep{chu2025sft} and lesser forgetting \citep{shenfeld2026rls}. RL calibration fine-tuning using proper scoring methods such as log loss \citep{bani-harouni2026rewarding} and Brier score \citep{xu-etal-2024-sayself} is possible, but doing so in isolation may improve calibration at the expense of accuracy \citep{ma2026decoupling}. Reward functions can reward a combination of accuracy and confidence calibration but may incentivize the LLM to answer incorrectly, which includes abstentions if they are considered wrong, so that its confidence is well-calibrated. We call this phenomenon confidence reward hacking and show experimentally that it can arise in practical datasets. 

In this paper, we define a reward scheme as a pair of confidence-dependent functions that respectively provide the rewards for answering correctly and answering incorrectly. We provide conditions for a reward scheme to be non-hackable confidence, in other words, incentivize both calibration and accuracy. Our results generalize those in \cite{damani2026beyond}, which proposed a provably non-hackable scheme that combines a constant correctness reward with  Brier score. A separate line of work \citep{wu2026mitigatingllmhallucinationbehaviorally, wu2026basdecisiontheoreticapproachevaluating} proposed designing reward schemes based on the behavioral calibration framework of \cite{kalai2025languagemodelshallucinate}, which provides 0 reward for abstaining, a fixed positive reward for answering correctly, and an abstention-threshold dependent penalty for answering incorrectly. We prove their construction yields non-hackable confidence reward schemes, unifying both lines of work.

Subsequently, we define a spectrum of reward schemes containing reward schemes where underconfidence is preferred and where overconfidence is preferred. Such schemes may be useful when there is a need to trade off accuracy with calibration, e.g. when the LLM has limited representational capacity or where the optimization process is difficult. We demonstrate experimentally that these schemes tend to result in the corresponding overconfidence and underconfidence calibration biases relative to other schemes in the spectrum.

\section{Related Work}\label{section:related-work} 

\subsection{RL for Confidence Calibration}

Various RL methods have been proposed for confidence calibration. In \cite{stengel-eskin2024lacie}, preference modeling is utilized to encourage the LLM to not believe as fact whenever unsure to mitigate hallucination risk, and empirically shown to encourage abstention whenever unsure. In \cite{ma2026decoupling}, the negative of the absolute difference between a smoothed accuracy measure and the confidence of the LLM is proposed as a reward function. Recently, a reward function is proposed that directly optimizes Area Under Risk-Coverage Curve (AURC) \citep{geifman2019biasreduceduncertaintyestimationdeep} to better differentiate between truth and hallucination \citep{luo2026aligning}. Neither method satisfies proper scoring. Therefore, the model is not necessarily incentivized to be well-calibrated in confidence.

\subsection{Theory of Proper Scoring}

The set of proper scoring functions of two possible outcomes have been thoroughly characterized \citep{gneiting2007strictly, schervish1989general, shuford1966admissible} with various representations \citep{savage1971elicitation}. However, the study of proper scoring methods that also incentivize accuracy is relatively underexplored. The most relevant work is \cite{shuford1966admissible}, which reformulated in the context of LLM question-answering, considers the setting where the LLM provides \(N\) answer options and their corresponding confidence estimates. The LLM obtains an additional constant reward for correctness if the correct answer is one of the provided options. Our work provides a comprehensive characterization for the case where \(N=1\) and the additional reward for correctness can depend on the output confidence.

\subsection{Proper Scoring Methods for Confidence Calibration}

Proper scoring methods have been used for confidence calibration in LLM. In \cite{bani-harouni2026rewarding}, a clipped log-loss reward function is used over the confidence scores of the LLM. As an extension to multi-sentence generation, \cite{zhang-etal-2026-lovec} proposes adding a confidence score to each sentence and awards a log-loss reward based on both the verbalized confidence and whether the sentence is judged to be factually correct. In \cite{xu-etal-2024-sayself}, RL with a Brier-based reward function is used after supervised finetuning (SFT) using empirical correctness rates. In \cite{li2025conftuner}, a discretized variant of the Brier loss is used as the SFT objective for confidence calibration. In these papers, accuracy is not optimized and is treated as secondary to confidence calibration.

\begin{comment}
In \cite{li2026confidenceansweringparadigmshift}, it is shown that two stage training, first for accuracy, then for confidence calibration results in reward hacking when using the standard Group Relative Policy Optimization (GRPO) \citep{shao2024deepseekmathpushinglimitsmathematical} loss. To address the issue, a combination of a correctness reward and a confidence alignment reward denoting how close the output confidence is to the question-wise empirical accuracy of the model is proposed. While experimentally successful, this work did not show theoretically that the reward functions are non-hackable. We provide mathematical foundations for understanding non-hackability in this paper.
\end{comment}

\cite{damani2026beyond} proposed a reward scheme combining accuracy with Brier score and proven to incentivize both answering correctly and accurate confidence calibration. More generally, in their theoretical derivation, they considered reward schemes with the sum of a correctness reward and a proper scoring confidence calibration reward. To the best of our knowledge, \cite{damani2026beyond} were the first to successfully utilize RL for simultaneously learning to reason and performing confidence calibration. Despite experimenting with reward schemes that allow for confidence reward hacking, they did not show concrete evidence of reward hacking in practical datasets when trained simultaneously for correctness and calibration \citep{damani2026beyond}. We generalize their work to more general reward schemes where the additional correctness reward can vary by confidence. Subsequently, we demonstrate the presence of confidence reward hacking in practical datasets.

Similar reward functions have been proposed in \cite{wu2026mitigatingllmhallucinationbehaviorally} and \cite{wu2026basdecisiontheoreticapproachevaluating} based on the behavioral calibration framework of \cite{kalai2025languagemodelshallucinate}. In addition, \cite{wu2026basdecisiontheoreticapproachevaluating} constructed a family of reward schemes proven to satisfy proper scoring. We show that weakly non-hackable confidence reward schemes are completely characterized by the sum of a reward scheme derived by the behavioral calibration framework of \cite{kalai2025languagemodelshallucinate}, and an infimum reward difference between unconfidently correct and unconfidently wrong answers proposed by \cite{xuan2026confidencedichotomyanalyzingmitigating}. 

\section{Non-hackable Confidence Reward Schemes} \label{section:non-hackable-confidence-reward-schemes}

In this section, we motivate and propose the concept of non-hackable confidence reward schemes, designed to encourage the LLM to answer to the best of its ability and to be honest about its confidence. We demonstrate that our concept unifies existing frameworks, especially the behavioral calibration framework of \cite{kalai2025languagemodelshallucinate}. Subsequently, we designed an overconfidence to underconfidence spectrum of reward schemes. 

\subsection{Problem Formulation and Motivation} \label{section:problem-formulation}

Our goal is to investigate the theoretical incentives of reward functions in RL when LLMs are simultaneously trained for accuracy and confidence calibration. We consider the following setting:

For each question, the LLM is tasked to answer the question and provide a confidence value \(c \in (0, 1)\). We assume that the questions have a deterministic and well-defined set of correct answers. The LLM obtains a reward of \(f(c)\) for answering correctly and a reward of \(g(c)\) for answering incorrectly. The objective of the LLM is to maximize its reward. Note that this differs from the setting in \cite{damani2026beyond} since the reward for correctness is no longer constrained to be independent of \(c\). 

\begin{definition}
We define a ``binary correctness confidence reward scheme'', or ``reward scheme'' for short, as a pair of functions \((f(c), g(c))\) defined over \((0, 1)\), where \(f\) and \(g\) have continuous first derivatives. The corresponding "binary correctness confidence reward function", or "reward function" for short, is \(r(q, a, c) = \text{isCorrect}(q, a)f(c) + (1-\text{isCorrect}(q, a))g(c)\) where \(\text{isCorrect}(q, a)\) is 1 if the answer \(a\) to question \(q\) is correct and 0 otherwise. 

A reward function can be viewed as the composition of the reward scheme \((f(c), g(c))\) and the evaluator \(\text{isCorrect}\), where \(f(c)\) is returned if \(\text{isCorrect}(q, a)=1\) and \(g(c)\) otherwise. 
\end{definition}

If the reward scheme is not properly designed, the LLM may select some questions to intentionally answer wrongly or abstain with low confidence to obtain a higher reward, a phenomenon we term ``confidence reward hacking''. For more details on how this may occur, refer to Remark \ref{remark:why-consider-non-hackable-confidence-reward-schemes} in Appendix \ref{appendix:non-hackable-confidence-reward-schemes}. Therefore, we aim to characterize the reward schemes that disincentivize confidence reward hacking, termed as 
``non-hackable confidence''.

\subsection{Formulation of Non-hackable Confidence Reward Schemes}

A confident and correct answer should be rewarded more than an unconfident and correct answer. Likewise, a confident and incorrect answer should be rewarded less than an unconfident and incorrect answer. Ideally, we would want the LLM to answer to the best of its ability, and yet understand its own limitations by providing an informative estimate of its epistemic uncertainty. This notion, refined from \cite{damani2026beyond} via the inclusion of the Interpretability property, is formalized as follows:

\begin{restatable}{definition}{definitionnonhackableconfidence}\label{definition:non-confidence-hackable-reward-scheme}
A ``(strictly) non-hackable confidence reward scheme'' over an open interval \((a, b) \subseteq (0, 1)\), taken as over \((0, 1)\) without an explicitly mentioned interval, is a pair of functions \((f(c), g(c))\) satisfying the following properties:

(i) \textbf{(Strict) Interpretability}: On \((0, 1)\), \(f(c)\) is strictly increasing and \(g(c)\) is strictly decreasing.

(ii) \textbf{(Strict) Proper Scoring}: Let \(E(c, p) = \mathbb{E}_{q \sim Bernoulli(p)}[q f(c) + (1-q)g(c)] = p f(c) + (1-p)g(c) \), defined over \(c \in (0, 1)\) and \(p \in [0, 1]\). Then, for all \(p \in (0, 1)\), \(\text{argmax}_{c} E(c, p) = \{p\}\).

(iii) \textbf{(Strict) Best Effort}: Let \(E_{max}(c) = E(c, c)\). Then, \(E_{max}(c)\) is strictly increasing on \((a, b)\).

\end{restatable}

\begin{restatable}{definition}{definitionweaknonhackableconfidence}\label{definition:weak-non-confidence-hackable-reward-scheme}
A ``weakly non-hackable confidence reward scheme'' over an open interval \((a, b) \subseteq (0, 1)\), taken as over \((0, 1)\) without an explicitly mentioned interval, is a pair of functions \((f(c), g(c))\) satisfying the following properties, where \(E(c, p)\) and \(E_{max}(c)\) are as defined in Definition \ref{definition:non-confidence-hackable-reward-scheme}:

(i) \textbf{Weak Interpretability}: \(f(c)\) is non-decreasing on \((0, 1)\) and \(g(c)\) is non-increasing on \((0, 1)\).

(ii) \textbf{Weak Proper Scoring}: For all \(p \in (0, 1)\), \(p \in \text{argmax}_{c} E(c, p)\).

(iii) \textbf{Weak Best Effort}: \(E_{max}(c)\) is non-decreasing in \((a, b)\).

\end{restatable}

We characterize weakly non-hackable confidence reward schemes and provide a sufficient condition for a weakly non-hackable confidence reward scheme to be strictly non-hackable confidence in the following theorem, with the proof found in Appendix \ref{appendix:non-hackable-confidence-reward-schemes}.

\begin{theorem}\label{theorem:main-theorem}
A reward scheme is weakly non-hackable confidence over \((a, b) \subseteq (0, 1)\) if and only if all of the following hold:

(i) There exists a function \(h(c)\) defined over \((0, 1)\) such that \(h(c)=\frac{f'(c)}{c-1}=\frac{g'(c)}{c}\)

(ii) \(h(c) \leq 0\) for all \(c \in (0, 1)\)

(iii) \(f(a^+) \geq g(a^+)\)

If in addition, \(h(c) <  0\) for all \(c \in (0, 1)\), then \((f(c), g(c))\) is strictly non-hackable confidence.
\end{theorem}

For example, the standard binary reward for correctness \((1, 0)\) is weakly non-hackable confidence while the Brier-based reward scheme \((1 - (1-c)^2, -c^2)\), originally proposed by \cite{damani2026beyond} is strictly non-hackable confidence. Table \ref{table:reward-schemes} shows the reward schemes used in our experiments and whether they are strictly or weakly non-hackable confidence.

In practice, the LLM may be instructed to output confidence in discrete values from an allowable finite set, such as an integer between 0 and 100 inclusive. With mild additional assumptions, non-hackable confidence reward schemes incentivize the LLM to output the nearest higher or nearest lower representable confidence value corresponding to the true probability of answering correctly. The formalisms and the proof can be found in Theorem \ref{theorem:unimodal-R-discrete} and Remark \ref{remark:discrete-strict-variant} of Appendix \ref{appendix:discrete-confidence-values}.

\begin{table}[t]
    \caption{Table of reward schemes used in experiments. \(f(c)\) is the reward for answering correctly while \(g(c)\) is the reward for answering incorrectly. Details on whether reward schemes listed are strictly or weakly non-hackable confidence can be found in Appendix \ref{appendix:non-hackable-determination}.}
    \centering
    \begin{tabular}{c c c c}\toprule
         Reward Scheme&  \(f(c)\)& \(g(c)\) &Non-hackable confidence?\\ \midrule
 Correctness-only & 1& 0& Weak\\ \midrule
         Log-\(k\)&  \(1+ k\ln(c)\)&  \(k\ln(1-c)\)&No\\ \midrule
 Log Loss & \(\ln(c)\)& \(\ln(1-c)\)&No\\ \midrule
         Brier-\(k\) &  \(1- k(1-c)^2\)&  \(- kc^2\)& Only for \(k \leq 1\)\\ \midrule
 Brier Score & \(-(1-c)^2\)& \(-c^2\) &No\\ \midrule
         Brier-Log Hybrid&  \(c\)&  \(c + \ln(1-c)\)& Strict\\ \midrule
         Overconfidence-\(k\)& \(\frac{(k+1)\ln(ck+1)-ck}{(k+1)\ln(k+1)-k}\) & \( \frac{\ln(ck+1)-ck}{(k+1)\ln(k+1)-k}\) & Strict \\ \midrule
         Underconfidence-\(k\)& \(\frac{kc+\ln\left(1-\frac{kc}{1+k}\right)}{k-\ln(1+k)}\) & \(\frac{kc+(k+1)\ln\left(1-\frac{kc}{1+k}\right)}{k-\ln(1+k)}\)& Strict \\ \bottomrule
    \end{tabular}
    \label{table:reward-schemes}
\end{table}

\subsection{Alternative Formulation though Behavioral Calibration}

The behavioral calibration framework of \cite{kalai2025languagemodelshallucinate} considers a quiz where obtaining the correct answer has reward 1, abstaining has reward 0 and answering incorrectly has reward \(-\frac{t}{1-t}\). It would be rational to abstain if the confidence in obtaining the correct answer is lower than \(t\) and answer otherwise, assuming internal confidence estimates are well-calibrated. 

Let \(w(t)\) be a non-negative risk weighting function of confidence threshold \(t\), below which it would be better to abstain. Following \cite{wu2026mitigatingllmhallucinationbehaviorally, wu2026basdecisiontheoreticapproachevaluating}, we integrate with respect to the risk weighting function to obtain weighted abstention confidence threshold reward schemes. We first demonstrate that this results in a weakly non-hackable confidence reward scheme, with proof in Appendix \ref{appendix:decision-making-construction}.

\begin{restatable}{theorem}{abstentionframeworknonhackableconfidence}\label{theorem:abstention-framework-non-hackable-confidence}
The weighted abstention confidence threshold reward scheme 
\((f(c), g(c)) = (\int^c_0 w(t) \: dt, \int^c_0 w(t)\frac{t}{t-1} \: dt )\) is a weakly non-hackable confidence reward scheme for any non-negative and continuous \(w(t)\) on \((0, 1)\) such that \(\int_0^c w(t) \: dt\) converges for all \(c \in (0, 1)\).
\end{restatable}

If \(w(t)\) is positive on \((0, 1)\), then the weighted abstention confidence threshold reward scheme is strictly non-hackable confidence, as explained in Remark \ref{remark:abstention-framework-strictly-non-hackable-confidence}.

\begin{restatable}{remark}{abstentionframeworkstrictlynonhackableconfidence}\label{remark:abstention-framework-strictly-non-hackable-confidence}
If in addition \(w(t) > 0\) for all \(t \in (0, 1)\), then \(h(t) = \frac{w(t)}{t-1} < 0\) for all \(t \in (0, 1)\), therefore, by Theorem \ref{theorem:main-theorem}, the weighted abstention confidence threshold reward scheme \((\int^c_0 (t-1)h(t) \: dt, \int^c_0 th(t) \: dt)\) is a strictly non-hackable confidence reward scheme.
\end{restatable}

Unifying the ideas of \cite{wu2026basdecisiontheoreticapproachevaluating} and \cite{xuan2026confidencedichotomyanalyzingmitigating}, we define generalized weighted abstention confidence threshold reward schemes \((f(c), g(c))\) as the sum of weighted abstention confidence threshold reward schemes and offsets \(f_0 = f(0^+)\) and \(g_0 = g(0^+)\) such that \(f_0 \geq g_0\). An equivalent stand-alone definition can be found in Definition \ref{definition:generalized-weighted-abstention-confidence-threshold-reward-scheme} in Appendix \ref{appendix:decision-making-construction}. It can be shown that generalized weighted abstention confidence threshold reward schemes and weakly non-hackable confidence reward schemes are equivalent, with formal statement and proof in Theorem \ref{theorem:abstention-framework-non-hackable-confidence-parametrization} of Appendix \ref{appendix:decision-making-construction}.

\subsection{Overconfidence to Underconfidence Spectrum}

In this subsection, we define an overconfidence to underconfidence spectrum of weakly non-hackable confidence reward schemes. To understand the spectrum, we first define the miscalibration penalty. Intuitively, the miscalibration penalty is the loss in expected reward due to outputting a sub-optimal confidence value \(c\) instead of the true probability \(p\) of answering the question correctly.

\begin{definition}\label{definition:miscalibration-penalty}

The miscalibration penalty for a reward scheme \((f(c), g(c))\), denoted as \(E_{pen}(c, p)\), is defined as \(E_{pen}(c, p) = E(p, p) - E(c, p)\). The miscalibration penalty can be termed ``overconfidence penalty" if \(c > p\), and ``underconfidence penalty" if \(c < p\).

\end{definition}

\begin{definition}\label{definition:underconfidence-to-overconfidence-spectrum}

We define an overconfidence to underconfidence spectrum of non-confidence hackable reward schemes as follows, where the mentioned reward schemes are as defined in Table \ref{table:reward-schemes}. In order from the overconfident end to the underconfident end of the spectrum, reading from top to bottom, we have the following reward schemes:

\begin{itemize}[noitemsep, topsep=0pt]

\item Correctness-only
\item Overconfidence-\(k\), \(k > 0\), larger \(k\) indicates more towards the overconfident end.
\item Brier-1
\item Underconfidence-\(k\), \(k > 0\), larger \(k\) indicates more towards the underconfident end.
\item Brier-log Hybrid\footnote{Brier-log Hybrid is named as such since it approximates half of Brier-1 reward at small \(c\) and \(p\), while it approximates Log Loss at \(c\) close to 1. More detailed justification can be found in Appendix \ref{appendix:brier-log-hybrid-naming}.}

\end{itemize}

\end{definition}

We also introduce the notion of overconfidence and underconfidence bias. Intuitively, a reward scheme exhibits overconfidence bias when the LLM is penalized less by being slightly overconfident compared to being slightly underconfident. Definition \ref{definition:overconfidence-underconfidence-bias} formalizes this notion.

\begin{restatable}{definition}{definitionconfidencebias} \label{definition:overconfidence-underconfidence-bias}

A weakly non-hackable confidence reward scheme \((f(p), g(p))\) has overconfidence bias (respectively, underconfidence bias) if the miscalibration penalty \(E_{pen}\), as defined in Definition \ref{definition:miscalibration-penalty}, satisfies the following property:

For all \(\delta > 0\) and \(p \in (0, 1)\) such that \(p + \delta \in (0, 1)\) and \(p - \delta \in (0, 1)\), \(E_{pen}(p+\delta, p) < E_{pen}(p - \delta, p)\) (respectively \(E_{pen}(p+\delta, p) > E_{pen}(p - \delta, p)\)) 

\end{restatable}

Figure \ref{fig:miscalibration-penalty} shows the miscalibration penalty plot for some reward schemes along the overconfidence to underconfidence spectrum. Brier-log Hybrid and Underconfidence-\(k\) are towards the underconfident end of the spectrum and exhibit underconfidence bias. Overconfidence-\(k\) is towards the overconfident end of the spectrum and exhibits overconfidence bias. Two guiding lines corresponding to where the LLM is overconfident by 0.2 and underconfident by 0.2 are included to help the reader compare the overconfidence penalty and the underconfidence penalty.

\begin{figure}[h]
    \centering
    \includegraphics[width=1\linewidth]{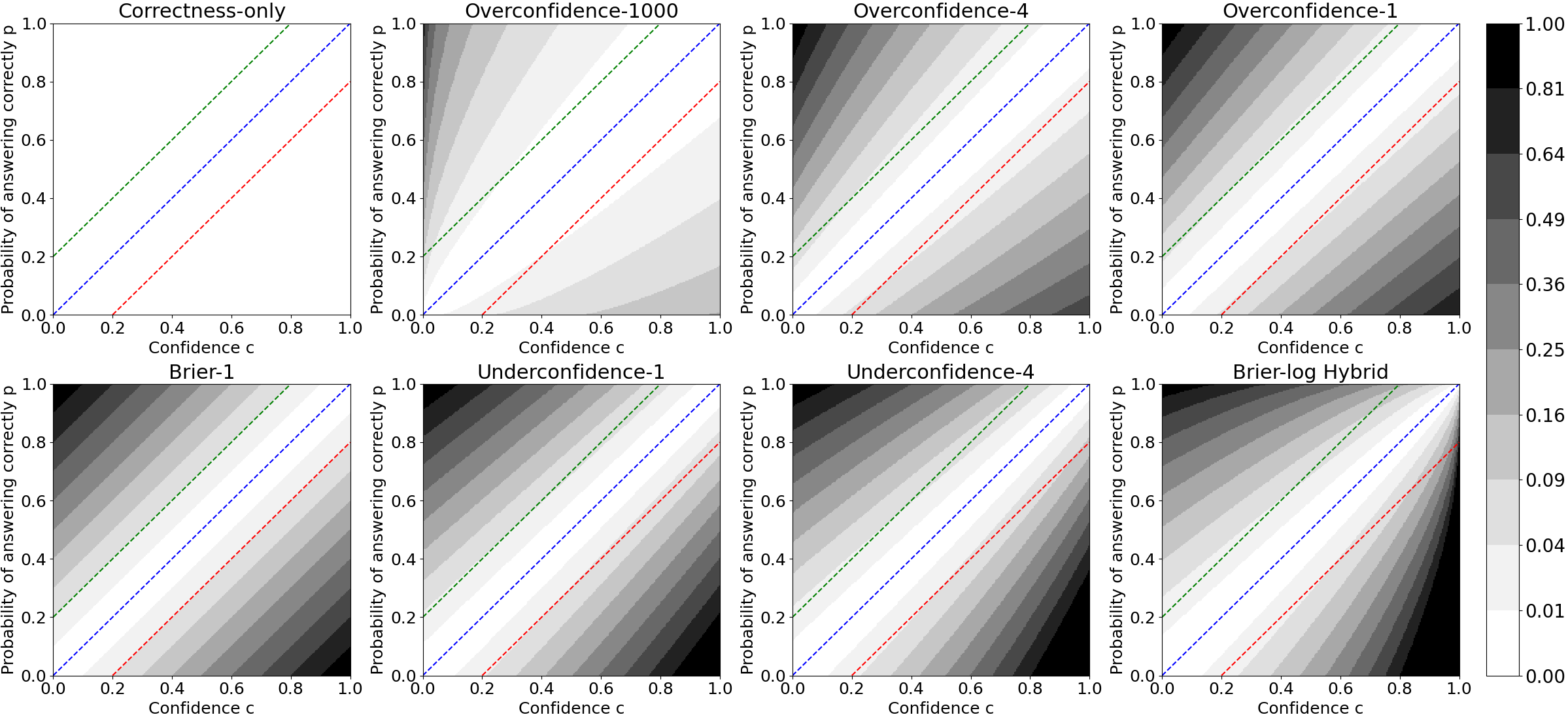}
    \caption{Miscalibration penalty of Correctness-only, Overconfidence-1000, Overconfidence-4, Overconfidence-1, Brier-1, Underconfidence-1, Underconfidence-4 and Brier-log Hybrid reward schemes. The reward schemes are as defined in Table \ref{table:reward-schemes}. Darker greys mean higher miscalibration penalty. Blue line represents ideal confidence calibration. Green and red line represents being consistently underconfident and overconfident by 0.2 respectively.}
    \label{fig:miscalibration-penalty}
\end{figure}

A sufficient condition to guarantee overconfidence and underconfidence bias is provided in Theorem \ref{theorem:confidence-bias-test}. More details on the formalization, the proof of the theorem and why reward schemes in our spectrum exhibit overconfidence or underconfidence bias can be found in Appendix \ref{appendix:confidence-bias}.

\begin{restatable}{theorem}{confidencebiastest}\label{theorem:confidence-bias-test}

Let \((f(c), g(c))\) be weakly non-hackable confidence such that \(f\) and \(g\) have continuous second-order derivatives. Let \(E(c, p)\) be as defined in Definition \ref{definition:non-confidence-hackable-reward-scheme}. Then, the following hold:

(i) If for all \(p \in (0, 1)\), \(\frac{\partial^2 E}{\partial c^2}\) is strictly decreasing with respect to \(c\) on \((0, 1)\) when keeping \(p\) fixed, then \((f(c), g(c))\) has underconfidence bias.

(ii) If for all \(p \in (0, 1)\), \(\frac{\partial^2 E}{\partial c^2}\) is strictly increasing with respect to \(c\) on \((0, 1)\) when keeping \(p\) fixed, then \((f(c), g(c))\) has overconfidence bias.

\end{restatable}

Both Brier-1 and Correctness-only have neither overconfidence nor underconfidence bias. Nonetheless, Overconfidence-\(k\) converges pointwise in terms of the LLM confidence \(c\) to Correctness-only over \((0, 1)\) as \(k \rightarrow \infty\), which justifies the position of Correctness-only furthest in the overconfident end of the spectrum. A rigorous derivation can be found in Appendix \ref{appendix:confidence-bias}.

\section{Experimental Settings}

In this section, we describe the datasets, reward schemes and the procedure for our experiments.

\subsection{Datasets Used}

For our experiments, we used the HotpotQA \citep{yang2018hotpotqa}, HotpotQA-Modified (inspired by \citep{damani2026beyond}), BigMath \citep{albalak2025bigmathlargescalehighqualitymath} and DeepMath-103K \citep{deepmath} datasets. HotpotQA is a textual reasoning dataset where the LLM is given sources and required to reason about the sources to provide the answer. HotpotQA-Modified is a variant of HotpotQA where two sources are removed at random. Questions with one or both removed sources being relevant are harder to answer because the LLM is required to answer based on its trained knowledge. BigMath and DeepMath-103K are math reasoning datasets. More details of the datasets used and their preprocessing can be found on Appendix \ref{appendix:experimental-setting-details}.

\subsection{Reward Schemes Used}\label{section:reward-schemes-used}

For our experiments, we use the reward schemes shown in Table \ref{table:reward-schemes}. Let \(R_{spec}\) = \{Correctness-only, Overconfidence-1000, Overconfidence-4, Overconfidence-1, Brier-1, Underconfidence-1, Underconfidence-4, Brier-log Hybrid\} denote the set of weakly non-hackable confidence reward schemes that we examine in our experiments. Justifications behind why each of the reward schemes are or are not hackable can be found in Appendix \ref{appendix:non-hackable-determination}. 

For Log-\(k\), we tested \(k \in \{1, \frac{1}{\ln 202}\}\). For Brier-\(k\), we tested \(k \in \{1, 2\}\). For Overconfidence-\(k\), we tested \(k \in \{1, 4, 1000\}\). For Underconfidence-\(k\), we tested \(k \in \{1, 4\}\). All reward schemes in Table \ref{table:reward-schemes} except Log Loss and Brier Score satisfy \(f(1^-) = 1\) and \(g(0^+) = 0\), which implies \(E_{max}(1^-)=1\) and \(E_{max}(0^+) = 0\). This ensures a consistent reward scale for a fair comparison among reward schemes with correctness rewards.

\subsection{Experimental Procedure}

We ran our experiments on Qwen 2.5 (3B) Instruct \citep{qwen2, qwen2.5}, Qwen 3 (4B) \citep{qwen3technicalreport} and Gemma 4 (E2B) Instruct \citep{gemmateam2026gemma4}. For each reward scheme, we ran RL training on the LLM using Dr GRPO \citep{liu2025understandingr1zeroliketrainingcritical}, except for Gemma 4 (E2B) Instruct, which we used DAPO \citep{yu2025dapoopensourcellmreinforcement} for better training stability. Standard deviation advantage normalization is not used. For Qwen 2.5 (3B) Instruct, we ran our experiments for all datasets. Due to limited computational resources, we trained Qwen 3 (4B) only on DeepMath-103K and HotpotQA-Modified, and Gemma 4 (E2B) Instruct only on HotpotQA-Modified. 

We evaluated the model on Accuracy, Expected Calibration Error (ECE), Area Under Receiver Operating Characteristic Curve (AUROC), Brier score, average Brier-1 reward and Calibration bias. In the main text, we focus on the accuracy and the calibration bias metrics. Results regarding the other metrics can be found in Appendix \ref{appendix:other-experimental-observations}. Calibration bias is defined as accuracy minus average confidence. A negative calibration bias indicates overconfidence while a positive calibration bias indicates underconfidence. More details on the metrics used can be found in Appendix \ref{appendix:evaluation-metrics-used}. For full details on the experimental procedure, refer to Appendix \ref{appendix:experimental-procedure}.

\section{Experimental Results}

In this section, we demonstrate the presence of confidence reward hacking for hackable confidence reward schemes, largely matching our theory. Subsequently, we demonstrate that reward schemes in the overconfident end in our spectrum have more overconfidence bias relative to other reward schemes in our spectrum, and likewise for underconfidence. More observations, such as those regarding accuracy and speed of confidence calibration, can be found in Appendix \ref{appendix:other-experimental-observations}. The full results, containing performance statistics and abstention rates by difficulty split, can be found in Appendix \ref{appendix:full-results}. 

\subsection{Presence of Confidence Reward Hacking}\label{section:results-rq1}

In this subsection, we demonstrate the presence of confidence reward hacking in reward schemes that allow for such hacking, which can occur to varying degrees in practice. Table \ref{table:results-bigmath-hotpotqa-modified} shows the accuracy statistics broken down by question difficulty when Qwen 2.5 (3B) Instruct is trained with different datasets and Correctness-only, Log-1 and Brier-2 reward schemes.

\begin{table}[t]
\centering
% Booktabs: https://nhigham.com/2019/11/19/better-latex-tables-with-booktabs/ is a good resource
% Adjustbox: https://tex.stackexchange.com/questions/10863/is-there-a-way-to-slightly-shrink-a-table-including-font-size-to-fit-within-th
\caption{Accuracy results for BigMath, DeepMath-103K, HotpotQA-Modified and HotpotQA for different reward schemes when trained with Qwen 2.5 (3B) Instruct, split by difficulty. The reward schemes are subdivided into three categories from top to bottom: independent of confidence, hackable confidence and strictly non-hackable confidence. Bolded results correspond to the best result while underlined results correspond to the second best result.}
\label{table:results-bigmath-hotpotqa-modified}
\begin{adjustbox}{width=\columnwidth,center}
\begin{tabular}{l c c c c c c c c c c} 
\toprule
\multirow{2}{*}{Reward Scheme} & \multicolumn{3}{c}{BigMath} & \multicolumn{3}{c}{DeepMath-103K} & \multicolumn{3}{c}{HotpotQA-Modified} & HotpotQA \\ \cmidrule(lr){2-4} \cmidrule(lr){5-7} \cmidrule(lr){8-10} \cmidrule(lr){11-11}
 & easy & medium & hard & easy & medium & hard & easy & medium & hard & hard \\ \midrule
Correctness-only & \underline{0.9540} & \underline{0.7545} & \underline{0.3229} & 0.5308 & 0.3940 & 0.4189 & 0.6273 & \textbf{0.4250} & \underline{0.2100} & \textbf{0.6571} \\ \midrule
Log Loss & 0.0001 & 0.0001 & 0.0000 & 0.0000 & 0.0000 & 0.0003 & 0.0001 & 0.0000 & 0.0000 & 0.0000 \\
Brier Score & 0.0003 & 0.0001 & 0.0000 & 0.0001 & 0.0001 & 0.0001 & 0.0000 & 0.0001 & 0.0000 & 0.0000 \\
Log-1 & 0.9388 & 0.6981 & 0.2748 & 0.0001 & 0.0000 & 0.0000 & 0.2022 & 0.0619 & 0.0204 & 0.6472 \\
Brier-2 & 0.9530 & 0.7520 & 0.3185 & 0.0000 & 0.0001 & 0.0004 & 0.5296 & 0.2323 & 0.0958 & 0.6507 \\
Log-\(\frac{1}{\ln 202}\) & 0.9501 & 0.7427 & 0.3139 & \textbf{0.5334} & \textbf{0.4059} & \textbf{0.4260} & 0.6271 & 0.4203 & 0.2093 & 0.6537 \\ \midrule
Overconfidence-1000 & 0.9529 & 0.7506 & 0.3172 & 0.5262 & 0.3886 & \underline{0.4197} & 0.6228 & 0.4135 & 0.2050 & 0.6548 \\
Overconfidence-4 & 0.9509 & 0.7474 & 0.3169 & \underline{0.5324} & \underline{0.4054} & 0.4167 & 0.6248 & 0.4187 & 0.2070 & 0.6526 \\
Overconfidence-1 & 0.9521 & 0.7532 & 0.3211 & 0.5205 & 0.3855 & 0.4063 & \textbf{0.6328} & \underline{0.4238} & 0.2055 & 0.6507 \\
Brier-1 & 0.9529 & 0.7480 & 0.3117 & 0.5108 & 0.3796 & 0.3976 & \underline{0.6317} & 0.4184 & 0.2028 & 0.6504 \\
Underconfidence-1 & 0.9505 & 0.7444 & 0.3146 & 0.5191 & 0.3876 & 0.4105 & 0.6245 & 0.4197 & \textbf{0.2110} & \underline{0.6551} \\
Underconfidence-4 & \textbf{0.9552} & \textbf{0.7573} & \textbf{0.3279} & 0.4948 & 0.3613 & 0.3911 & 0.6199 & 0.4158 & 0.2008 & 0.6529 \\
Brier-log-hybrid & 0.9513 & 0.7460 & 0.3169 & 0.3312 & 0.2263 & 0.2919 & 0.6111 & 0.4003 & 0.1915 & 0.6439 \\ \bottomrule

\end{tabular}
\end{adjustbox}
\end{table}

The LLMs trained with the Log Loss and Brier Score reward schemes have near-zero accuracies, highlighting the necessity of adding a correctness reward (which may vary by confidence) when answering ability and confidence calibration are simultaneously trained. 

Compared to the Correctness-only reward scheme, both Log-1 and Brier-2 reward schemes generally result in accuracy drops in HotpotQA-Modified at all difficulty levels, with the relative accuracy decrease being most prominent at hard difficulty and least prominent at easy difficulty. This suggests that the LLM is selectively choosing harder problems to answer incorrectly. 

In theory, Brier-2 and Log-1 satisfy \(E_{max}(0^+) = 0\), \(E_{max}(c) < 0\) for all \(c \in (0, p)\) and \(E_{max}(c) > 0\) for all \(c \in (p, 1)\) for some \(p \in (0, 1)\). The value of \(p\) for Brier-2 is \(\frac{1}{2}\), which is smaller than \(p \approx 0.64842\) for Log-1. The confidence threshold below which it is optimal for the LLM to always answer incorrectly with confidence approaching \(0^+\) is lower for Brier-2, hence Brier-2 is theoretically less prone to confidence reward hacking. 

Since all reward schemes in \(R_{spec}\) are weakly non-hackable confidence, the LLMs are theoretically not incentivized to answer incorrectly in order to obtain a greater reward. In line with the theory, the accuracies are largely similar to Correctness-only.

No evidence of confidence reward hacking occurred for Log-\(\frac{1}{\ln 202}\) as the accuracy remained comparable to Correctness-only and abstention rates remain low. Log-\(\frac{1}{\ln 202}\) is a strictly non-hackable confidence reward scheme over \((\frac{1}{203}, 1)\) but not \((0, 1)\). In practice, we can clip the confidence values to \([\frac{1}{203}, 1-\frac{1}{203}]\) to prevent confidence reward hacking, similar to what is proposed in Equation 4 of \cite{wu2026mitigatingllmhallucinationbehaviorally}, and the ``truncated logarithmic scoring system'' in \citet[p. 137]{shuford1966admissible}.

During the final step of the training run, as more evidence of confidence reward hacking, Qwen 2.5 (3B) Instruct may give up on answering the question, giving responses such as ``Unknown" with confidence 0. Appendix \ref{appendix:reward-hacking-examples} shows a case study of confidence reward hacking. Moreover, the results of other LLMs are largely aligned with our theory, with more information found in Appendix \ref{appendix:presence-of-confidence-reward-hacking}.

Overall, the results suggest that confidence reward hacking, selective or total, can occur when hackable confidence reward schemes are used. Further investigation into when exactly confidence reward hacking occurs in practice will be left to future work.

\subsection{Confidence Calibration Biases} \label{section:results-rq2}

We demonstrate using Figure \ref{fig:confidence-calibration-plot-bigmath} that reward schemes nearer to the overconfident end of our overconfidence-to-underconfidence spectrum empirically tend to result in more overconfident LLMs relative to other reward schemes in the spectrum. Correctness-only, being in the overconfident end of the spectrum, displays the most overconfidence. This is consistent with the empirical observation that LLMs trained for accuracy in RL tend to be overconfident \citep{xiaohu2026knowyourewrongaligning, openai2024gpt4technicalreport}. Overconfidence-1000 is slightly less overconfident. In contrast, the LLMs trained with Brier-1 and Brier-log Hybrid, which are nearer to the underconfident end of the spectrum than Overconfidence-1000, are well calibrated in confidence.

\begin{figure}[t]
    \centering
    \includegraphics[width=1\linewidth]{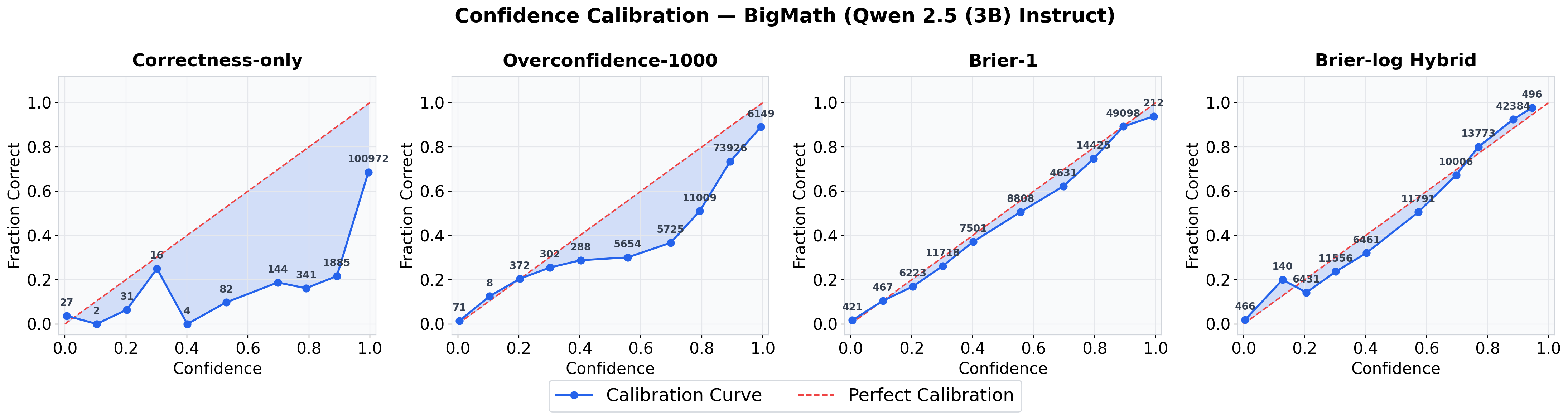}
        \vspace{0.05in}
    \caption{Confidence calibration plots of Qwen 2.5 (3B) Instruct after RL training with Correctness-only, Overconfidence-1000, Brier-1 and Brier-log Hybrid reward schemes %(from left to right, in order from overconfident end to underconfident end of spectrum)
    from the BigMath dataset. Each question in the test set was evaluated with 16 different responses. Confidence values are subdivided into 10 equally spaced bins, with tie-breaking at boundaries to the right. The label of each point represents the size of each bin. The ideal calibration line is in red.}

    \label{fig:confidence-calibration-plot-bigmath}
\end{figure}

\begin{figure}[h]
    \centering
    \includegraphics[width=1\linewidth]{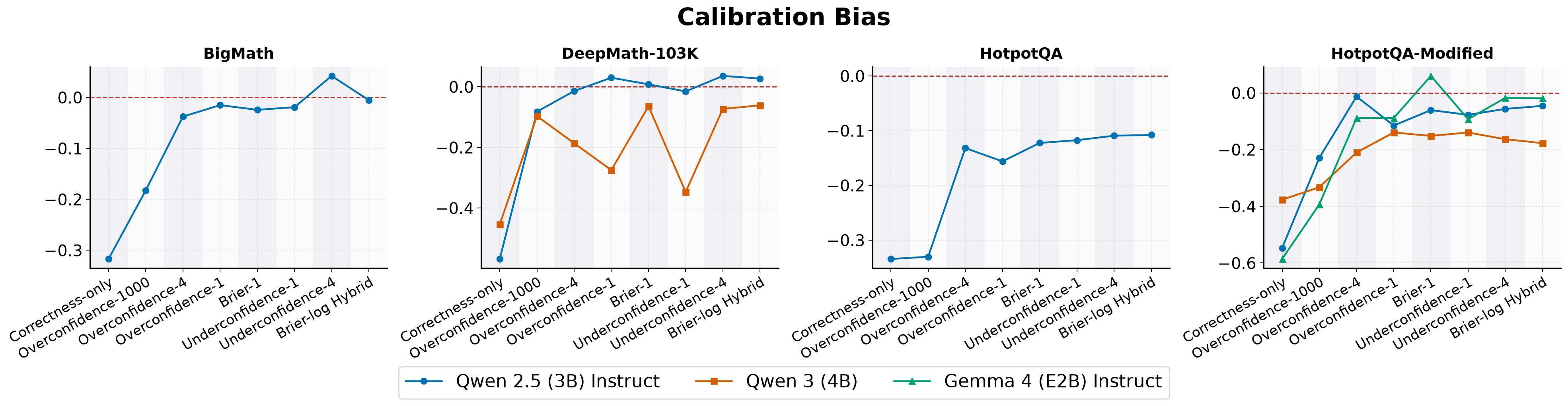}
        \vspace{0.05in}
    \caption{Calibration biases of Qwen 2.5 (3B) Instruct, Qwen 3 (4B) and Gemma 4 (E2B) Instruct after RL training with reward schemes in \(R_{spec}\), plotted for each of the datasets we tested. Reward schemes in the x-axis are arranged, in the order from left to right, in order from the overconfident end to the underconfident end of the spectrum. Each question in the test set was evaluated with 16 different responses except for Qwen 3 (4B) DeepMath, where only 2 different responses per question were evaluated. Negative calibration bias means that the LLM is overconfident on average.}

    \label{fig:calibration-bias-plot}
\end{figure}

As shown in Figure \ref{fig:calibration-bias-plot}, the calibration biases show similar trends in other datasets and other LLMs, with less overconfidence bias as one moves from the overconfident end to the underconfident end of the spectrum. The exception is Qwen 3 (4B) when trained on DeepMath-103K, which can be attributed to high invalid formatting rates affecting accuracy and therefore confidence calibration, with more details in Appendix \ref{appendix:invalid-formatting-rates}.

\section{Limitations and Future Work}

In this section, we analyze the limitations of our work and provide future research directions. First, correctness evaluation is subject to evaluation noise that may seep into confidence estimates, leading to post-training biases. Second, this work focuses on the setting in which one answer is provided and marked as either completely correct or completely wrong. Future work can generalize to scenarios involving partial credit and categorical distributions of possible correct answers. Third, this work is dependent on the stability of RL training of LLM. Future work can investigate into the loss landscapes of different reward schemes to improve learning stability. Fourth, this work does not focus on the modeling of the ground truth in cases where abstention is desired, i.e. whether the LLM should be encouraged to output ``not enough information'' with 100\% confidence, or encouraged to attempt to answer and speculate with near 0\% confidence and highlight that the lack of necessary information. Future work can determine the pros and cons of the approaches.

\section{Conclusion}

In this paper, we examined RL for LLMs in the setting where the LLM is trained to jointly improve reasoning accuracy and verbalized confidence calibration, and characterize appropriate reward schemes, termed non-hackable confidence reward schemes, for the task. Empirically, we showed that poorly-designed reward schemes can lead to selective confidence reward hacking on practical datasets, with relative degradation of accuracy concentrated on harder questions compared to training for accuracy alone. We proposed an overconfidence to underconfidence spectrum of weakly non-hackable confidence reward schemes, where we experimentally showed that reward schemes near the overconfident and the underconfident ends respectively tend to result in relatively overconfident and underconfident calibration biases compared to other reward schemes in the spectrum.

\bibliography{ref}
\bibliographystyle{iclr2027_conference}

\appendix

\section{More Theory of Non-hackable Confidence Reward Schemes}\label{appendix:proof-theorems}

The first three subsections show details behind the theory discussed in the paper. The first subsection shows more details in the characterization of weakly non-hackable confidence reward schemes while the second subsection shows an alternative formulation of weakly non-hackable confidence reward schemes based on the behavioral calibration framework of \cite{kalai2025languagemodelshallucinate}. Sufficient conditions are also provided for a weakly non-hackable confidence reward scheme to be strictly non-hackable confidence. The third subsection shows that discretization in the set of representable confidence values has a limited effect on the proper scoring incentive provided mild conditions in addition to proper scoring are met.

The fourth subsection shows that non-hackable confidence reward schemes can be linearly scaled while preserving non-hackable confidence property. The fifth subsection briefly extends to the setting where abstention for low confidence below the abstention threshold is desired but the LLM is still incentivized to output its true confidence above the abstention threshold. The sixth subsection briefly demonstrates that the Interpretability property is redundant and can therefore be safely skipped when checking for non-hackable confidence property using the definition.

\subsection{Characterization of non-hackable confidence reward schemes}\label{appendix:non-hackable-confidence-reward-schemes}

In this section, we first motivate our definition of weakly and strictly non-hackable confidence reward schemes, then proceed to characterize the set of weakly non-hackable confidence reward schemes and provide a sufficient condition for a weakly non-hackable confidence reward scheme to be strictly non-hackable confidence. Moreover, we provide a general form of weakly non-hackable confidence reward schemes.

\begin{remark} \label{remark:why-consider-non-hackable-confidence-reward-schemes} Why we consider non-hackable confidence reward schemes.

Consider the strategy where the LLM performs at its best and honestly reports its confidence \(c\), aligned with the Bayesian probability of answering correctly. We want this strategy to be optimal. Its expected reward is \(E_{max}(c)\).

Suppose \((f(c), g(c))\) is a hackable confidence reward scheme. Then, there exists \(c \in (0, 1)\) and \(p \in (0, c)\) such that \(E_{max}(p) > E_{max}(c)\). Consider the strategy where the LLM reports a confidence of \(p\), and answers to the best of its ability with probability \(\frac{p}{c}\) and answers gibberish, assumed to be a wrong, otherwise. Then, the probability of the LLM answering correctly is \(\frac{p}{c}c + (1-\frac{p}{c})(0) = p\). Therefore, the expected reward of this strategy is \(E_{max}(p) > E_{max}(c)\).

Therefore, the LLM can obtain a higher maximum expected reward by intentionally answering wrongly with some probability, for instance, by answering some gibberish. In practice, this defeats the purpose of doing RL to improve LLM reasoning, as this loophole encourages the LLM to answer incorrectly. Since such a \(p\) may only exist for a proper subset of questions in the dataset, it is possible that the LLM will selectively choose the subset of questions to answer incorrectly, which may lead to difficult debugging or even non-detection of bugs if the overall accuracy happens to increase during the training. 

\end{remark}

For ease of reference, the definition of non-hackable confidence reward schemes for both the strict and weak variants are restated.

\definitionnonhackableconfidence*

\definitionweaknonhackableconfidence*

\begin{remark}

Corresponding reward functions of (strictly) non-hackable confidence reward schemes are termed ``(strictly) non-hackable confidence reward functions''. Likewise, corresponding reward functions of weakly non-hackable confidence reward schemes are termed ``weakly non-hackable confidence reward functions''. 

\end{remark}

\begin{remark}\label{remark:strict-non-hackable} Difference between strictly non-hackable confidence and weakly non-hackable confidence

Strict Interpretability ensures that correct answers with higher confidence are always given a higher reward and incorrect answers with higher confidence are always penalized more heavily. 

Strict Proper Scoring provides a stronger theoretical guarantee that the only optimal solution is for the LLM to be honest about its epistemic uncertainty, assuming the answer to each question is deterministic and the evaluator is deterministic. Without the strict condition, the LLM can output an inaccurate confidence estimate and still obtain the same reward.

Strict Best Effort ensures that the LLM is incentivized to always answer to the best of its ability as an increase the probability for answering correctly is always rewarded with an increase with the maximum possible expected reward for accurate confidence calibration.
\end{remark}

\begin{remark}\label{remark:strict-non-hackable-converse}
In general, strictly non-hackable confidence reward schemes are weakly non-hackable confidence. But the converse is not true. Correctness-only, i.e. the reward scheme \((1, 0)\), is an example of a weakly non-hackable confidence reward scheme that is not strictly non-hackable confidence. Theorem \ref{theorem:strict-non-hackable} describes a condition under which the converse is true.
\end{remark}

\begin{lemma}\label{lemma:derivation-of-h}

Reward schemes \( (f(c), g(c)) \) that satisfy the weak Proper Scoring property over \( (0, 1) \) must necessarily satisfy \(  \frac{f'(c)}{c-1} = \frac{g'(c)}{c} \) over \( (0, 1) \). 

\end{lemma}

\begin{remark}
A related result, i.e. both Lemma \ref{lemma:derivation-of-h} and its converse hold assuming the weak Interpretability property, was proven in \cite{shuford1966admissible} and generalized by \cite{schervish1989general}, with varying assumptions about \(f(c)\) and \(g(c)\). In this section, we take a slightly different approach and prove Lemma \ref{lemma:derivation-of-h} in the context of characterizing weakly non-hackable confidence reward schemes.
\end{remark}

\begin{proof}[\textbf{\textup{Proof of Lemma \ref{lemma:derivation-of-h}}}] 

\mbox{}

Consider \( E(c, p) = p f(c) + (1-p) g(c) \), where \(p\) is treated as constant. Since \(f\) and \(g\) are assumed to be have continuous first derivatives, \(E\) is partially differentiable with respect to \(c\). Differentiating with respect to \(E\), we obtain \(\frac{\partial E}{\partial c} = p f'(c) + (1-p) g'(c)\).

By the weak Proper Scoring property, since \(E(c, p)\) obtains its maximum at \(c=p\) when keeping \(p\) constant, \( p f'(p) + (1-p) g'(p) = 0 \) for all \( p \in (0, 1) \). Therefore, \( p f'(p) = (p-1) g'(p) \). Since \( p \in (0, 1)\), $p \neq 0$ and $p-1 \neq 0$, hence we divide by \(p(p-1)\) on both sides to obtain \( \frac{f'(p)}{p-1} = \frac{g'(p)}{p} \). A change of variable from \(p\) to \(c\) yields the result in the lemma. 

\end{proof}

\begin{definition}\label{definition:h}

For any reward scheme such that \(\frac{f'(c)}{c-1} = \frac{g'(c)}{c}\) for all \( c \in (0, 1) \), the function \(h(c)\) is well-defined over \((0,1)\) as \(h(c) = \frac{f'(c)}{c-1} = \frac{g'(c)}{c}\). In such a scenario, \(f(c)\) and \(g(c)\) can be expressed as \(f(c) = \int (c-1)h(c) \: \text{d}c\) and \(g(c) = \int ch(c) \: \text{d}c\), up to arbitrary additive constants.

\end{definition}

\begin{lemma} \label{lemma:h-interpretability}
Assuming \(h(c)\) is well-defined (as defined in Definition \ref{definition:h}), the weak Interpretability property holds if and only if \(h(c) \le 0\) for all \(c \in (0, 1)\). In addition, the strict Interpretability property holds if \(h(c) < 0\) for all \(c \in (0, 1)\).
\end{lemma}

\begin{proof}[\textbf{\textup{Proof of Lemma \ref{lemma:h-interpretability}}}]

\mbox{}

Recall that the weak Interpretability property is that \(f\) is non-decreasing on \((0, 1)\) and \(g\) is non-increasing on \((0, 1)\).
Because \(f\) and \(g\) are differentiable on \((0, 1)\):
\begin{flalign*}
&\text{\(f\) is non-decreasing on } (0, 1) &\\
&\iff f'(c) = (c-1) h(c) \geq 0 \text{ for all } c \in (0, 1)&\\
&\iff h(c) \le 0 \text{ for all } c \in (0, 1) \text{ since } c-1 < 0 \text { for all } c \in (0, 1)&
\end{flalign*}

Similarly,
\begin{flalign*}
&\text{\(g\) is non-increasing on } (0, 1) &\\
&\iff g'(c) = c h(c) \leq 0 \text{ for all } c \in (0, 1)&\\
&\iff h(c) \leq 0 \text{ for all } c \in (0, 1) \text{ since } c > 0 \text{ for all } c \in (0, 1).&
\end{flalign*}

Thus, the weak Interpretability property holds if and only if \(h(c) \le 0\) over \((0, 1)\).

Recall that the strict Interpretability property is that \(f\) is strictly increasing on \((0, 1)\) and \(g\) is strictly decreasing on \((0, 1)\).

Hence, 
\begin{flalign*}
& h(c) < 0 \text{ for all } c \in (0, 1) & \\
& \Rightarrow f'(c) = (c-1)h(c) > 0 \text{ for all } c \in (0, 1) \text{ since } c-1 < 0 \text{ for all } c \in (0, 1) & \\
& \Rightarrow \text{\(f\) is strictly increasing on } (0, 1) 
\end{flalign*}

Similarly,
\begin{flalign*}
& h(c) < 0 \text{ for all } c \in (0, 1) & \\
& \Rightarrow g'(c) = ch(c) < 0 \text{ since } c > 0 \text{ for all } c \in (0, 1) & \\
& \Rightarrow \text{\(g\) is strictly decreasing on } (0, 1) 
\end{flalign*}

Thus, the strict Interpretability property holds if \(h(c) < 0\) for all \(c \in (0, 1)\).

\end{proof}

\begin{lemma}\label{lemma:partial-derivative-computation}

Assuming \(h\) is well-defined, \(\frac{\partial}{\partial c}{E(c, p)} = (c - p)h(c)\) for all \(c \in (0, 1)\)

\end{lemma}

\begin{proof}[\textbf{\textup{Proof of Lemma \ref{lemma:partial-derivative-computation}}}]

\mbox{}

By the definition of the expected reward \(E\) in the Proper Scoring property of Definition \ref{definition:non-confidence-hackable-reward-scheme}, \(E(c, p) = p f(c) + (1-p)g(c)\). Therefore, \(\frac{\partial}{\partial c}{E(c, p)} = p f'(c) + (1-p) g'(c)\).

From Definition \ref{definition:h}, $f'(c) = (c-1)h(c)$ and $g'(c) = c h(c)$ for all $c \in (0, 1)$. 

Therefore, \(\frac{\partial}{\partial c}{E(c, p)} = p(c-1)h(c) + (1-p)ch(c) = (pc - p)h(c) + (c - pc)h(c) = (c - p)h(c)\)

\end{proof}

\begin{lemma}\label{lemma:rmax-derivative}
Assuming \(h\) is well-defined, \( E_{max}'(c) = f(c) - g(c) \) for all \( c \in (0, 1) \).
\end{lemma}

\begin{proof}[\textbf{\textup{Proof of Lemma \ref{lemma:rmax-derivative}}}] 
\,\hfill\break\\
By definition of \(E_{max}\) and \(E\), \( E_{max}(c) = E(c, c) = c f(c) + (1-c) g(c) \).

\( E_{max}'(c) = f(c) + c f'(c) - g(c) + (1-c)g'(c) \)

Since \(h\) is assumed to be well-defined, \(\frac{f'(c)}{c-1} = \frac{g'(c)}{c}\), which implies that \(c f'(c) + (1-c)g'(c) = 0\). Therefore, \(E_{max}'(c) = f(c) - g(c)\).
\end{proof}

\begin{remark}\label{remark:rmax-convex}
If the weak Interpretability property is satisfied, then \(E_{max}'\) is non-decreasing, therefore \(E_{max}\) is weakly convex. If the strict Interpretability property is satisfied, then \(E_{max}'\) is strictly increasing, therefore \(E_{max}\) is strictly convex. 
\end{remark}

\begin{restatable}{theorem}{maintheorem} \label{theorem:characterization-of-non-reward-hackable-schemes}
A reward scheme is weakly non-hackable confidence over \((a, b) \subseteq (0, 1)\) if and only if all of the following hold:

(i) There exists a function \(h(c)\) defined over \((0, 1)\) such that \(h(c)=\frac{f'(c)}{c-1}=\frac{g'(c)}{c}\)

(ii) \(h(c) \leq 0\) for all \(c \in (0, 1)\)

(iii) \(f(a^+) \geq g(a^+)\)

\end{restatable}

\begin{proof}[\textbf{\textup{Proof of Theorem \ref{theorem:characterization-of-non-reward-hackable-schemes}}}] 

\mbox{} % Google Search Gemini AI helped with this

\( (\Rightarrow) \) 

If a reward scheme is weakly non-hackable confidence, then properties (i), (ii) and (iii) must be proven to be true.

We are given that \((f, g)\) is a non-reward hackable scheme. Since \((f, g)\) satisfies the weak Proper Scoring property over \((0, 1)\), 
by Lemma \ref{lemma:derivation-of-h}, \(h(c)\) exists and is well-defined, satisfying property (i).

By weak Interpretability property, since we assume \(f\) to be differentiable and non-decreasing on \((0, 1)\), since \(c \in (0,1)\), \(f'(c) \geq 0 \Rightarrow h(c) = \frac{f'(c)}{c-1}\leq 0 \), which fulfills property (ii).

By Lemma \ref{lemma:rmax-derivative}, \(R'_{max}(c) = f(c) - g(c)\) for all \(c \in (0, 1)\). By weak Best Effort property, since \(R'_{max}\) is non-decreasing on \((a, b)\), \(R'_{max}(c) \geq 0\) for all \(c \in (a, b)\). Therefore, since \((a, b) \subseteq (0, 1)\), \(f(c) \geq g(c)\) for all \(c \in (a, b)\). 

Take any \(x, y \in (a, b)\). By weak Interpretability property, $f$ is non-decreasing and $g$ is non-increasing.

Case 1: \(x \leq y\).
\(f(x) \geq g(x) \geq g(y)\)

Case 2: \(x > y\).
\(f(x) \geq f(y) \geq g(y)\)

Therefore, \(f(x) \geq g(y)\) for all \(x, y\in (a, b)\). 

This implies \(\inf_{x \in (a,b)} f(x) \geq g(y)\) for all \(y \in (a, b)\). 

This implies \(\inf_{x \in (a,b)} f(x) \geq \sup_{y \in (a, b)} g(y)\).
 
By weak Interpretability property, since f is non-decreasing on \((a, b)\), \(\inf_{x \in (a,b)} f(x) = f(a^+)\).

By weak Interpretability property, since g is non-increasing on \((a, b)\), \(\sup_{y \in (a, b)} g(y) = g(a^+)\).

Therefore, \(f(a^+) \geq g(a^+)\), satisfying property (iii).

\( (\Leftarrow) \)

If the following three properties are satisfied, then we need to prove that the scheme is non-confidence hackable.

(i) There exists a function \(h(c)\) defined over \((0, 1)\) such that \(h(c)=\frac{f'(c)}{c-1}=\frac{g'(c)}{c}\)

(ii) \(h(c) \leq 0\) for all \(c \in (0, 1)\)

(iii) \(f(a^+) \geq g(a^+)\)

From property (i), \(h\) is well-defined. Therefore, by Lemma \ref{lemma:h-interpretability} and property (ii), weak Interpretability property holds.

Next, we aim to show the weak Proper Scoring property. Fix any \(p \in (0, 1)\) and let \(r(c) = E(c, p)\), where \(r\) is defined over \(c \in (0, 1)\). By Lemma \ref{lemma:partial-derivative-computation}, \(r'(c) = \frac{\partial E}{\partial c} = (c - p)h(c)\) for all \(c\) in \((0, 1)\).

Therefore, from property (ii), \(r'(c) \geq 0 \Leftrightarrow c \leq p\) and \(r'(c) \leq 0 \Leftrightarrow c \geq p\).

We aim to show that \(r(c) \leq r(p)\) for all \(c \in (0, 1)\).

Case 1: \(c \in (0, p]\)
Since \(r'(x) \geq 0\) for all \(x \in (0, p]\),
\(r(p) - r(c) = \int^{p}_{c} r'(x) \: \text{d}x \geq 0\).
This implies \(r(p) \geq r(c)\).

Case 2: \(c \in (p, 1)\)
Since \(r'(x) \leq 0\) for all \(x \in (p, 1)\),
\(r(c) - r(p) = \int^{c}_{p} r'(x) \: \text{d}x \leq 0\).
This implies \(r(c) \leq r(p)\).

In both cases, \(r(p)\) is an upper bound of \(r(c)\). Therefore, the maximum value of \(E\) when \(p\) is fixed is achieved when \(c\) is set to \(p\). Hence, the weak Proper Scoring property is satisfied.

Last, we aim to show the weak Best Effort property. By Lemma \ref{lemma:rmax-derivative}, since \(h\) is well-defined, \(E_{max}'(c) = f(c) - g(c)\) for all \(c \in (0, 1)\).

By weak Interpretability property, for all \(c \in (a, b)\), \(f(c) \geq \inf_{c' \in (a,b)} f(c') = f(a^+) \). Similarly, \( g(a^+) = \sup_{c' \in (a,b)} g(c') \geq g(c)\). Since property (iii) states that \(f(a^+) \geq g(a^+)\), we have \(f(c) - g(c) = E_{max}'(c) \geq 0\). Therefore, \(E_{max}\) is non-decreasing on \((a, b)\), which is the weak Best Effort property.

\end{proof}

\begin{corollary}\label{corollary:non-reward-hackable-interval}
If a reward scheme \((f(c), g(c))\) is a weakly non-hackable confidence reward scheme over \((a, b) \subseteq (0, 1)\), then it must be weakly non-hackable confidence over \((a, 1)\). 
\end{corollary}

\begin{proof}[\textbf{\textup{Proof of Corollary \ref{corollary:non-reward-hackable-interval}}}] 

\mbox{}

The corollary follows from the fact that properties (i), (ii) and (iii) in Theorem \ref{theorem:characterization-of-non-reward-hackable-schemes} do not rely on the endpoint \(b\). Here is an alternative proof which does not rely on Theorem \ref{theorem:characterization-of-non-reward-hackable-schemes}.

Since \((f(c), g(c))\) is a weakly non-hackable confidence reward scheme, it satisfies weak proper scoring, hence the corresponding function \(h\) is well-defined and Lemma \ref{lemma:rmax-derivative} applies. By weak Interpretability property, both \(f(c)\) and \(-g(c)\) are non-decreasing on \((0, 1)\). Therefore, using Lemma \ref{lemma:rmax-derivative}, \(R'_{max} = f(c) - g(c)\) is non-decreasing on \((0, 1)\). By the weak Best Effort property, \(E_{max}(c)\) is non-decreasing on \((a, b)\), which implies \(R'_{max}(c) \geq 0\) for all \(c \in (a, b)\).

Since \(R'_{max}\) is non-decreasing on \((0, 1)\) and is non-negative on \((a, b)\), \(R'_{max}\) must be non-negative on \((a, 1)\). Therefore, \(E_{max}(c)\) must be non-decreasing on \((a, 1)\), satisfying the weak Best Effort property.

The weak Interpretability and weak Proper Scoring properties are automatically satisfied as the properties remain the same when the interval is changed from \((a, b)\) to \((a, 1)\).

Therefore, \((f(c), g(c))\) is a weakly non-hackable confidence reward scheme on \((a, 1)\).

\end{proof}

We then proceed to derive the general form of weakly non-hackable confidence reward schemes. But we have to first show that certain improper integrals are convergent to establish well-definedness of the general form.

\begin{lemma}\label{lemma:comparison-theorem-nonstandard}
Let \(f_1, f_2\) be continuous on \((a, b]\), where \(a, b \in \mathbb{R}\) and \(0 \leq f_1(x) \leq f_2(x)\) for all \(x \in (a, b]\). If \(\int_a^b f_2(x) \: dx\) converges, then \(\int_a^b f_1(x) \: dx\) converges.
\end{lemma}

\begin{lemma}\label{lemma:limit-comparison-theorem-nonstandard}
Let \(f_1, f_2\) be continuous on \((a, b]\), where \(a, b \in \mathbb{R}\) and \(f_1(x), f_2(x) \geq 0\) for all \(x \in (a, b]\). If \(\lim_{x \rightarrow a^+} \frac{f_1(x)}{f_2(x)}\) exists and is positive and finite, then \(\int_a^b f_1(x) \: dx\) and \(\int_a^b f_2(x) \: dx\) either both converge or both diverge.
\end{lemma}

The proofs of Lemmas \ref{lemma:comparison-theorem-nonstandard} and \ref{lemma:limit-comparison-theorem-nonstandard} are analogous to the proofs of Theorem 15 and its corollary respectively in \citet[pp. 214-215]{calculusbook1978}.

\begin{lemma}\label{lemma:improper-integral-convergence-one-implies-all}
Let \(f_1\) be a continuous function on \((0, 1)\). If \(\int_0^c f_1(x) \: dx\) converges for some \(c \in (0, 1)\), then \(\int_0^p f_1(x) \: dx\) converges for all \(p \in (0, 1)\).
\end{lemma}

\begin{proof}[\textbf{\textup{Proof of Lemma \ref{lemma:improper-integral-convergence-one-implies-all}}}]

\mbox{}

We are given that \(\int_0^c f_1(x) \: dx\) converges for some \(c \in (0, 1)\). Take any \(p \in (0, 1)\). 

Case 1: \(p = c\)

Trivially proven as \(\int_0^c f_1(x) \: dx\) is known to converge.

Case 2: \(p < c\)

Since \(\int_0^c f_1(x) \: dx\) converges and \(f_1\) is continuous on \((0, c]\), both \(\int_p^c f_1(x) \: dx\) and \(\int_0^p f_1(x) \: dx\) must converge. 

Case 3: \(p > c\)

Since \(f_1\) is continuous on \([c, p]\), \(\int_c^p f_1(x) \: dx\) is a proper integral and therefore converges. 

Since \(\int_0^c f_1(x) \: dx\) converges by hypothesis, \(\int_0^p f_1(x) \: dx\) converges and is equal to \(\int_c^p f_1(x) \: dx\) + \(\int_0^c f_1(x) \: dx\).

Therefore, combining the three cases, \(\int_0^p f(x) \: dx\) converges for all \(p \in (0, 1)\).

\end{proof}

\begin{corollary}\label{corollary:improper-integral-convergence-both-ways}
Let \(h(c)\) be any continuous function on \((0, 1)\) such that \(h(c) \leq 0\) for all \(c \in (0, 1)\). Then, \(\int_0^p h(c) \: dc\) converges for all \(p \in (0, 1) \Leftrightarrow \int_0^p (c-1)h(c) \: dc\) converges for all \(p \in (0, 1)\).
\end{corollary}

\begin{proof}[\textbf{\textup{Proof of Corollary \ref{corollary:improper-integral-convergence-both-ways}}}] 

\mbox{}

Let \(f_1\) and \(f_2\) be defined on \((0, 1)\) such that \(f_1(c) = (c-1)h(c)\) and \(f_2(c) = -h(c)\).

Since \(h\) is continuous on \((0, 1)\), \(f_1\) and \(f_2\) are continuous on \((0, 1)\). 

Since \(h(c) \leq 0\) and \(c - 1 < 0\) for all \(c \in (0, 1)\), \(f_1(c) \geq 0\) and \(f_2(c) \geq 0\) for all \(c \in (0, 1)\).

\(\lim_{c \rightarrow 0^+} \frac{f_1(c)}{f_2(c)} = \lim_{c \rightarrow 0^+} (1-c) = 1\) is positive and finite.

Therefore, by Lemma \ref{lemma:limit-comparison-theorem-nonstandard}, \(\int_0^p f_1(c) \: dc\) and \(\int_0^p f_2(c) \: dc\) either both converge or both diverge for all \(p \in (0, 1)\). Since \(h(c) = -f_2(c)\) and \((c-1)h(c) = f_1(c)\), it follows that \(\int_0^p h(c) \: dc\) and \(\int_0^p (c-1)h(c) \: dc\) either both converge or both diverge for all \(p \in (0, 1)\).

Note that since we did not rule out the possibility where \(\int_0^{p_1} h(c) \: dc\) converges but \(\int_0^{p_2} h(c) \: dc\) diverges for some \(p_1, p_2 \in (0, 1)\), we are not done yet.

Take any \(p \in (0, 1)\) and consider whether \(\int_0^p h(c) \: dc\) diverges.

Case 1: \(\int_0^p h(c) \: dc\) diverges

By Lemma \ref{lemma:improper-integral-convergence-one-implies-all}, \(\int_0^p h(c) \: dc\) must diverge for all \(p \in (0, 1)\).

Hence, \(\int_0^p (c-1)h(c) \: dc\) must diverge for all \(p \in (0, 1)\).

Case 2: \(\int_0^p f_1(x) \: dx\) converges

By Lemma \ref{lemma:improper-integral-convergence-one-implies-all}, \(\int_0^p h(c) \: dc\) must converge for all \(p \in (0, 1)\).

Hence, \(\int_0^p (c-1)h(c) \: dc\) must converge for all \(p \in (0, 1)\).

Therefore, combining the two cases, \(\int_0^p h(c) \: dc\) converges for all \(p \in (0, 1)\) \(\Leftrightarrow\) \(\int_0^p (c-1)h(c) \: dc\) converges for all \(p \in (0, 1)\)

Hence, \(\int_0^p h(c) \: dc\) either both converge or both diverge.

\end{proof}

\begin{corollary}\label{corollary:improper-integral-convergence-one-way}

Let \(h(c)\) be any continuous function on \((0, 1)\) such that \(h(c) \leq 0\) for all \(c \in (0, 1)\). Then, \(\int_0^c h(x) \: dx\) converges for all \(c \in (0, 1) \implies \int_0^c xh(x) \: dx\) converges for all \(c \in (0, 1)\).

\end{corollary}

\begin{proof}[\textbf{\textup{Proof of Corollary \ref{corollary:improper-integral-convergence-one-way}}}] 

\mbox{}

Suppose \(\int_0^c h(x) \: dx\) converges for all \(c \in (0, 1)\).

Since \(h(c)\) is continuous on \((0, 1)\), \(ch(c)\) is continuous on \((0, 1)\).

Since \(h(c) \leq 0\) for all \(c \in (0, 1)\) and \(\int_0^c h(x) \: dx\) converges for all \(c \in (0, 1)\), \(\int_0^c |h(x)| \: dx = \int_0^c -h(x) \: dx\) converges for all \(c \in (0, 1)\).

Since \(0 \leq |ch(c)| \leq |h(c)|\) and \(\int_0^c |h(x)| \: dx\) converges for all \(c \in (0, 1)\), by Lemma \ref{lemma:comparison-theorem-nonstandard}, \(\int_0^c |xh(x)| \: dx\) converges for all \(c \in (0, 1)\).

For all \(c \in (0, 1)\), since \(c > 0\) and \(h(c) \leq 0\), \(ch(c) \leq 0\). Therefore, \(\int_0^c xh(x) \: dx = \int_0^c -|xh(x)| \: dx\) converges for all \(c \in (0, 1)\).

\end{proof}

\begin{theorem}\label{theorem:non-hackable-confidence-general-form}
The general form of weakly non-hackable confidence reward schemes is
\((\int_0^c (x-1)h(x) \: dx + f_0, \int_0^c xh(x) \: dx + g_0)\), where the following conditions hold:

(i) \(h(c) \leq 0\) for all \(c \in (0, 1)\)

(ii) \(h\) is continuous on \((0, 1)\)

(iii) \(\int_0^c h(x) \: dx\) converges for all \(c \in (0, 1)\)

(iv) \(f_0, g_0\) are constants in \(\mathbb{R}\) satisfying \(f_0 \geq g_0\)
\end{theorem}

\begin{remark}
A related theorem that describes the general form of reward schemes (not necessarily non-hackable confidence) that satisfy Proper Scoring property (both weak and strict variants) can be found in Theorem 4.2 of \cite{schervish1989general}.
\end{remark}

\begin{proof}[\textbf{\textup{Proof of Theorem \ref{theorem:non-hackable-confidence-general-form}}}] 

\mbox{}

\((\Rightarrow)\)

We aim to show that if a reward scheme \((f(c), g(c))\) is weakly non-hackable confidence, then it must be of the given form.

Since \((f(c), g(c))\) is weakly non-hackable confidence, by Theorem \ref{theorem:characterization-of-non-reward-hackable-schemes}, \(h(c)\) is well-defined and \(h(c) = \frac{f'(c)}{c-1} = \frac{g'(c)}{c}\), which implies that 
\(f'(c) = (c-1)h(c)\) and \(g'(c) = c h(c)\). Moreover, \(h(c) \leq 0\) for all \(c \in (0, 1)\), satisfying (i).

By weak Interpretability property, since \(g\) is non-increasing on \((0, 1)\), it is impossible for \(g(0^+)\) to diverge to \(-\infty\). Since \(f(0^+) \geq g(0^+)\), it is impossible for \(f(0^+)\) to diverge to \(-\infty\). Likewise, since \(f\) is non-decreasing on \((0, 1)\), it is impossible for \(f(0^+)\) to diverge to \(\infty\). Since \(f(0^+) \geq g(0^+)\), it is impossible for \(g(0^+)\) to diverge to \(\infty\). Therefore, the antiderivatives of \(f'\) and \(g'\) must have well-defined real number limits at \(0^+\). 

Since \(f\) and \(g\) are assumed to have continuous first order derivatives, \(h(c)\) is continuous on \((0, 1)\), satisfying (ii). 

Therefore, \(f'(c) = (c-1)h(c)\) and \(g'(c) = ch(c)\) are continuous on \((0, 1)\). This implies that \(\int_0^c (x-1)h(x) \: dx\) and \(\int_0^c xh(x) \: dx\) are well-defined and must converge for all \(c \in (0, 1)\). Hence, \(f(c) = \int_0^c (x-1)h(x) \: dx + f_0\) and \(g(c) = \int_0^c xh(x) \: dx + g_0\) for some real constants \(f_0\) and \(g_0\), satisfying the general form.

By Corollary \ref{corollary:improper-integral-convergence-both-ways}, since \(\int_0^c (x-1)h(x) \: dx\) converges for all \(c \in (0, 1)\), \(\int_0^c h(x) \: dx\) must converge for all \(c \in (0, 1)\), satisfying (iii).

Since \(f(0^+) \geq g(0^+)\) by Theorem \ref{theorem:characterization-of-non-reward-hackable-schemes}, \(\lim_{c\rightarrow0^+} (\int_0^c (x-1)h(x) \: dx + f_0) \geq \lim_{c\rightarrow0^+} (\int_0^c xh(x) \: dx + g_0)\), which implies that \(f_0 \geq g_0\) since both integrals converge, satisfying (iv).

\((\Leftarrow)\)

We want to show that if \((f(c), g(c)) = (\int_0^c (x-1)h(x) \: dx + f_0, \int_0^c xh(x) \: dx + g_0)\) such that the given constraints on \(h\), \(f_0\) and \(g_0\) are satisfied, then \((f(c), g(c))\) is weakly non-hackable confidence.

Since \(h(c) \leq 0\) for all \(c \in (0, 1)\) from (i), \(h\) is continuous on \((0, 1)\) from (ii) and\(\int_0^c h(x) \: dx\) converges for all \(c \in (0, 1)\) from (iii), it follows that \(\int_0^c (x-1)h(x) \: dx\) converges for all \(c \in (0, 1)\) by Corollary \ref{corollary:improper-integral-convergence-both-ways} and \(\int_0^c xh(x) \: dx\) converges for all \(c \in (0, 1)\) by Corollary \ref{corollary:improper-integral-convergence-one-way}.

Therefore, both \(f\) and \(g\) are well-defined functions.

Since \(h\) is continuous on \((0, 1)\) from (ii), \(f'(c) = (c-1)h(c)\) and \(g'(c) = ch(c)\) are both continuous on \((0, 1)\), hence both \(f\) and \(g\) have continuous first derivatives on \((0, 1)\).

\(f(0^+) = \lim_{c \rightarrow 0 } (\int_0^c (x-1)h(x) \: dx + f_0) = f_0\)

\(g(0^+) = \lim_{c \rightarrow 0 } (\int_0^c xh(x) \: dx + g_0) = g_0\)

Since \(f_0 \geq g_0\) from (iv), we have \(f(0^+) \geq g(0^+)\). In addition, we also have \(h(c) \leq 0\) for all \(c \in (0, 1)\), therefore, \((f(c), g(c))\) is weakly non-hackable confidence on \((0, 1)\).

\end{proof}

\begin{restatable}{theorem}{strictlynonhackable}\label{theorem:strict-non-hackable}
If a weakly non-hackable confidence reward scheme \((f(c), g(c))\) further satisfies \(h(c) < 0\) for all \(c \in (0, 1)\), then \((f(c), g(c))\) is strictly non-hackable confidence.
\end{restatable} 

\begin{proof}[\textbf{\textup{Proof of Theorem \ref{theorem:strict-non-hackable}}}] 

\mbox{}

By Remark \ref{remark:unimodal-R-argmax}, since \(h(c) < 0\) for all \(c \in (0, 1)\), \(\text{argmax}_{c} E(c, p) = \{p\}\), satisfying the Strict Proper Scoring property.

From Definition \ref{definition:h}, \(f'(c) = (c-1)h(c)\) and \(g'(c) = ch(c)\). Since \(c-1 < 0\) and \(c > 0\) for all \(c \in (0, 1)\), \(f'(c) > 0\) and \(g'(c) < 0\) for all \(c \in (0, 1)\). Therefore, \(f\) is strictly increasing and \(g\) is strictly decreasing on \((0, 1)\), satisfying the Strict Interpretability property.

In addition, by Lemma \ref{lemma:rmax-derivative}, since \(h\) is well-defined by definition of weakly non-hackable confidence reward schemes (Definition \ref{definition:weak-non-confidence-hackable-reward-scheme}), \( E_{max}'(c) = f(c) - g(c) \) for all \( c \in (0, 1) \). Therefore, \(E_{max}'\) is strictly increasing on \((0, 1)\).

Since weak Best Effort property is satisfied in weakly non-hackable confidence reward schemes, \(E_{max}'(c) \geq 0\) for all \(c \in (0, 1)\). Suppose \(E_{max}'(\epsilon) = 0\) for some \(\epsilon \in (0, 1)\). Then \(\frac{\epsilon}{2} \in (0, 1)\) and \(\frac{\epsilon}{2} < \epsilon\), which implies that \(E_{max}'(\frac{\epsilon}{2}) < 0\) since \(E_{max}'\) is strictly increasing. This contradicts the weak Best Effort Property. Hence, \(E_{max}'(c) > 0\) for all \(c \in (0, 1)\). This implies that \(E_{max}\) is strictly increasing on \((0, 1)\), satisfying the Strict Best Effort property.

\end{proof}

\begin{remark}

Combining Theorems \ref{theorem:characterization-of-non-reward-hackable-schemes} and \ref{theorem:strict-non-hackable} leads to Theorem \ref{theorem:main-theorem} in the main text.

\end{remark}

\subsection{Alternative Construction of Non-hackable Confidence Reward Schemes}\label{appendix:decision-making-construction}

As mentioned in Section \ref{section:related-work}, \cite{wu2026mitigatingllmhallucinationbehaviorally} and \cite{wu2026basdecisiontheoreticapproachevaluating} proposed reward functions based on the behavioral calibration framework of \cite{kalai2025languagemodelshallucinate}.

In this subsection, we show that the abstention-based risk-thresholding framework leads to a weakly non-confidence hackable reward scheme and can be used to construct all weakly non-hackable confidence reward schemes such that \(f(0^+) = g(0^+) = 0\).

In the behavioral calibration framework of \cite{kalai2025languagemodelshallucinate}, the LLM is awarded 1 point for a correct answer, 0 for an abstention, and \(-\frac{t}{1-t}\) points for an incorrect answer. This is formally described in Definition \ref{definition:abstention-reward-scheme} following the convention in \cite{wu2026basdecisiontheoreticapproachevaluating}, and termed ``abstention reward function''. 

\begin{remark}
Note that \cite{kalai2025languagemodelshallucinate} uses different terminology. The version that we use is termed ``error-penalty score'' in \cite{kalai2025languagemodelshallucinate} but the version where rewards are linearly scaled to the interval \([0, 1]\) is termed ``abstention reward function''. 
\end{remark}

\begin{definition}\label{definition:abstention-reward-scheme}
Let \(\text{isAbstention}(q, a)\) be 1 if the LLM answer \(a\) abstains from answering question \(q\) and 0 otherwise. Let \(\text{isCorrect}(q, a)\) be equal to 1 if answer \(a\) is a correct answer to question \(q\) and 0 otherwise. Then, for any \(t \in [0, 1)\), the abstention reward function with confidence threshold \(t\) is defined as follows:  

\(
R_{abs}(q, a) = 
\begin{cases}
1, & \text{if} \: \text{isCorrect}(q, a) = 1 \wedge \text{isAbstention}(q, a) = 0\\
0, & \text{if} \:\text{isAbstention}(q, a) = 1\\
-\frac{t}{1-t}, & \text{if} \: \text{isCorrect}(q, a) = 0 \wedge \text{isAbstention}(q, a) = 0\\
\end{cases}
\)
\end{definition}

\begin{lemma}\label{lemma:abstention-threshold}
Assuming that the confidence output \(c\) of the LLM is well-calibrated, under the abstention reward function with confidence threshold \(t\), it is optimal on expectation for the LLM to answer if \(c \geq t\) and to abstain if \(c < t\). 
\end{lemma}

\begin{proof}[\textbf{\textup{Proof of Lemma \ref{lemma:abstention-threshold}}}] 

\mbox{}

Note that \(\mathbb{E}[\text{LLM answers}] = c - \frac{(1-c)t}{1-t}\) and \(\mathbb{E}[\text{LLM abstains}] = 0\) since we assume that the confidence outputs are well-calibrated.

Since \(1-t \in (0, 1]\), 

\(\mathbb{E}[\text{LLM answers}] \geq \mathbb{E}[\text{LLM abstains}]\)

\(\Leftrightarrow c - \frac{(1-c)t}{1-t} \geq 0 \)

\(\Leftrightarrow c \geq \frac{(1-c)t}{1-t}\)

\(\Leftrightarrow c \geq \frac{t}{1-t} - \frac{ct}{1-t}\)

\(\Leftrightarrow c + \frac{ct}{1-t} \geq \frac{t}{1-t}\)

\(\Leftrightarrow c(1 + \frac{t}{1-t}) \geq \frac{t}{1-t}\)

\(\Leftrightarrow c(\frac{(1-t) + t}{1-t}) \geq \frac{t}{1-t}\)

\(\Leftrightarrow c(\frac{1}{1-t}) \geq \frac{t}{1-t}\)

\(\Leftrightarrow c \geq t\)

Therefore, it is optimal for the LLM to answer the question if \(c \geq t\) and abstain otherwise.

\end{proof}

We define a function \(w(t)\) of confidence thresholds, where \(w\) is defined, non-negative and continuous on \((0, 1)\) and \(\int_0^c w(t) \: dt\) converges for all \(c \in (0, 1)\). Note that this generalizes from \cite{wu2026basdecisiontheoreticapproachevaluating} by no longer restricting \(w(t)\) to be a probability distribution over \([0, 1)\).

Following \cite{wu2026basdecisiontheoreticapproachevaluating}, we aim to maximize the expected reward of the LLM under the abstention reward function with confidence threshold \(t\) weighted by \(w(t)\). Under the optimal strategy as described in Lemma \ref{lemma:abstention-threshold}, we derive the reward scheme \((f(c), g(c))\) as follows:

\(f(c) = \int^1_c 0 \: dt + \int^c_0 w(t) \: dt = \int^c_0 w(t) \: dt\)

\(g(c) = \int^1_c 0 \: dt + \int^c_0  w(t)\frac{-t}{1-t} \: dt = \int^c_0 w(t)\frac{t}{t-1} \: dt\)

Definition \ref{definition:weighted-abstention-confidence-threshold-reward-scheme} formally summarizes this notion as ``weighted abstention confidence threshold reward schemes''.

\begin{definition}\label{definition:weighted-abstention-confidence-threshold-reward-scheme}

\((f(c), g(c))\) is a weighted abstention confidence threshold reward scheme if there exists \(w(c)\) defined on \((0, 1)\) such that all of the following hold:

(i) \(w(c) \geq 0\) for all \(c \in (0, 1)\)

(ii) \(w\) is continuous on \((0, 1)\)

(iii) \(\int_0^c w(t) \: dt\) converges for all \(c \in (0, 1)\)

(iv) \(f(c) = \int^c_0 w(t) \: dt\)

(v) \(g(c) = \int^c_0 w(t)\frac{t}{t-1} \: dt\) 

\end{definition}

\begin{remark}
\cite{shuford1966admissible} derived a general form of reward schemes \((f(c), g(c))\) that satisfy Interpretability and Proper Scoring, but not necessarily Best Effort. Up to notational differences, the expressions of \(f(c)\) and \(g(c)\) are almost identical to weighted abstention confidence threshold reward schemes except for the integration endpoints for \(g(c)\). In this work, we aim to characterize weakly non-hackable confidence reward schemes in terms of weighted abstention confidence threshold reward schemes, which satisfies weak Best Effort in addition to weak Interpretability and weak Proper Scoring.
\end{remark}

Next, we proceed to prove Theorem \ref{theorem:abstention-framework-non-hackable-confidence} in the main text, and subsequently restate Remark \ref{remark:abstention-framework-strictly-non-hackable-confidence}. Theorem \ref{theorem:abstention-framework-non-hackable-confidence} has been restated as below for convenience.

\abstentionframeworknonhackableconfidence*

\begin{proof}[\textbf{\textup{Proof of Theorem \ref{theorem:abstention-framework-non-hackable-confidence}}}] 

\mbox{}

Since \(f'(c) = w(c)\) and \(g'(c) = w(c)\frac{c}{c-1}\) from (iv) and (v) respectively and \(w\) is continuous on \((0, 1)\) from (ii), \(f\) and \(g\) have continuous first derivatives on \((0, 1)\). Furthermore, \(\frac{f'(c)}{c-1} = \frac{w(c)}{c-1}\) and \(\frac{g'(c)}{c} = \frac{w(c)}{c-1}\). Hence, \(h\) is well defined and \(h(c) = \frac{w(c)}{c-1}\). 

Since \(w(c) \geq 0\) for all \(c \in (0, 1)\) from (i) and \(c - 1 < 0\) for all \(c \in (0, 1)\), \(h(c) \leq 0\) for all \(c \in (0, 1)\). Since \(w(c)\) is continuous on \((0, 1)\) from (ii), it follows that \(h(c)\) must also be continuous on \((0, 1)\).

By Corollary \ref{corollary:improper-integral-convergence-both-ways}, since \(w(c) = (c-1)h(c)\) and \(\int^c_0 w(t) \: dt\) converges for all \(c \in (0, 1)\) from (iii), \(\int^c_0 h(t) \: dt\) converges for all \(c \in (0, 1)\). Therefore, by Corollary \ref{corollary:improper-integral-convergence-one-way}, \(\int^c_0 th(t) \: dt\) converges for all \(c \in (0, 1)\).

Since \(w(c) = (c-1)h(c)\) and \(w(c)\frac{c}{c-1} = ch(c)\) for all \(c \in (0, 1)\), \((\int^c_0 w(t) \: dt, \int^c_0 w(t)\frac{t}{t-1} \: dt) = (\int^c_0 (t-1)h(t) \: dt, \int^c_0 th(t) \: dt)\) is a weakly non-hackable confidence reward scheme on \((0, 1)\) by Theorem \ref{theorem:non-hackable-confidence-general-form}.

\end{proof}

\abstentionframeworkstrictlynonhackableconfidence*

As the proof of Theorem \ref{theorem:abstention-framework-non-hackable-confidence} suggests, the formulation of the weighted abstention confidence threshold reward scheme has \(f_0 = g_0 = 0\), hence does not take into account possible additive constant rewards for answering correctly and answering incorrectly. Therefore, we generalize the notion of weighted abstention confidence threshold reward schemes in Definition \ref{definition:generalized-weighted-abstention-confidence-threshold-reward-scheme}.

\begin{definition}\label{definition:generalized-weighted-abstention-confidence-threshold-reward-scheme}

\((f(c), g(c))\) is a generalized weighted abstention confidence threshold reward scheme if there exists \(w(t)\) defined on \((0, 1)\) and \(f_0, g_0 \in \mathbb{R}\) such that all of the following hold:

(i) \(w(c) \geq 0\) for all \(c \in (0, 1)\)

(ii) \(w\) is continuous on \((0, 1)\)

(iii) \(\int_0^c w(t) \: dt\) converges for all \(c \in (0, 1)\)

(iv) \(f(c) = \int^c_0 w(t) \: dt + f_0\)

(v) \(g(c) = \int^c_0 w(t)\frac{t}{t-1} \: dt + g_0\) 

(vi) \(f_0 \geq g_0\)

\end{definition}

\begin{restatable}{theorem}{abstentionframeworknonhackableconfidenceparametrization}\label{theorem:abstention-framework-non-hackable-confidence-parametrization}
The set of weakly non-hackable confidence reward schemes is the set of generalized weighted abstention confidence threshold reward schemes.
\end{restatable}

\begin{proof}[\textbf{\textup{Proof of Theorem \ref{theorem:abstention-framework-non-hackable-confidence-parametrization}}}]

\mbox{}

\((\Rightarrow)\)

We aim to show weakly non-hackable confidence reward schemes are generalized weighted abstention confidence threshold reward schemes. Let \((f(c), g(c))\) be a weakly non-hackable confidence reward scheme.

By Theorem \ref{theorem:non-hackable-confidence-general-form}, \(f(c) = \int_0^c (t-1)h(t) \: dt + f_0\) and \(g(c) = \int_0^c th(t) \: dt + g_0\), where all of the following hold:

(a) \(h(c) \leq 0\) for all \(c \in (0, 1)\)

(b) \(h\) is continuous on \((0, 1)\)

(c) \(\int_0^c h(t) \: dt\) converges for all \(c \in (0, 1)\)

(d) \(f_0, g_0 \in \mathbb{R}\) such that \(f_0 \geq g_0\).

We aim to show that \((f(c), g(c))\) is a generalized weighted abstention confidence threshold reward scheme.

Let \(w(c) = (c-1)h(c)\) for all \(c \in (0, 1)\). Since \(h(c) \leq 0\) for all \(c \in (0, 1)\) from (a) and \(c - 1 < 0\) for all \(c \in (0, 1)\), \(w(c) \geq 0\) for all \(c \in (0, 1)\), satisfying (i) in Definition \ref{definition:generalized-weighted-abstention-confidence-threshold-reward-scheme}.

Since \(h\) is continuous on \((0, 1)\) from (b), \(w\) is continuous on \((0, 1)\), satisfying (ii) in Definition \ref{definition:generalized-weighted-abstention-confidence-threshold-reward-scheme}.

Since in addition, \(\int_0^c h(t) \: dt\) converges for all \(c \in (0, 1)\) from (c), by Corollary \ref{corollary:improper-integral-convergence-both-ways}, \(\int_0^c (t-1)h(t) \: dt\) converges for all \(c \in (0, 1)\) from (c), satisfying (iii) in Definition \ref{definition:generalized-weighted-abstention-confidence-threshold-reward-scheme}. Therefore, by Corollary \ref{corollary:improper-integral-convergence-one-way}, \(\int^c_0 th(t) \: dt\) converges for all \(c \in (0, 1)\). This ensures both \(f\) and \(g\) are well-defined on \((0, 1)\).

\(f(c) = \int_0^c (t-1)h(t) \: dt + f_0 = \int_0^c w(t) \: dt + f_0\), satisfying (iv) in Definition \ref{definition:generalized-weighted-abstention-confidence-threshold-reward-scheme}.

For all \(c \in (0, 1)\), since \(w(c) = (c-1)h(c)\), \(h(c) = \frac{w(c)}{c-1}\). Therefore, 

\(g(c) = \int_0^c th(t) \: dt + g_0 = \int_0^c w(t)\frac{t}{t-1} \: dt + g_0\), which satisfies (v) in Definition \ref{definition:generalized-weighted-abstention-confidence-threshold-reward-scheme}.

(vi) in Definition \ref{definition:generalized-weighted-abstention-confidence-threshold-reward-scheme} is automatically satisfied from (d).

Therefore, \((f(c), g(c))\) is a generalized weighted abstention confidence threshold reward scheme.

\((\Leftarrow)\)

We aim to show that generalized weighted abstention confidence threshold reward schemes are weakly non-hackable confidence reward schemes. Let \((f(c), g(c))\) be a generalized weighted abstention confidence threshold reward scheme.

Since \(f'(c) = w(c)\) and \(g'(c) = w(c)\frac{c}{c-1}\) from (iv) and (v) in Definition \ref{definition:generalized-weighted-abstention-confidence-threshold-reward-scheme} respectively and \(w\) is continuous on \((0, 1)\) from (ii) in Definition \ref{definition:generalized-weighted-abstention-confidence-threshold-reward-scheme}, \(f\) and \(g\) have continuous first derivatives on \((0, 1)\). Furthermore, \(\frac{f'(c)}{c-1} = \frac{w(c)}{c-1}\) and \(\frac{g'(c)}{c} = \frac{w(c)}{c-1}\). Hence, \(h\) is well defined and \(h(c) = \frac{w(c)}{c-1}\). 

Since \(w(c) \geq 0\) for all \(c \in (0, 1)\) from (i) in Definition \ref{definition:generalized-weighted-abstention-confidence-threshold-reward-scheme} and \(c - 1 < 0\) for all \(c \in (0, 1)\), \(h(c) \leq 0\) for all \(c \in (0, 1)\). Since \(w(c)\) is continuous on \((0, 1)\) from (ii) in Definition \ref{definition:generalized-weighted-abstention-confidence-threshold-reward-scheme}, it follows that \(h(c)\) must also be continuous on \((0, 1)\).

By Corollary \ref{corollary:improper-integral-convergence-both-ways}, since \(w(c) = (c-1)h(c)\) and \(\int^c_0 w(t) \: dt\) converges for all \(c \in (0, 1)\) from (iii) in Definition \ref{definition:generalized-weighted-abstention-confidence-threshold-reward-scheme}, \(\int^c_0 h(t) \: dt\) converges for all \(c \in (0, 1)\). Therefore, by Corollary \ref{corollary:improper-integral-convergence-one-way}, \(\int^c_0 th(t) \: dt\) converges for all \(c \in (0, 1)\).

Since \(f_0 \geq g_0\) from (vi) in Definition \ref{definition:generalized-weighted-abstention-confidence-threshold-reward-scheme}, \(w(c) = (c-1)h(c)\) for all \(c \in (0, 1)\) and \(w(c)\frac{c}{c-1} = ch(c)\) for all \(c \in (0, 1)\), \((\int^c_0 w(t) \: dt + f_0, \int^c_0 w(t)\frac{t}{t-1} \: dt + g_0) = (\int^c_0 (t-1)h(t) \: dt + f_0, \int^c_0 th(t) \: dt + g_0)\) is a weakly non-hackable confidence reward scheme on \((0, 1)\) by Theorem \ref{theorem:non-hackable-confidence-general-form}.

\end{proof}

\begin{remark}\label{remark:generalized-abstention-framework-strictly-non-hackable-confidence}
If in addition \(w(c) > 0\) for all \(c \in (0, 1)\), then the generalized weighted abstention confidence threshold reward scheme \((\int^c_0 w(t) \: dt + f_0, \int^c_0 w(t)\frac{t}{t-1} \: dt + g_0)\) is a strictly non-hackable confidence reward scheme for the same reason as outlined in Remark \ref{remark:abstention-framework-strictly-non-hackable-confidence}. 
\end{remark}

As shown in Theorem \ref{theorem:abstention-framework-non-hackable-confidence-parametrization}, generalized weighted abstention confidence threshold reward schemes form an alternative parametrization of weakly non-hackable confidence reward schemes, providing an alternative perspective from the decision theoretic point of view.

The scenario in \cite{wu2026basdecisiontheoreticapproachevaluating} is a special case with all of the following additional constraints:

(i) \(w(t)\) has a well-defined continuous extension at \(0^+\) by setting \(w(0) = w(0^+)\).

(ii) \(f_0 = g_0 = 0\)

(iii) \(\int_0^1 w(t) \: dt = 1\)

\subsection{Discrete Confidence Values}\label{appendix:discrete-confidence-values}

In practice, confidence values can be expressed as discrete values, such as an integer between 0 and 100 inclusive, and then normalized in the range \((0, 1)\). In this section, we demonstrate that the LLM can be incentivized to output the nearest representable confidence value just above or just below the true probability of answering correctly if the reward scheme satisfies mild conditions related to proper scoring. Therefore, if the reward scheme is non-hackable confidence and additional mild conditions are met, the effect of discretization of allowable output confidence values on the proper scoring incentive will be limited to the precision of the allowable discretized confidence values.

\begin{lemma}\label{lemma:unimodal-R}
For any fixed \(p \in (0, 1)\), if the reward scheme satisfies both the weak Interpretability and the weak Proper Scoring property, \(E(c, p)\) is non-decreasing on \(c \in (0, p)\) and non-increasing on \(c \in (p, 1)\). 
\end{lemma}

\begin{proof}[\textbf{\textup{Proof of Lemma \ref{lemma:unimodal-R}}}] 

\mbox{}

Since the reward scheme satisfies the weak Proper Scoring property, by Lemma \ref{lemma:derivation-of-h}, \(h\) is well-defined.

Since the reward scheme additionally satisfies the weak Interpretability property, by Lemma \ref{lemma:h-interpretability}, \(h\) is non-positive.

By Lemma \ref{lemma:partial-derivative-computation}, \(\frac{\partial}{\partial c}E(c, p) = (c-p)h(c)\).

For all \(c \in (0, p)\), \(c - p < 0 \Rightarrow \frac{\partial}{\partial c}E(c, p) \geq 0\).

For all \(c \in (p, 1)\), \(c - p > 0 \Rightarrow \frac{\partial}{\partial c}E(c, p) \leq 0\).

This implies that keeping \(p\) fixed, \(E(c, p)\) is non-decreasing on \(c \in (0, p)\) and non-increasing on \(c \in (p, 1)\).

\end{proof}

\begin{lemma}\label{lemma:unimodal-R-strict}
For any fixed \(p \in (0, 1)\), if the reward scheme has well-defined \(h\) such that \(h(c) < 0\) for all \(c \in (0, 1)\), then \(E(c, p)\) is strictly increasing on \(c \in (0, p)\) and strictly decreasing on \(c \in (p, 1)\).
\end{lemma}

\begin{proof}[\textbf{\textup{Proof of Lemma \ref{lemma:unimodal-R-strict}}}] 

\mbox{}

By Lemma \ref{lemma:partial-derivative-computation}, \(\frac{\partial}{\partial c}E(c, p) = (c-p)h(c)\).

For all \(c \in (0, p)\), \(c - p < 0 \Rightarrow \frac{\partial}{\partial c}E(c, p) > 0\).

For all \(c \in (p, 1)\), \(c - p > 0 \Rightarrow \frac{\partial}{\partial c}E(c, p) < 0\).

This implies that keeping \(p\) fixed, \(E(c, p)\) is strictly increasing on \(c \in (0, p)\) and strictly decreasing on \(c \in (p, 1)\).

\end{proof}

\begin{remark}\label{remark:unimodal-R-argmax}
Due to the continuity of \(E\), if the reward scheme satisfies both the weak Interpretability property and the weak Proper Scoring property, then Lemma \ref{lemma:unimodal-R} implies that \(p \in \text{argmax}_{c \in (0, 1)} E(c, p)\). Likewise, by Lemma \ref{lemma:unimodal-R-strict}, if \(h\) is well-defined and negative on \((0, 1)\), \(\{p\} = \text{argmax}_{c \in (0, 1)} E(c, p)\).
\end{remark}

Intuitively, Theorem \ref{theorem:unimodal-R-discrete} implies that even in discretization of confidences, for weakly non-hackable confidence reward schemes, the LLM is at least weakly incentivized to output a confidence value that is the closest to (either just lesser than or just greater than) the probability of answering the question correctly.

\begin{restatable}{theorem}{unimodalRdiscrete}\label{theorem:unimodal-R-discrete}

Suppose the LLM can only output confidences \(c \in (0, 1)\) in some non-empty finite set \(C_{valid}\), whose elements are in \((0, 1)\), and let \(p\) be any real number in \((0, 1)\). Let \(c_{min} = \min C_{valid}\), \(c_{max} = \max C_{valid}\), \(C_{\leq} = \{c \in C_{valid} | c \leq p\}\) and \(C_{>} = \{c \in C_{valid} | c > p\}\). Let \(c_{left} = \max C_{\leq} \) if \(C_{\leq} \neq \emptyset\), otherwise set to \(c_{min}\) and \(c_{right} = \min C_{>} \) if \(C_{>} \neq \emptyset\), otherwise set to \(c_{max}\).

Suppose the reward scheme satisfies both the weak Interpretability property and the weak Proper Scoring property, then an optimal value of \(c \in C_{valid}\) that maximizes \(E(c, p)\) is either \(c_{left}\) or \(c_{right}\).

\end{restatable}

\begin{proof}[\textbf{\textup{Proof of Theorem \ref{theorem:unimodal-R-discrete}}}] 

\mbox{}

We consider three cases based on where \(p \in (0, 1)\) lies relative to \(c_{min}\) and \(c_{max}\).

Case 1: \(p \in (0, c_{min})\)

In this case, \(c_{left} = c_{min}\) since \(C_{\leq}\) is empty and \(c_{right} = c_{min}\).

Since \(c_{min} \in (p, 1)\), all values in \(C_{valid}\) are in \((p, 1)\). Therefore, since \(E(c, p)\) is non-increasing on \(c \in (p, 1)\) by Lemma \ref{lemma:unimodal-R}, \(c_{min} \in \text{argmax}_{c \in C_{valid}} E(c, p)\).

Case 2: \(p \in [c_{min}, c_{max})\).

In this case, since \(C_{\leq}\) and \(C_{>}\) are both non-empty, \(c_{left} = \max C_{\leq} \) and \(c_{right} = \min C_{>} \).

If \(c_{left} = p\), then by Remark \ref{remark:unimodal-R-argmax}, \(c_{left} \in \text{argmax}_{c \in (0, 1)} E(c, p) \) implies \(c_{left} \in \text{argmax}_{c \in C_{valid}} E(c, p) \) since \(C_{valid} \subseteq (0, 1)\).

Otherwise, since all elements in \( C_{\leq} \) are in \( (0, p) \) and \(E(c, p)\) is non-decreasing on \(c \in (0, p)\) by Lemma \ref{lemma:unimodal-R}, \(c_{left} = \max C_{\leq} \in \text{argmax}_{c \in C_{\leq}} E(c, p)\).

Likewise, since all elements in \( C_{>} \) are in \( (p, 1) \) and \(E(c, p)\) is non-increasing on \(c \in (p, 1)\) by Lemma \ref{lemma:unimodal-R}, \(c_{right} = \min C_{>} \in \text{argmax}_{c \in C_{>}} E(c, p)\).

Since \(C_{valid} = C_{\leq} \cup C_{>}\), either \(c_{left} \in \text{argmax}_{c \in C_{valid}} E(c, p)\) or \(c_{right} \in \text{argmax}_{c \in C_{valid}} E(c, p)\).

Case 3: \(p \in [c_{max}, 1)\).

In this case, \(c_{right} = c_{max}\) since \(C_{>}\) is empty and \(c_{left} = c_{max}\).

If \(p = c_{max}\), then by Remark \ref{remark:unimodal-R-argmax}, \(c_{left} = c_{max} \in \text{argmax}_{c \in (0, 1)} E(c, p) \) implies \(c_{left} \in \text{argmax}_{c \in C_{valid}} E(c, p) \) since \(C_{valid} \subseteq (0, 1)\).

Otherwise, \(p \in (c_{max}, 1)\). Therefore, all values in \(C_{valid}\) are in \((0, p)\). Since \(E(c, p)\) is non-decreasing on \(c \in (0, p)\) by Lemma \ref{lemma:unimodal-R}, \(c_{max} \in \text{argmax}_{c \in C_{valid}} E(c, p)\).

\end{proof}

\begin{restatable}{remark}{discretestrictvariant}\label{remark:discrete-strict-variant}
Let \(c_{left}\) and \(c_{right}\) be as defined for the corresponding \(p\) in Theorem \ref{theorem:unimodal-R-discrete} and \(h\) be as defined in Theorem \ref{theorem:main-theorem}. Similarly, we can show that reward schemes with well-defined \(h\) such that \(h(c) < 0\) for all \(c \in (0, 1)\) have the property that for all \(p \in (0, 1)\), \(\text{argmax}_{c \in C_{valid}} E(c, p) \subseteq \{c_{left}, c_{right}\}\).
\end{restatable}

\begin{remark}
For our experiments, $C_{valid} = \left\{\frac{\tilde{c}+0.5}{101} | \tilde{c} \in \{0,1,\cdots,100\}\right\}$.
\end{remark}

\subsection{Scaling of Non-hackable Confidence Reward Schemes}

In this short subsection, we show that non-hackable confidence reward schemes (both weak and strict variants) can be linearly scaled subject to constraints on the multiplier.

\begin{theorem}\label{theorem:scaling-strict-non-hackable}
Let \(R_1(c) = (f_1(c), g_1(c))\) be a strictly non-hackable confidence reward scheme, and \(a, b\) be real numbers such that \(a > 0\). Then, \(R_2(c) = (f_2(c), g_2(c)) = (af_1(c) + b, ag_1(c) + b)\) is a strictly non-hackable confidence reward scheme.
\end{theorem}

\begin{proof}[\textbf{\textup{Proof of Theorem \ref{theorem:scaling-strict-non-hackable}}}]

\mbox{}

For all \(i \in {1, 2}\), let \(E_i(c, p) = pf_i(c) + (1-p)g_i(c)\) and \(E_{i, max}(c) = E_i(c, c)\).

Since \(R_1\) satisfies Strict Interpretability Property, \(f_1(c)\) is strictly increasing in \((0, 1)\) and \(g_1(c)\) is strictly decreasing in \((0, 1)\). Therefore, since \(a > 0\), \(af_1(c)\) is strictly increasing in \((0, 1)\) and \(ag_1(c)\) is strictly decreasing in \((0, 1)\). Hence, \(af_1(c) + b\) is strictly increasing in \((0, 1)\) and \(ag_1(c) + b\) is strictly decreasing in \((0, 1)\). This implies that \(R_2(c)\) satisfies the Strict Interpretability property.

\begin{flalign*}
E_2(c, p) & = pf_2(c) + (1-p)g_2(c) \\
& = p(af_1(c) + b) + (1-p)(ag_1(c) + b) \\
& = paf_1(c) + pb + (1-p)ag_1(c) + (1-p)b  \\
& = paf_1(c) + (1-p)ag_1(c) + b   \\
& = a(pf_1(c) + (1-p)g_1(c)) + b   \\
& = aE_1(c, p)+ b   \\
\end{flalign*}

Since \(R_1\) satisfies the Strict Proper Scoring property, for all \(c, p \in (0, 1)\) such that \(c \neq p\),

\begin{flalign*}
& \text{argmax}_{x \in (0, 1)} E_1(x, p) = \{p\} \\
& \Rightarrow E_1(p, p) > E_1(c, p) \\
& \Rightarrow aE_1(p, p) > aE_1(c, p) \text{ (since }a > 0\text{)} \\
& \Rightarrow aE_1(p, p) + b > aE_1(c, p) + b\\
& \Rightarrow E_2(p, p) > E_2(c, p) \\
\end{flalign*}

Therefore, \(\text{argmax}_{x \in (0, 1)} E_2(x, p) = \{p\}\), which implies that \(R_2\) satisfies the Strict Proper Scoring property.

Since \(R_1\) satisfies the Strict Best Effort property, \(E_{1, max}(c) = E_1(c, c)\) is strictly increasing on \(c \in (0, 1)\). Since \(a > 0\), \(aE_1(c, c)\) is strictly increasing on \(c \in (0, 1)\). Therefore, \(E_{2, max}(c) = E_2(c, c) = aE_1(c, c) + b\) is strictly increasing on \(c \in (0, 1)\), which implies that \(R_2\) satisfies the Strict Best Effort property.

Therefore, by Definition \ref{definition:non-confidence-hackable-reward-scheme}, \(R_2\) is strictly non-hackable confidence.

\end{proof}

\begin{remark}\label{remark:scaling-weak-non-hackable}
Let \(R_1(c) = (f_1(c), g_1(c))\) be a weakly non-hackable confidence reward scheme, and \(a, b\) be real numbers such that \(a \geq 0\). Then, using a similar proof to Theorem \ref{theorem:scaling-strict-non-hackable}, we can show that \(R_2(c) = (f_2(c), g_2(c)) = (af_1(c) + b, ag_1(c) + b)\) is a weakly non-hackable confidence reward scheme.
\end{remark}

\subsection{Hackable Confidence Reward Scheme Construction from Non-hackable Confidence Reward Schemes}\label{appendix:hackable-reward-schemes}

Depending on the application, it may be more responsible for the LLM to abstain from answering than to speculate the answer if the confidence of the answer is low, for instance, if it would lead to spreading of falsehoods.

In this section, we aim to construct reward schemes that are strictly non-hackable confidence on \((c, 1)\) for some threshold \(c \in (0, 1)\) but not strictly non-hackable confidence on \((p, 1)\) for all \(p \in (0, c)\), allowing for control of the confidence threshold \(p\) below which it is better to abstain than to attempt to answer. Theorem \ref{theorem:abstention-construction} shows our construction from a known non-hackable confidence reward scheme.

\begin{theorem}\label{theorem:abstention-construction}
Let \(R_1(c) = (f_1(c), g_1(c))\) be a strictly non-hackable confidence reward scheme. Then, \(R_2(c) = (f_2(c), g_2(c)) = (f_1(c) - f_1(p),  g_1(c) - g_1(p))\) is a strictly non-hackable confidence reward scheme on \((p, 1)\) but not non-hackable confidence on \((q, 1)\) for all \(q \in (0, p)\).

If the condition on \(R_1(c)\) is relaxed to be a weakly non-hackable confidence reward scheme, \(R_2(c)\) would be a weakly non-hackable confidence reward scheme on \((p, 1)\).
\end{theorem}

\begin{proof}[\textbf{\textup{Proof of Theorem \ref{theorem:abstention-construction}}}] 

Suppose \(R_1(c)\) is a strictly non-hackable confidence reward scheme.

By strict Interpretability property of \(R_1(c)\), \(f_2(c) = f_1(c) - f_1(p)\) is strictly increasing on \(c \in (0, 1)\) and \(g_2(c) = g_1(c) - g_1(p)\) is strictly decreasing on \(c \in(0, 1)\).

By Lemma \ref{lemma:rmax-derivative}, 
\begin{flalign*}
E_{2,max}'(c) & = f_2(c) - g_2(c) &\\
& = (f_1(c) - f_1(p)) - (g_1(c) - g_1(p)) &\\
& = f_1(c) - f_1(p) - g_1(c) + g_1(p)  \\
& = f_1(c) - g_1(c) - f_1(p) + g_1(p)  \\
& = (f_1(c) - g_1(c)) - (f_1(p) - g_1(p))  \\
& = E_{1,max}'(c) - E_{1,max}'(p)  \\
\end{flalign*}

By Remark \ref{remark:rmax-convex}, \(E_{1,max}'(c)\) is strictly increasing on \((0, 1)\). Therefore, \(E_{2,max}'(c) > 0 \) for all \(c \in (p, 1)\) but \(E_{2,max}'(c) < 0\) for all \(c \in (0, p)\). Hence, \(E_{2,max}(c)\) is increasing on \((p, 1)\) but decreasing on \((0, p)\), which implies that strict Best Effort property is satisfied on \((p, 1)\) but not on \((c, 1)\) for all \(c \in (0, p)\).

\begin{flalign*}
E_2(c, p) & = p f_2(c) + (1-p) g_2(c) &\\
& = p (f_1(c) - f_1(p)) + (1-p) (g_1(c) - g_1(p)) &\\
& = p f_1(c) + (1-p) g_1(c) - p f_1(p) - (1-p) g_1(p) \\
& = E_1(c, p) - p f_1(p) - (1-p) g_1(p) \\
\end{flalign*}

Since \(E_2(c, p)\) and \(E_1(c, p)\) differ by a constant independent of \(c\), the strict Proper Scoring property is maintained because \(\text{argmax}_{c \in (0,1)} E_1(c, p) = \text{argmax}_{c \in (0,1)} E_2(c, p)\).

Therefore, \(R_2\) is a strictly non-hackable confidence reward scheme on \((p, 1)\) but not weakly non-hackable confidence on \((c, 1)\) for all \(c \in (0, p)\).

Suppose we relax the condition and allow \(R_1(c)\) to be weakly non-hackable confidence.

First, \(f_2(p^+) = \lim_{c \rightarrow p^+} f_1(c) - f_1(p) = 0\) and \(g_2(p^+) = \lim_{c \rightarrow p^+} g_1(c) - g_1(p) = 0\).

Moreover, \(f_2'(c) = \frac{d}{dc} (f_1(c) - f_1(p)) = f_1'(c)\) and \(g_2'(c) = \frac{d}{dc} (g_1(c) - g_1(p)) = g_1'(c)\). Hence, \(h_2(c) = h_1(c) = \frac{f_1'(c)}{c-1} = \frac{g_1'(c)}{c}\) remains the same and is non-positive on \((0, 1)\). 

Therefore, by Theorem \ref{theorem:characterization-of-non-reward-hackable-schemes}, \(R_2(c)\) is weakly non-hackable confidence on \((p, 1)\).

\end{proof}

\subsection{Importance of Interpretability Property}\label{appendix:interpretability-redundant}

While \cite{damani2026beyond} considered the Proper Scoring and Best Effort properties, they did not consider the Interpretability property. Nonetheless, their proposed Brier-1 \citep{damani2026beyond} reward scheme satisfies the strict Interpretability property. As demonstrated in Theorem \ref{theorem:unimodal-R-discrete}, proper scoring reward schemes that also satisfy the Interpretability property encourage a response confidence value close to the true probability of answering the question correctly even when the allowed confidence values are discretized. This is useful as the difference in the optimal confidence value to the true probability of answering the question correctly can be bounded by the maximum gap of two consecutive allowable confidence values when the set of allowable confidence values is finite. For instance, this applies to settings where the LLM is tasked to output a raw unnormalized confidence value as an integer between 0 and 100 inclusive.

This section provides a simplified proof that reward schemes that satisfy the weak Proper Scoring property must also satisfy the weak Interpretability property, demonstrating that we did not further restrict the set of weakly non-hackable confidence reward schemes by inclusion of the weak Interpretability property. 

\begin{theorem}\label{theorem:interpretability-redundant}
Reward schemes \((f(c), g(c))\) that satisfy the weak Proper Scoring property must also satisfy the weak Interpretability property.
\end{theorem}

\begin{remark}
A more general version of Theorem \ref{theorem:interpretability-redundant}, which does not assume continuity of \(f(c)\) and \(g(c)\), and its corresponding proof, can be found in Lemma A.1 of \cite{schervish1989general}. This lemma also covers the strict case, i.e. strict Proper Scoring property implies strict Interpretability property. A different explanation of why Theorem \ref{theorem:interpretability-redundant} holds can be found in \cite{gneiting2007strictly}.
\end{remark}

\begin{proof}[\textbf{\textup{Proof of Theorem \ref{theorem:interpretability-redundant}}}]

\mbox{}

Since \((f(c), g(c))\) satisfies the weak Proper Scoring property, by Lemma \ref{lemma:derivation-of-h}, \(h\) is well-defined and \(h(c) = \frac{f'(c)}{c-1} = \frac{g'(c)}{c}\) for all \(c \in (0, 1)\). 

Suppose the weak Interpretability property is not satisfied. Then, by Lemma \ref{lemma:h-interpretability}, there exists \(p \in (0, 1)\) such that \(h(p) > 0\).

Since \(f\) and \(g\) are assumed to have continuous first derivatives on \((0, 1)\), \(h\) is continuous on \((0, 1)\). Therefore, there exist \(\epsilon > 0\) such that \(|h(c) - h(p)| < h(p)\) for all \(c \in (p-\epsilon, p+\epsilon) \). Note that since \(h(c)\) must be well-defined for all \(c \in (p - \epsilon, p+\epsilon)\), \(p-\epsilon, p+\epsilon \in [0, 1]\). Since \(\epsilon > 0\), \(0 \leq p - \epsilon < p - \frac{\epsilon}{2} < p < p + \epsilon \leq 1\).

Hence, for all \(c \in [p-\frac{\epsilon}{2}, p)\), \(|h(c) - h(p)| < h(p) \implies h(c) - h(p) > -h(p) \implies h(c) > 0 \).

Let \(r(c) = E(c, p)\). By Lemma \ref{lemma:partial-derivative-computation}, \(r'(c) = \frac{\partial}{\partial c} E(c, p) = (c-p)h(c) < 0\) for all \(c \in [p-\frac{\epsilon}{2}, p)\). 

Therefore, \(r(p) - r(c) = \int_c^p (c-p)h(c) < 0 \implies r(c) > r(p) \) for all \(c \in [p-\frac{\epsilon}{2}, p)\). Hence, \(p \notin \text{argmax}_c r(c) \implies p \notin \text{argmax}_c E(c, p)\). Therefore, we obtain a contradiction on the weak Proper Scoring property and all reward schemes that satisfy the weak Proper Scoring property must also satisfy the weak Interpretability property.

\end{proof}

\section{Brier-log Hybrid Naming} \label{appendix:brier-log-hybrid-naming}

This section of the appendix aims to justify why the Brier-log Hybrid reward scheme is named as such. To recall, Brier-log Hybrid reward scheme is the reward scheme \((f(c), g(c))\), where \(f(c) = c\) and \(g(c) = c+\ln(1-c) \). In short, the Brier-log Hybrid reward scheme resembles the log loss at higher confidences and a scaled version of Brier-1 score at lower confidences. 

\subsection{Calculations for Brier-1 reward scheme}

To understand the comparison with the Brier-1 reward scheme when the probability of answering correctly is low, we need to compute \(E_{max}(c)\) and \(E(c, p)\). To recall, the Brier-1 reward scheme, originally proposed (but not named as such) by \cite{damani2026beyond}, is defined as \((1-(1-c)^2, -c^2)\).

For the Brier-1 reward scheme, the maximum expected reward with probability of correctness \(c\) is

\(E_{max}(c)\)

\(= c(1 - (1-c)^2) + (1-c)(-c^2)\)

\(= c(1 - 1 + 2c - c^2) - c^2 + c^3\)

\(= 2c^2 - c^3 - c^2 + c^3\)

\(= c^2\). 

The expected reward with confidence \(c\) and probability of correctness \(p\) is

\(E(c, p)\)

\(= p(1-(1-c)^2) + (1-p)(-c^2)\)

\(= p(1 - 1 - 2c + c^2) -c^2 + pc^2\)

\(= - 2pc + pc^2 - c^2 + pc^2\)

\(= 2pc - c^2\)

\subsection{Low confidence}

When \(c \in (0, 1)\) is close to 0, the incorrectness reward is \(c + \ln(1-c) \approx c - c - \frac{c^2}{2} = -\frac{c^2}{2}\). The correctness reward is approximately \( c \).

The maximum possible expected reward is \(E_{max}(c) \approx c^2 + (1-c)(-\frac{c^2}{2}) \approx c^2 - \frac{c^2}{2} = \frac{c^2}{2}\), which is approximately half of that in the Brier-1 reward scheme.

Up to the second order terms, for the Brier-log Hybrid function, assuming both \(p\) and \(c\) are small, 

\(E(c, p)\)

\(= p(c) + (1-p)(c + \ln(1-c)) \)

\(\approx pc + (1-p)(c - c - \frac{c^2}{2}) \) 

\(= pc + (1 - p)(-\frac{c^2}{2}) \)

\(= pc - \frac{c^2}{2} + \frac{pc^2}{2} \)

\(\approx pc - \frac{c^2}{2} \)

which is approximately one-half of the corresponding value for Brier-1. 

\subsection{High confidence}

As observed by \cite{wu2026basdecisiontheoreticapproachevaluating}, both Brier-log Hybrid and Log Loss have a heavy overconfidence penalty. In this short subsection, we explain from an intuitive mathematical perspective how Brier-log Hybrid resembles the Log Loss reward scheme when the confidence \(c\) is high.

When \(c \in (0, 1)\) is close to 1, \(\ln c \approx 0\). Therefore, the correctness reward of \(c\) approximates \(c + \ln(c)\). The incorrectness reward of \(c + \ln(1-c)\) stays the same as Log Loss reward scheme except for an additional reward of \(c\). Hence, Brier-log Hybrid approximates Log Loss with an additional constant reward of \(c\) regardless of the correctness of the answer.

\section{Determination of whether reward schemes in experiments are non-hackable confidence} \label{appendix:non-hackable-determination}

For all reward schemes parametrized by \(k\), we assume that \(k\) is non-negative.

The reward schemes used in experiments all have \(h(c)\), as defined in Definition \ref{definition:h}, well defined. The corresponding values of \(h(c)\) are in Table \ref{table:h-list}. Note that all the \(h(c)\) in Table \ref{table:h-list} are non-positive for \(c \in (0, 1)\), assuming positive \(k\). Moreover, assuming positive \(k\), \(h(c) < 0\) for all \(c \in (0, 1)\) for all examined examined reward schemes except Correctness-only. Hence, by Theorem \ref{theorem:strict-non-hackable}, apart from Correctness-only, all other examined reward schemes are strictly non-hackable confidence over the same interval they are weakly non-hackable confidence.

\begin{table}
    \centering
    \caption{List of corresponding \(h(c)\) for each reward scheme}
    \begin{tabular}{cc}\toprule
         Reward Scheme
&  \(h(c)\)\\\midrule
Correctness-only &0\\ \midrule
 Log-\(k\)
& \(\frac{k}{c(c-1)}\)\\ \midrule
         Log Loss
&  \(\frac{1}{c(c-1)}\)\\ \midrule
         Brier-\(k\) 
&  \(-2k\)\\ \midrule
 Brier Score
& \(-2\)\\ \midrule
 Brier-Log Hybrid
& \(\frac{1}{c-1}\)\\ \midrule
 Overconfidence-\(k\)
& \(-\frac{k}{(\frac{1}{k}+c)((k+1)\ln(k+1)-k)}\)\\ \midrule
 Underconfidence-\(k\)& \(-\frac{1}{(\frac{1}{k}+(1-c))(1-\frac{\ln(k+1)}{k})}\)\\ \bottomrule\end{tabular}
    \label{table:h-list}
\end{table}

By Theorem \ref{theorem:characterization-of-non-reward-hackable-schemes}, it remains to check the values of \(a\) such that \(f(a^+) \geq g(a^+)\) holds

\subsection{Correctness-only}

Since \(f(c) = 1 > g(c) = 0\) for all \(c \in (0, 1)\), Correctness-only is a weakly, but not strictly, non-hackable confidence reward scheme on \((0, 1)\).

\subsection{Log-\(k\)}

Recall that for Log-\(k\), \(f(c) = 1 + k\ln(c)\) and \(g(c) = k\ln(1-c)\). 

\(f(c) - g(c)\)

\(= 1 + k\ln(c) - k\ln(1-c)\)

\(= 1 + k(\ln(c) - \ln(1-c))\)

\(= 1 + k\ln(\frac{c}{1-c})\)

Therefore, 
\(f(c) \geq g(c)\)

\(\Leftrightarrow 1 + k\ln(\frac{c}{1-c}) \geq 0\)

\(\Leftrightarrow k\ln(\frac{c}{1-c}) \geq -1\)

\(\Leftrightarrow \ln(\frac{c}{1-c}) \geq -\frac{1}{k}\)

\(\Leftrightarrow \frac{c}{1-c} \geq e^{-\frac{1}{k}} \)

\(\Leftrightarrow c \geq (1-c)e^{-\frac{1}{k}} \) (which holds since \(c \in (0, 1) \implies 1-c>0\))

\(\Leftrightarrow c \geq e^{-\frac{1}{k}} - ce^{-\frac{1}{k}} \) 

\(\Leftrightarrow c + ce^{-\frac{1}{k}} \geq e^{-\frac{1}{k}} \) 

\(\Leftrightarrow c (1 + e^{-\frac{1}{k}}) \geq e^{-\frac{1}{k}} \) 

\(\Leftrightarrow c \geq \frac{e^{-\frac{1}{k}}}{1 + e^{-\frac{1}{k}}} \) 

Therefore, Log-\(k\) is a strictly non-hackable confidence reward scheme on \(\left(\frac{e^{-\frac{1}{k}}}{1 + e^{-\frac{1}{k}}}, 1\right)\). An alternative explanation for why Log-\(k\) is a hackable confidence reward scheme on \((0, 1)\) can be found in \cite{damani2026beyond}.

In particular, when \(k = \frac{1}{\ln 202}\), \(e^{-\frac{1}{k}} = e^{-\ln 202} = \frac{1}{202}\). Therefore, Log-\(\frac{1}{\ln 202}\) is a strictly non-hackable confidence reward scheme on \((\frac{1}{203}, 1)\). For our experiments, the minimum possible representable confidence value is \(\frac{1}{202}\), hence confidence reward hacking cannot take place in Log-\(\frac{1}{\ln 202}\). In practice, with a sufficiently small \(k\), which need not be too small (e.g. less than \(0.1\)), confidence reward hacking is unlikely to take place.

\subsection{Log Loss}

Recall that \(f(c) = \ln(c)\) and \(g(c) = \ln(1-c)\).

Hence, \(f(c) \geq g(c) \Leftrightarrow \ln(c) \geq \ln(1-c) \Leftrightarrow c \geq 1-c \Leftrightarrow 2c \geq 1\).

Therefore, the Log Loss reward scheme is strictly non-hackable confidence only on \((0.5, 1)\).

\subsection{Brier-\(k\)}

Recall that for Brier-\(k\), \(f(c) = 1 - k(1 - c)^2\) and \(g(c) = -kc^2\). 

\(f(c) - g(c)\)

\(= 1 - k(1 - c)^2 - (-kc^2)\)

\(= 1 - k(1 - 2c + c^2) + kc^2\)

\(= 1 - k + 2kc - kc^2 + kc^2\)

\(= 1 - k + 2kc\)

Therefore, \(f(c) \geq g(c) \Leftrightarrow 1 - k + 2kc \geq 0 \Leftrightarrow 2kc \geq k-1 \Leftrightarrow c \geq \frac{k-1}{2k} \).

Hence, Brier-\(k\) is strictly non-confidence hackable on \((0, 1)\) for all \(k \in (0, 1]\) and on \((\frac{k-1}{2k}, 1)\) for all \(k \in (1, \infty)\).

An alternative explanation for why Brier-\(k\) is a weakly non-hackable confidence reward scheme over \((0, 1)\) if and only if \(k \in (0, 1]\) (excluding the Interpretability property) can be found in \cite{damani2026beyond}.

\subsection{Brier Score}

Recall that the Brier Score reward scheme has \(f(c) = -(1-c)^2\) and \(g(c) = -c^2\).

\(f(c) - g(c) = - (1-c)^2 - (-c^2) = - (1 - 2c + c^2) + c^2 = - 1 + 2c\)

Therefore, \(f(c) \geq g(c)\) if and only if \(c \in [0.5, 1)\). 

Hence, Brier score is a strictly non-hackable confidence reward scheme only over \((0.5, 1)\).

\subsection{Brier-log Hybrid}

Since \(f(c) - g(c) = -\ln(1-c)\) is always non-negative for all \(c \in (0, 1)\), Brier-Log Hybrid is strictly non-hackable confidence. 

\subsection{Overconfidence-\(k\)}\label{appendix:overconfidence-k-non-hackable-confidence}

Let \(l(k) = (k+1)\ln(k+1) - k\). Since \(l(0^+) = 0\) and \(l'(k) = \ln(k+1) + 1 - 1 = \ln(k+1) > 0\) for all positive \(k\), \(l(k) > 0\) for all positive \(k\).

Since \(f(c) - g(c) = \frac{k\ln(ck+1)}{(k+1)\ln(k+1)-k}\) is always non-negative for all \(c \in (0, 1)\), Overconfidence-\(k\) is strictly non-hackable confidence.

\subsection{Underconfidence-\(k\)}\label{appendix:underconfidence-k-non-hackable-confidence}

Let \(l(k) = k - \ln(1+k)\). Since \(l(0^+) = 0\) and \(l'(k) = 1 - \frac{1}{1+k} > 0\) for all positive \(k\), \(l(k)\) is positive for all positive \(k\). 

Therefore, \(f(c) - g(c) = \frac{-k\ln(1-\frac{kc}{1+k})}{k-\ln(1+k)}\) is always non-negative for all \(c \in (0, 1)\). Hence, Underconfidence-\(k\) is a strictly non-hackable confidence reward scheme on \((0, 1)\).

\section{Determination of Non-hackable Confidence Reward Schemes with Overconfidence and Underconfidence Bias} \label{appendix:confidence-bias}

Intuitively, overconfidence bias means that the LLM will receive a better reward when overconfident relative to the true probability of answering correctly by \(\delta\), compared to being underconfident by \(\delta\). When unsure about the well-calibrated subjective probability of answering correctly, LLM is incentivized to be overconfident about its correctness. 

Likewise, underconfidence bias means that the LLM will receive a better reward when underconfident relative to the true probability of answering correctly by \(\delta\), compared to being overconfident by \(\delta\). When unsure about the well-calibrated subjective probability of answering correctly, LLM is incentivized to be underconfident about its correctness. 

Definition \ref{definition:overconfidence-underconfidence-bias} is restated below. Note that for simplicity, in this section of the appendix, we only consider weakly non-hackable confidence reward schemes.

\definitionconfidencebias*

\cite{wu2026basdecisiontheoreticapproachevaluating} proposed Brier-log Hybrid and alluded to its underconfidence bias. In Appendix \ref{appendix:sufficient-condition-confidence-bias}, we generalize this finding by proving a sufficient condition for a non-hackable reward confidence scheme to exhibit overconfidence or underconfidence bias.

In Appendix \ref{appendix:reward-scheme-overconfidence-underconfidence}, we determine whether each of the non-reward hackable schemes examined in Appendix \ref{appendix:non-hackable-determination} exhibit overconfidence bias or underconfidence bias. Subsequently, for the rest of the section, we justify the ordering within the overconfidence to underconfidence spectrum as defined in Defintion \ref{definition:underconfidence-to-overconfidence-spectrum}.

\subsection{Sufficient Condition for Overconfidence/Underconfidence Bias} \label{appendix:sufficient-condition-confidence-bias}

For this subsection, we further assume that both the reward for answering correctly \(f(c)\) and the reward for answering incorrectly \(g(c)\) have continuous second-order derivatives. This allows for the expected reward function \(E(c, p)\) to have continuous second-order derivatives.

\begin{lemma}\label{lemma:confidence-bias-first-derivative}
Let \(r\) be a function on \((0, 1)\) such that \(r\) has continuous second-order derivatives on \((0, 1)\) and \(r''\) is strictly decreasing on \((0, 1)\). Suppose \(r\) has a maximum at \(p\). 

Then, for all \(\delta > 0\) such that \(p + \delta \in (0, 1)\) and \(p - \delta \in (0, 1)\), \(-r'(p+\delta) > r'(p-\delta)\).

\end{lemma}

\begin{proof}[\textbf{\textup{Proof of Lemma \ref{lemma:confidence-bias-first-derivative}}}]

\mbox{}

Take any \(\delta > 0\) such that such that \(p + \delta \in (0, 1)\) and \(p - \delta \in (0, 1)\).

Since \(r\) has continuous second-order derivatives in \((0, 1)\), \(r'\) must have continuous first-order derivatives in \((0, 1)\). Since \(r\) has a maximum at \(p\), \(r'(p) = 0\) and \(r''(p) \leq 0\). 

\(r'(p+\delta)  = r'(p+\delta) - r'(p) = \int^{p+\delta}_{p} r''(c) \: dc\)

\(-r'(p-\delta)  = r'(p) - r'(p-\delta) = \int^{p}_{p-\delta} r''(c) \: dc\)

Let \(u = c + \delta\). Then, when \(c = p\), \(u = p + \delta\); and when \(c = p - \delta\), \(u = p\). Moreover, \(du = dc\). Therefore, since \(r''(u-\delta) > r''(u)\) for all \(u \in [p, p+\delta]\),

\(-r'(p-\delta) = \int^{p+\delta}_{p} r''(u - \delta) \: du > \int^{p+\delta}_{p} r''(u) \: du = r'(p+\delta)\)

Hence, \(r'(p-\delta) < -r'(p+\delta)\), yielding the result in the lemma.

\end{proof}

\begin{lemma}\label{lemma:confidence-bias-function-value-decreasing}
Let \(r\) be a function on \((0, 1)\) such that \(r\) has continuous second-order derivatives on \((0, 1)\) and \(r''\) is strictly decreasing on \((0, 1)\). Suppose \(r\) has a maximum at \(p\). 

Then, for all \(\delta > 0\) such that \(p + \delta \in (0, 1)\) and \(p - \delta \in (0, 1)\), \(r(p)-r(p+\delta) > r(p)-r(p-\delta)\)

\end{lemma}

\begin{proof}[\textbf{\textup{Proof of Lemma \ref{lemma:confidence-bias-function-value-decreasing}}}]

\mbox{}

Take any \(\delta > 0\) such that such that \(p + \delta \in (0, 1)\) and \(p - \delta \in (0, 1)\).

For all \(x \in (0, \delta]\), since \(p - \delta \leq p - x < p < p + x \leq p + \delta\), both \(p - x\) and \(p + x\) are in \((0, 1)\).

Therefore, by Lemma \ref{lemma:confidence-bias-first-derivative}, \(-r'(p+x) > r'(p-x)\) for all \(x \in (0, \delta]\).

Note that \(r(p) - r(p + \delta) = \int_p^{p+\delta} -r'(c_+) \: dc_+\) and \(r(p) - r(p - \delta) = \int_{p-\delta}^{p} r'(c_-) \: dc_-\)

Let \(u_+ = c_+ - p\). Then, when \(c_+ = p\), \(u_+ = 0\) and when \(c_+ = p + \delta\), \(u_+ = \delta\). Moreover, \(du_+ = dc_+\).

let \(u_- = p - c_-\). Then, when \(c_- = p\), \(u_- = 0\) and when \(c_- = p - \delta\), \(u_- = \delta\). Moreover, \(du_- = - dc_-\).

Hence, \(r(p) - r(p + \delta) = \int_p^{p+\delta} -r'(c_+) \: dc_+ = \int_0^{\delta} -r'(p + u_+) \: du_+ \) and

\(r(p) - r(p - \delta) = \int_{p-\delta}^{p} r'(c_-) \: dc_- = \int_{\delta}^{0} -r'(p - u_-) \: du_- = \int_{0}^{\delta} r'(p - u_-) \: du_-\).

Since \(-r'(p+u) > r'(p-u)\) for all \(u \in (0, \delta]\), we conclude that \(\int_0^{\delta} -r'(p + u_+) \: du_+ > \int_{0}^{\delta} r'(p - u_-) \: du_-\), therefore, \(r(p) - r(p + \delta) > r(p) - r(p - \delta)\).
\end{proof}

Similarly, using the idea of the proofs of Lemmas \ref{lemma:confidence-bias-first-derivative} and \ref{lemma:confidence-bias-function-value-decreasing}, we can prove the following result:

\begin{lemma}\label{lemma:confidence-bias-function-value-increasing}
Let \(r\) be a function on \((0, 1)\) such that \(r\) has continuous second-order derivatives on \((0, 1)\) and \(r''\) is strictly increasing on \((0, 1)\). Suppose \(r\) has a maximum at \(p\). 

Then, for all \(\delta > 0\) such that \(p + \delta \in (0, 1)\) and \(p - \delta \in (0, 1)\), \(r(p)-r(p+\delta) < r(p)-r(p-\delta)\)

\end{lemma}

We can now proceed to prove Theorem \ref{theorem:confidence-bias-test} in the main paper. The theorem is restated for convenience.

\confidencebiastest*

\begin{proof}[\textbf{\textup{Proof of Theorem \ref{theorem:confidence-bias-test}}}]

\mbox{}

Since the proofs for (i) and (ii) are similar, we show only the proof of (i) here.

Fix any \(p \in (0, 1)\). Note that when \(p\) is fixed, by weak Proper Scoring property of weakly non-hackable confidence reward schemes in Definition \ref{definition:weak-non-confidence-hackable-reward-scheme}, \(E(c, p)\) has a maximum at \(c = p\).

Let \(r_p(c) = E(c, p)\) for all \(c \in (0, 1)\). Note that by Definition \ref{definition:non-confidence-hackable-reward-scheme}, \(E(c, p) = pf(c) + (1-p)g(c)\). Therefore, since \(f(c)\) and \(g(c)\) are assumed to have continuous second-order derivatives, \(E(c, p)\) has continuous second-order derivatives.

Since \(r_p\) has a maximum at \(c = p\) and \(r_p''(c) = \frac{\partial^2}{\partial c^2}(E(c, p))\) is strictly decreasing on \((0, 1)\), by Lemma \ref{lemma:confidence-bias-function-value-decreasing}, \(r_p(p) - r(p + \delta) > r_p(p) - r_p(p - \delta)\) for all \(\delta > 0\) such that \(p + \delta \in (0, 1)\) and \(p - \delta \in (0, 1)\).

Therefore, for all \(\delta > 0\) such that \(p + \delta \in (0, 1)\) and \(p - \delta \in (0, 1)\), \( E(p, p) - E(p + \delta, p) > E(p, p) - E(p - \delta, p)\), which implies that \(E_{pen}(p + \delta, p) > E_{pen}(p - \delta, p)\).

Therefore, \((f(c), g(c))\) has underconfidence bias.

\end{proof}

Theorem \ref{theorem:confidence-bias-test} provides a sufficient condition of overconfidence and underconfidence bias based on the second derivative of the expected reward function \(E(c, p)\), allowing us to show that Overconfidence-\(k\) has overconfidence bias and Underconfidence-\(k\) has underconfidence bias in the next subsection.

\subsection{Determination of Reward Schemes with Overconfidence/Underconfidence Bias} \label{appendix:reward-scheme-overconfidence-underconfidence}

In this subsection, we use Definition \ref{definition:overconfidence-underconfidence-bias} and Theorem \ref{theorem:confidence-bias-test} to determine if each of the weakly non-hackable confidence reward schemes we used for our experiments exhibits overconfidence bias or underconfidence bias.

\subsubsection{Correctness-only}

Since Correctness-only reward scheme has a reward that is independent of the confidence of the LLM, the miscalibration penalty is always zero. Therefore, the Correctness-only reward scheme has neither overconfidence nor underconfidence bias.

\subsubsection{Overconfidence-\(k\)}\label{appendix:overconfidence-k-confidence-bias-determination}

\(E(c, p)\)

\(= p f(c) + (1-p) g(c)\)

\(= p \frac{(k+1)\ln(ck+1)-ck}{(k+1)\ln(k+1)-k} + (1-p) \frac{\ln(ck+1)-ck}{(k+1)\ln(k+1)-k}\)

\(= \frac{p((k+1)\ln(ck+1)-ck) + (1-p)(\ln(ck+1)-ck)}{(k+1)\ln(k+1)-k}\)

\(= \frac{p(k+1)\ln(ck+1) + (1-p)\ln(ck+1) - ck}{(k+1)\ln(k+1)-k}\)

\(= \frac{(pk + p + 1 - p)(\ln(ck+1)) - ck}{(k+1)\ln(k+1)-k}\)

\(= \frac{(pk+1)\ln(ck+1) - ck}{(k+1)\ln(k+1)-k}\)

\(\frac{\partial E}{\partial c} = \frac{(pk+1)\frac{k}{ck+1} - k}{(k+1)\ln(k+1)-k}\)

\(\frac{\partial^2 E}{\partial c^2} = \frac{-(pk+1)\frac{k^2}{(ck+1)^2}}{(k+1)\ln(k+1)-k}\)

As shown in Appendix \ref{appendix:overconfidence-k-non-hackable-confidence}, \((k+1)\ln(k+1)-k\) is always positive for all \(k > 0\). Note that \(pk + 1 > 0\) and \(k^2 > 0\) since \(p\) is non-negative and \(k\) is positive. 

As \(c\) increases in \((0, 1)\), \(ck + 1\) increases and is positive, hence \((ck + 1)^2\) increases and is positive. This implies that \(\frac{k^2}{(ck + 1)^2}\) decreases and is positive.

Therefore, keeping \(p\) fixed, \(\frac{\partial^2 E}{\partial c^2}\) increases with respect to \(c\) on the interval \((0, 1)\). Hence, by Theorem \ref{theorem:confidence-bias-test}, Overconfidence-\(k\) has overconfidence bias.

\subsubsection{Brier-\(k\), where \(k \leq 1\)}

The expected reward with confidence \(c\) and probability of correctness \(p\) is

\(E(c, p)\)

\(= p(1-k(1-c)^2) + (1-p)(-kc^2)\)

\(= p(1 - k(1 - 2c + c^2)) - kc^2 + pkc^2\)

\(= p(1 - k + 2ck - c^2k) - kc^2 + pkc^2\)

\(= p - pk + 2pck - pc^2k - kc^2 + pkc^2\)

\(= p - pk + 2pck - kc^2\)

Therefore, the miscalibration penalty \(E_{pen}(c, p)\) is equal to

\(E(p, p) - E(c, p)\)

\(= (p - pk + 2p^2k - kp^2) - (p - pk + 2pck - kc^2)\)

\(= p - pk + 2p^2k - kp^2 - p + pk - 2pck + kc^2\)

\(= 2p^2k - kp^2 - 2pck + kc^2\)

\(= k(2p^2 - p^2 - 2pc + c^2)\)

\(= k(p^2 - 2pc + c^2)\)

\(= k(p - c)^2\)

Therefore, for all \(\delta > 0\) and \(p \in (0, 1)\) such that \(p + \delta \in (0, 1)\) and \(p - \delta \in (0, 1)\), \(E_{pen}(p + \delta, p) = k\delta^2 = E_{pen}(p - \delta, p)\). Therefore, Brier-\(k\) has neither overconfidence bias nor underconfidence bias.

\subsubsection{Underconfidence-\(k\)}\label{appendix:underconfidence-k-confidence-bias-determination}

\(E(c, p)\)

\(= pf(c) + (1-p)g(c)\)

\(= p\frac{kc + \ln(1-\frac{kc}{1+k})}{k - \ln(1+k)} + (1-p)\frac{kc + (k+1)\ln(1-\frac{kc}{1+k})}{k - \ln(1+k)} \)

\(= \frac{p(kc + \ln(1-\frac{kc}{1+k})) + (1-p)(kc + (k+1)\ln(1-\frac{kc}{1+k}))}{k - \ln(1+k)} \)

\(= \frac{pkc + p\ln(1-\frac{kc}{1+k}) + (1-p)kc + (1-p)(k+1)\ln(1-\frac{kc}{1+k})}{k - \ln(1+k)} \)

\(= \frac{kc + (p + (1-p)(k+1))\ln(1-\frac{kc}{1+k})}{k - \ln(1+k)} \)

\(= \frac{kc + (p + (1-p)k + (1-p))\ln(1-\frac{kc}{1+k})}{k - \ln(1+k)} \)

\(= \frac{kc + (1 + k-pk)\ln(1-\frac{kc}{1+k})}{k - \ln(1+k)} \)

\(\frac{\partial E}{\partial c}\)

\(= \frac{k + (1 + k-pk)\frac{-\frac{k}{1+k}}{1-\frac{kc}{1+k}}}{k - \ln(1+k)} \)

\(= \frac{k + (1 + k-pk)\frac{-k}{1 + k - kc}}{k - \ln(1+k)} \)

\(= \frac{k}{k - \ln(1+k)} + \frac{1 + k-pk}{k - \ln(1+k)}\frac{-k}{1 + k - kc}\)

\(\frac{\partial^2 E}{\partial c^2}\)

\(= \frac{1 + k-pk}{k - \ln(1+k)}\frac{(-k)(-1)(-k)}{(1 + k - kc)^2} \)

\(= \frac{1 + k - pk}{k - \ln(1+k)}\frac{-k^2}{(1 + k - kc)^2} \)

\(= \frac{1 + k(1-p)}{k - \ln(1+k)}\frac{-k^2}{(1 + k(1-c))^2} \)

Let \(p \in (0, 1)\) be fixed. As shown in Appendix \ref{appendix:underconfidence-k-non-hackable-confidence}, \(k - \ln(1+k) > 0\) for all \(k > 0\). Furthermore, since \(p \in (0, 1)\), \(1 + k(1-p) > 1 > 0\) and \(-k^2 < 0\). Likewise, for all \(c \in (0, 1)\), \(1 + k(1-c) > 1 > 0\). Therefore, \(\frac{\partial^2 E}{\partial c^2} < 0\) for all \(p, c \in (0, 1)\).

Keeping \(p\) fixed, since \(1 + k(1-c)\) decreases and is positive as \(c\) increases within the interval \((0, 1)\), \(\frac{1}{(1 + k(1-c))^2}\) increases as \(c\) increases within the interval \((0, 1)\). Therefore, \(\frac{\partial^2 E}{\partial c^2}\) is strictly decreasing with respect to \(c \in (0, 1)\). Therefore, by Theorem \ref{theorem:confidence-bias-test}, Underconfidence-\(k\) has underconfidence bias.

\subsubsection{Brier-log Hybrid}

\(E(c, p)\)

\(= pf(c) + (1-p)g(c)\)

\(= pc + (1-p)(c+\ln(1-c))\)

\(= pc + c+\ln(1-c) - pc - p\ln(1-c)\)

\(= c + \ln(1-c) - p\ln(1-c)\)

\(= c + (1-p)\ln(1-c)\)

\(\frac{\partial E}{\partial c} = 1 - \frac{1-p}{1-c}\)

\(\frac{\partial^2 E}{\partial c^2} = - \frac{1-p}{(1-c)^2}\)

For all \(p \in (0, 1)\), keeping \(p\) fixed, \(\frac{\partial^2 E}{\partial c^2}\) is strictly decreasing as \(c\) increases in the range \((0, 1)\). Therefore, by Theorem \ref{theorem:confidence-bias-test}, Brier-log Hybrid has underconfidence bias.

\subsection{Pointwise Continuity of Overconfidence to Underconfidence Spectrum}\label{appendix:overconfidence-underconfidence-spectrum-pointwise-continuity}

In this subsection, we index each function in the overconfidence to underconfidence spectrum to either a real number or \(\pm \infty\) and show that under any fixed confidence value \(c \in (0, 1)\) and with respect to the index, the reward schemes in the overconfidence to underconfidence spectrum are continuous.

\begin{definition}\label{definition:overconfidence-underconfidence-spectrum-mapping}
We define a mapping from the extended real number line to the set of reward schemes \(C(x)\).

Let \(x \in \mathbb{R} \cup \{\infty, -\infty\}\). Then, we denote \(C(x)\) as the following reward scheme:

\(C(x) = \begin{cases}
    \text{Correctness-only} & \text{if } x = -\infty \\
    \text{Overconfidence-\((-x)\)} & \text{if } x < 0 \text{ and } x \text{ is finite} \\
    \text{Brier-1} & \text{if } x = 0 \\
    \text{Underconfidence-\(x\)} & \text{if } x > 0 \text{ and } x \text{ is finite} \\
    \text{Brier-log Hybrid} & \text{if } x = +\infty
\end{cases}\)

Note that an increase in \(x\) is equivalent to moving towards the underconfident side of the spectrum.

We further denote \(Cf(x, c)\) and \(Cg(x, c)\) to represent, respectively, the reward for answering correctly \(f(c)\) and the reward for answering incorrectly \(g(c)\) for the reward scheme \(C(x)\).

\end{definition}

\begin{theorem}\label{theorem:overconfidence-underconfidence-spectrum-pointwise-continuity}
\(Cf(x, c)\) and \(Cg(x, c)\) are both continuous in \(x\) when keeping \(c \in (0, 1)\) fixed. The result also holds when taking \(x \rightarrow \pm \infty\), i.e. \(\lim_{x \rightarrow \infty} Cf(x, c) = Cf(\infty, c) \), \(\lim_{x \rightarrow -\infty} Cf(x, c) = Cf(-\infty, c) \) and likewise for \(Cg(x, c)\).
\end{theorem}

\begin{proof}[\textbf{\textup{Proof of Theorem \ref{theorem:overconfidence-underconfidence-spectrum-pointwise-continuity}}}]

\mbox{}

First, we explain why Overconfidence-\(k\) and Underconfidence-\(k\) are continuous in \(k\) when keeping \(c\) fixed.

For all \(k > 0\), \(Cf(-k, c) = \frac{(k+1)\ln(ck+1)-ck}{(k+1)\ln(k+1)-k}\), \(Cg(-k, c) = \frac{\ln(ck+1)-ck}{(k+1)\ln(k+1)-k}\), \(Cf(k, c) = \frac{kc+\ln(1-\frac{kc}{1+k})}{k-\ln(1+k)}\) and \(Cg(k, c) = \frac{kc+(k+1)\ln(1-\frac{kc}{1+k})}{k-\ln(1+k)}\) are all continuous. Therefore, \(Cf\) and \(Cg\) are continuous on both \((-\infty, 0) \times (0, 1)\) and \((0, \infty) \times (0, 1)\). Note that from the derivations in Sections \ref{appendix:overconfidence-k-non-hackable-confidence} and \ref{appendix:underconfidence-k-non-hackable-confidence} respectively, \((k+1)\ln(k+1)-k > 0\) and \(k-\ln(1+k) > 0\) for all \(k > 0\), therefore there is no division by zero.

Therefore, for all \(x\) in the extended real number line such that \(x \notin \{-\infty, 0, +\infty\}\), \(Cf\) and \(Cg\) are continuous in \(x\) when keeping \(c \in (0, 1)\) fixed.

Second, we aim to show that Underconfidence-\(k\) converges pointwise to Brier-log Hybrid as \(k \rightarrow \infty\). Note that Brier-log Hybrid reward scheme has \(Cf(+\infty, c) = c\) and \(Cg(+\infty, c) = c+\ln(1-c)\).

\(\lim_{x \rightarrow +\infty} Cf(x, c)\)

\(= \lim_{k \rightarrow \infty} \frac{kc+\ln(1-\frac{kc}{1+k})}{k-\ln(1+k)}\)

\(= \lim_{k \rightarrow \infty} \frac{c+\frac{\ln(1-\frac{kc}{1+k})}{k}}{1-\frac{\ln(1+k)}{k}}\)

\(= \lim_{k \rightarrow \infty} \frac{c+\frac{\ln(1-c)}{k}}{1-\frac{\ln(1+k)}{k}}\)

Since \(\ln(1+k) = o(k)\), \(\lim_{k \rightarrow \infty} \frac{\ln(1+k)}{k} = 0\), therefore, 
\(\lim_{x \rightarrow +\infty} Cf(x, c) = c = Cf(+\infty, c)\).

\(\lim_{x \rightarrow +\infty} Cg(x, c)\)

\(= \lim_{k \rightarrow \infty} \frac{kc+(k+1)\ln(1-\frac{kc}{1+k})}{k-\ln(1+k)} \)

\(= \lim_{k \rightarrow \infty} \frac{c+\frac{k+1}{k}\ln(1-\frac{kc}{1+k})}{1-\frac{\ln(1+k)}{k}} \)

\(= \lim_{k \rightarrow \infty} \frac{c+\ln(1-c)}{1-\frac{\ln(1+k)}{k}} \)

Since \(\ln(1+k) = o(k)\), \(\lim_{k \rightarrow \infty} \frac{\ln(1+k)}{k} = 0\), therefore, 
\(\lim_{x \rightarrow +\infty} Cg(x, c) = c+\ln(1-c) = Cg(+\infty, c)\). 

Therefore, as \(k \rightarrow +\infty\), Underconfidence-\(k\) converges pointwise to Brier-log Hybrid.

Third, we aim to show that Overconfidence-\(k\) converges pointwise to Correctness-only as \(k \rightarrow \infty\). Note that Correctness-only reward scheme has \(Cf(-\infty, c) = 1\) and \(Cg(-\infty, c) = 0\).

\(\lim_{x \rightarrow -\infty} Cf(x, c)\)

\(= \lim_{k \rightarrow \infty} \frac{(k+1)\ln(ck+1)-ck}{(k+1)\ln(k+1)-k}\)

\(= \lim_{k \rightarrow \infty} \frac{\frac{\ln(ck+1)}{\ln(k+1)} - \frac{ck}{(k+1)\ln(k+1)}}{1 - \frac{k}{(k+1)\ln(k+1)}}\)

\(= \lim_{k \rightarrow \infty} \frac{\frac{\ln(ck+1)}{\ln(k+1)} - \frac{c}{\ln(k+1)}}{1 - \frac{1}{\ln(k+1)}}\)

\(= \lim_{k \rightarrow \infty} \frac{\ln(ck+1)}{\ln(k+1)}\)

\(= \lim_{k \rightarrow \infty} \frac{\frac{c}{ck+1} }{\frac{1}{k+1}}\) (by L'H\^{o}pital's rule since both \(\ln(ck+1)\) and \(\ln(k+1)\) approach \(+\infty\) as \(k \rightarrow \infty\))

\(= \lim_{k \rightarrow \infty} \frac{c(k+1)}{ck+1}\)

\(= \lim_{k \rightarrow \infty} \frac{ck+c}{ck+1}\)

\(= \lim_{k \rightarrow \infty} \frac{c + \frac{c}{k}}{c + \frac{1}{k}}\)

\(= 1\)

\(= Cf(-\infty, c)\)

\(\lim_{x \rightarrow -\infty} Cg(x, c)\)

\(= \lim_{k \rightarrow \infty} \frac{\ln(ck+1)-ck}{(k+1)\ln(k+1)-k}\)

\(= \lim_{k \rightarrow \infty} \frac{\frac{\ln(ck+1)}{(k+1)\ln(k+1)}-\frac{ck}{(k+1)\ln(k+1)}}{1-\frac{k}{(k+1)\ln(k+1)}}\)

\(= \lim_{k \rightarrow \infty} \frac{\frac{\ln(ck+1)}{(k+1)\ln(k+1)}-\frac{c}{\ln(k+1)}}{1-\frac{1}{\ln(k+1)}}\)

\(= \lim_{k \rightarrow \infty} \frac{\ln(ck+1)}{(k+1)\ln(k+1)}\)

\(= (\lim_{k \rightarrow \infty} \frac{1}{k+1})\left(\lim_{k \rightarrow \infty} \frac{\ln(ck+1)}{\ln(k+1)}\right)\)

\(= 0\) (because \(\lim_{k \rightarrow \infty} \frac{\ln(ck+1)}{\ln(k+1)} = 1\) from earlier step)

\(= Cg(-\infty, c)\)

Therefore, as \(k \rightarrow +\infty\), Overconfidence-\(k\) converges pointwise to Correctness-only.

Fourth, we aim to show that Underconfidence-\(k\) converges pointwise to Brier-1 as \(k \rightarrow 0^+\). Note that Brier-1 reward scheme has \(Cf(0, c) = 2c - c^2 = 1 - 1 + 2c - c^2 = 1 - (1-c)^2\) and \(Cg(0, c) = - c^2\).

\(\lim_{x \rightarrow 0^+} Cf(x, c)\)

\(= \lim_{k \rightarrow 0^+} \frac{kc+\ln(1-\frac{kc}{1+k})}{k-\ln(1+k)}\)

\(= \lim_{k \rightarrow 0^+} \frac{c+\frac{-c\frac{d}{dk}(\frac{k}{k+1})}{1-\frac{kc}{1+k}}}{1-\frac{1}{1+k}}\)

(by L'H\^opital's rule since \(\lim_{k \rightarrow 0^+} kc+\ln(1-\frac{kc}{1+k}) = 0\) and \(\lim_{k \rightarrow 0^+} k-\ln(1+k) = 0\))

\(= \lim_{k \rightarrow 0^+} \frac{c+\frac{-c\frac{d}{dk}(1-\frac{1}{k+1})}{1-\frac{kc}{1+k}}}{1-\frac{1}{1+k}}\)

\(= \lim_{k \rightarrow 0^+} \frac{c+\frac{\frac{-c}{(k+1)^2}}{1-\frac{kc}{1+k}}}{1-\frac{1}{1+k}}\)

\(= \lim_{k \rightarrow 0^+} \frac{c-\frac{\frac{c}{(k+1)^2}}{\frac{1+k-kc}{1+k}}}{\frac{k}{1+k}}\)

\(= \lim_{k \rightarrow 0^+} \frac{c-\frac{\frac{c}{(k+1)}}{1+k-kc}}{\frac{k}{1+k}}\)

\(= \lim_{k \rightarrow 0^+} \frac{c(k+1)-\frac{c}{1+k-kc}}{k}\)

\(= \lim_{k \rightarrow 0^+} \frac{c(k+1)(1+k-kc)-c}{k(1+k-kc)}\)

\(= \lim_{k \rightarrow 0^+} \frac{(ck+c)(1+k-kc)-c}{k+k^2-k^2c}\)

\(= \lim_{k \rightarrow 0^+} \frac{ck + c + ck^2 + ck - k^2c^2 - kc^2 - c}{k+k^2-k^2c}\)

\(= \lim_{k \rightarrow 0^+} \frac{2c + ck - kc^2 - c^2}{1 + k - kc}\)

\(= \frac{2c-c^2}{1}\)

\(= 2c - c^2\)

\(= Cf(0, c)\)

\(\lim_{x \rightarrow 0^+} Cg(x, c)\)

\(= \lim_{k \rightarrow 0^+} \frac{kc+(k+1)\ln(1-\frac{kc}{1+k})}{k-\ln(1+k)}\)

\(= \lim_{k \rightarrow 0^+} \frac{kc+\ln(1-\frac{kc}{1+k})}{k-\ln(1+k)} + \lim_{k \rightarrow 0^+} \frac{k\ln(1-\frac{kc}{1+k})}{k-\ln(1+k)}\)

\(= (2c-c^2) + \lim_{k \rightarrow 0^+} \frac{k\ln(1-\frac{kc}{1+k})}{k-\ln(1+k)}\) (from derivation of \(\lim_{x \rightarrow 0^+} Cf(x, c))\)

\(= (2c-c^2) + \lim_{k \rightarrow 0^+} \frac{\ln(1-\frac{kc}{1+k})}{1-\frac{\ln(1+k)}{k}}\)

Note that \(\lim_{k \rightarrow 0^+} \frac{\ln(1+k)}{k} = \lim_{k \rightarrow 0^+} \frac{1}{1+k} = 1\) by L'H\^opital's rule since \(\lim_{k \rightarrow 0^+} \ln(1+k) = 0\) and \(\lim_{k \rightarrow 0^+} k = 0\).

Since \(\lim_{k \rightarrow 0^+} \ln(1-\frac{kc}{1+k}) = \ln(1) = 0\) and \(\lim_{k \rightarrow 0^+} 1-\frac{\ln(1+k)}{k} = 0\), we can apply L'H\^opital's rule. Therefore, 

\(\lim_{x \rightarrow 0^+} Cg(x, c)\)

\(= (2c-c^2) + \lim_{k \rightarrow 0^+} \frac{\frac{-c\frac{d}{dk}\left(\frac{k}{1+k}\right)}{1-\frac{kc}{1+k}}}{- \left(\frac{-\ln(1+k)}{k^2} + \frac{\frac{1}{1+k}}{k} \right)}\)

\(= (2c-c^2) + \lim_{k \rightarrow 0^+} \frac{\frac{-c\frac{d}{dk}(\frac{k}{1+k})}{1-\frac{kc}{1+k}}}{-(\frac{-\ln(1+k)}{k^2} + \frac{\frac{1}{1+k}}{k})}\)

\(= (2c-c^2) + \lim_{k \rightarrow 0^+} \frac{-c\frac{d}{dk}(1-\frac{1}{1+k})}{-(\frac{-\ln(1+k)}{k^2} + \frac{\frac{1}{1+k}}{k})(1-\frac{kc}{1+k})}\)

\(= (2c-c^2) + \lim_{k \rightarrow 0^+} \frac{-c(\frac{1}{(1+k)^2})}{-(\frac{-\ln(1+k)}{k^2} + \frac{1}{k(1+k)})(\frac{1+k-kc}{1+k})}\)

\(= (2c-c^2) + \lim_{k \rightarrow 0^+} \frac{c}{(\frac{-\ln(1+k)}{k^2} + \frac{1}{k(1+k)})(\frac{1+k-kc}{1+k})(1+k)^2}\)

\(= (2c-c^2) + \lim_{k \rightarrow 0^+} \frac{c}{(\frac{-\ln(1+k)}{k^2} + \frac{1}{k(1+k)})(\frac{1}{1})(1)}\)

\(= (2c-c^2) + \lim_{k \rightarrow 0^+} \frac{c}{\frac{-\ln(1+k)}{k^2} + \frac{1}{k(1+k)}}\)

\(= (2c-c^2) + \lim_{k \rightarrow 0^+} \frac{ck}{\frac{-\ln(1+k)}{k} + \frac{1}{1+k}}\)

\(= (2c-c^2) + \lim_{k \rightarrow 0^+} \frac{ck}{\frac{-\ln(1+k)}{k} + 1}\)

\(= (2c-c^2) + \lim_{k \rightarrow 0^+} \frac{ck^2}{k - \ln(1+k)}\)

\(= (2c-c^2) + \lim_{k \rightarrow 0^+} \frac{2ck}{1 - \frac{1}{1+k}}\) 

(by L'H\^opital's rule since \(\lim_{k \rightarrow 0^+} ck^2 = 0\) and \(\lim_{k \rightarrow 0^+} k-\ln(1+k) = 0\))

\(= (2c-c^2) + \lim_{k \rightarrow 0^+} \frac{2ck}{\frac{k}{1+k}}\)

\(= (2c-c^2) + \lim_{k \rightarrow 0^+} \frac{2ck(1+k)}{k}\)

\(= (2c-c^2) + \lim_{k \rightarrow 0^+} 2c(1+k)\)

\(= (2c-c^2) + 2c\)

\(= - c^2\)

\(= Cg(0, c)\)

Therefore, as \(k \rightarrow 0^+\), Underconfidence-\(k\) converges pointwise to Brier-1.

Lastly, we aim to show that Overconfidence-\(k\) converges pointwise to Brier-1 as \(k \rightarrow 0^+\).

\(\lim_{x \rightarrow 0^-} Cf(x, c)\)

\(= \lim_{k \rightarrow 0^+} \frac{(k+1)\ln(ck+1)-ck}{(k+1)\ln(k+1)-k}\)

\(= \lim_{k \rightarrow 0^+} \frac{\ln(ck+1) + \frac{c(k+1)}{ck+1} - c}{\ln(k+1) + \frac{k+1}{k+1} - 1}\)

(by L'H\^opital's rule since \(\lim_{k \rightarrow 0^+} (k+1)\ln(ck+1)-ck = 0\) and \(\lim_{k \rightarrow 0^+} (k+1)\ln(k+1)-k = 0\))

\(= \lim_{k \rightarrow 0^+} \frac{\ln(ck+1) + \frac{ck+c}{ck+1} - c}{\ln(k+1)}\)

\(= \lim_{k \rightarrow 0^+} \frac{\ln(ck+1) + \frac{ck+1 - (1-c)}{ck+1} - c}{\ln(k+1)}\)

\(= \lim_{k \rightarrow 0^+} \frac{\ln(ck+1) + 1 - \frac{1-c}{ck+1} - c}{\ln(k+1)}\)

\(= \lim_{k \rightarrow 0^+} \frac{\frac{c}{ck+1} - \frac{(1-c)(-1)(c)}{(ck+1)^2}}{\frac{1}{k+1}}\)

(by L'H\^opital's rule since \(\lim_{k \rightarrow 0^+} \ln(ck+1) + 1 - \frac{1-c}{ck+1} - c = 0+1-(1-c)-c = 0\) and \(\lim_{k \rightarrow 0^+} \ln(k+1) = 0\))

\(= \lim_{k \rightarrow 0^+} (k+1)(\frac{c}{ck+1} - \frac{c(c-1)}{(ck+1)^2})\)

\(= (1)(\frac{c}{1} - \frac{c(1-c)}{(1)^2})\)

\(= c - c(c - 1)\)

\(= c - c^2 + c\)

\(= 2c - c^2\)

\(= Cf(0, c)\)

\(\lim_{x \rightarrow 0^-} Cg(x, c)\)

\(= \lim_{k \rightarrow 0^+} \frac{\ln(ck+1)-ck}{(k+1)\ln(k+1)-k}\)

\(= \lim_{k \rightarrow 0^+} \frac{\frac{c}{ck+1}-c}{\ln(k+1) + \frac{k+1}{k+1} - 1}\)

(by L'H\^opital's rule since \(\lim_{k \rightarrow 0^+} \ln(ck+1)-ck = 0\) and \(\lim_{k \rightarrow 0^+} (k+1)\ln(k+1)-k = 0\))

\(= \lim_{k \rightarrow 0^+} \frac{\frac{c}{ck+1}-c}{\ln(k+1)}\)

\(= \lim_{k \rightarrow 0^+} \frac{\frac{-c^2}{(ck+1)^2}}{\frac{1}{k+1}}\)

(by L'H\^opital's rule since \(\lim_{k \rightarrow 0^+} \frac{c}{ck+1}-c = 0\) and \(\lim_{k \rightarrow 0^+} \ln(k+1) = 0\))

\(= \lim_{k \rightarrow 0^+} \frac{-c^2(k+1)}{(ck+1)^2}\)

\(= \frac{-c^2(1)}{(1)^2}\)

\(= -c^2\)

\(= Cg(0, c)\)

Therefore, as \(k \rightarrow 0^+\), Overconfidence-\(k\) converges pointwise to Brier-1.

Since \(\lim_{x \rightarrow 0^-} Cf(x, c) = \lim_{x \rightarrow 0^+} Cf(x, c) = Cf(0, c)\) and likewise for \(Cg\), both \(Cf\) and \(Cg\) are continuous at \(x=0\) when keeping \(c\) fixed. 

The theorem follows by combining the earlier results about continuity at \(x=+\infty\) and \(x=-\infty\) when keeping \(c\) fixed. 

\end{proof}

\subsection{Degree of Overconfidence and Underconfidence Bias}

In this subsection, we heuristically justify the ordering of the Overconfidence-\(k\) reward schemes and the Underconfidence-\(k\) reward schemes.

In experiments, we have normalized the Overconfidence-\(k\) reward schemes by dividing by \((k+1)\ln(k+1)-k\), and normalized the Underconfidence-\(k\) reward schemes by dividing by \(k-\ln(1+k)\) to ensure that \(f(1^-) = 1\) and \(g(0^+) = 0\), allowing for a fair comparison of the scale of the reward schemes by ensuring a consistent scale. 

In this section, we remove the normalizations of the Overconfidence-\(k\) and Underconfidence-\(k\) reward schemes, thereby calling them Unnormalized Overconfidence-\(k\) and Unnormalized Underconfidence-\(k\) respectively. Note that by Theorem \ref{theorem:scaling-strict-non-hackable}, Unnormalized Overconfidence-\(k\) and Unnormalized Underconfidence-\(k\) are both strictly non-hackable confidence since the respective normalization factors of \((k+1)\ln(k+1)-k\) and \(k-\ln(1+k)\) are positive for all positive \(k\). Moreover, Unnormalized Overconfidence-\(k\) has overconfidence bias and Unnormalized Underconfidence-\(k\) has underconfidence bias because \(\frac{\partial^2 E}{\partial c^2}\) differ by a positive multiplicative factor of the normalization factor.

We demonstrate through calculations in the following subsubsections that as \(k\) increases, the miscalibration penalties for both Unnormalized Overconfidence-\(k\) and Unnormalized Underconfidence-\(k\) increases, keeping both the confidence \(c\) and the probability of answering correctly \(p\) fixed.

Therefore, since both Overconfidence-\(k\) and Underconfidence-\(k\), and their unnormalized versions, have overconfidence and underconfidence biases as their name suggests (refer to Appendices \ref{appendix:overconfidence-k-confidence-bias-determination} and \ref{appendix:underconfidence-k-confidence-bias-determination} for more details), the respective degrees of overconfidence bias and underconfidence bias can be viewed as increasing as \(k\) increases.

\subsubsection{Unnormalized Overconfidence-\(k\)}

Unnormalized Overconfidence-\(k\) has \(f(c) = (k+1)\ln(ck+1)-ck\) and \(g(c) = \ln(ck+1)-ck\).

\(E_{pen}(c, p)\)

\(= E(p, p) - E(c, p)\)

\(= p f(p) + (1-p) g(p) - (p f(c) + (1-p) g(c))\)

\(= p(f(p) - f(c)) + (1-p)(g(p) - g(c))\)

\(= p( ((k+1)\ln(pk+1)-pk) - ((k+1)\ln(ck+1)-ck) ) + (1-p)((\ln(pk+1)-pk) - (\ln(ck+1)-ck))\)

\(= p( ((k+1)\ln(pk+1) - pk - (k+1)\ln(ck+1) + ck ) + (1-p)(\ln(pk+1) - pk - \ln(ck+1) + ck)\)

\(= p( (k+1)(\ln(pk+1) - \ln(ck+1) + (c-p)k ) + (1-p)(\ln(pk+1) - \ln(ck+1) + (c-p)k)\)

\(= (p(k+1) + (1-p))(\ln(pk+1) - \ln(ck+1)) + (c-p)k(p + (1-p))\)

\(= (pk+p + (1-p))(\ln(pk+1) - \ln(ck+1)) + (c-p)k\)

\(= (pk + 1)(\ln(pk+1) - \ln(ck+1)) + (c-p)k\)

Next, for all \(c, p \in (0, 1)\), we aim to show that the miscalibration penalty increases as \(k\) increases provided \(c \neq p\) by computing the first partial derivative of the miscalibration penalty with respect to \(k\) and showing that it is always positive provided \(c \neq p\).

\(\frac{\partial}{\partial k}(E_{pen}(c, p))\)

\(= p(\ln(pk+1) - \ln(ck+1)) + (pk+1)(\frac{p}{pk+1} - \frac{c}{ck+1}) + (c-p)\)

\(= p(\ln(pk+1) - \ln(ck+1)) + p - \frac{c(pk+1)}{ck+1} + (c-p)\)

\(= p(\ln(pk+1) - \ln(ck+1)) - \frac{c(pk+1)}{ck+1} + c\)

Note that \(\frac{\partial}{\partial k}(E_{pen}(p, p)) = p(\ln(pk+1)-\ln(pk+1)) - \frac{p(pk+1)}{pk+1} + p = 0 - p + p = 0\) for all \(p \in (0, 1)\).

Moreover, 

\(\frac{\partial^2}{\partial c \partial k}(E_{pen}(c, p))\)

\(= -\frac{pk}{ck+1} - \frac{pk+1}{ck+1} + \frac{ck(pk+1)}{(ck+1)^2} + 1 \)

\(= 1 - \frac{pk}{ck+1} - \frac{pk+1}{ck+1} + \frac{ck(pk+1)}{(ck+1)^2}\)

\(= 1 - \frac{pk}{ck+1} - \frac{pk+1}{ck+1}(1 - \frac{ck}{ck+1})\)

\(= 1 - \frac{pk}{ck+1} - \frac{pk+1}{ck+1}(\frac{1}{ck+1})\)

\(= 1 - \frac{pk}{ck+1} - \frac{pk+1}{(ck+1)^2}\)

\(= \frac{(ck+1)^2 - pk(ck+1) - pk-1}{(ck+1)^2}\)

\(= \frac{c^2k^2 + 2ck + 1 - pck^2 - pk - pk - 1}{(ck+1)^2}\)

\(= \frac{c^2k^2 - pck^2 + 2ck - 2pk}{(ck+1)^2}\)

\(= \frac{ck^2(c-p) + 2k(c-p)}{(ck+1)^2}\)

\(= \frac{k(ck + 2)(c-p)}{(ck+1)^2}\)

For all \(c, p \in (0, 1)\) and \(k \in (0, \infty)\), \(k, ck+2, (ck+1)^2 > 0\), therefore, \(\frac{\partial^2}{\partial c \partial k}(E_{pen}(c, p)) < 0\) if \(c < p\), \(\frac{\partial^2}{\partial c \partial k}(E_{pen}(c, p)) > 0\) if \(c > p\) and \(\frac{\partial^2}{\partial c \partial k}(E_{pen}(c, p)) = 0\) if \(c = p\). Since \(\frac{\partial}{\partial k}(E_{pen}(p, p)) = 0\), as \(c\) increases from 0 to \(p\) keeping \(p\) fixed, \(\frac{\partial}{\partial k}(E_{pen}(c, p))\) decreases towards 0 and as \(c\) increases from \(p\) to 1 keeping \(p\) fixed, \(\frac{\partial}{\partial k}(E_{pen}(c, p))\) increases from 0.

Since \(\frac{\partial}{\partial k}(E_{pen}(c, p)) > 0\) for all \(c, p \in (0, 1)\) and \(k \in (0, \infty)\) such that \(c \neq p\), \(E_{pen}(c, p)\) strictly increases in \(k\) keeping \(c\) and \(p\) fixed, i.e. miscalibration penalty increases as \(k\) increases, keeping \(c\) and \(p\) fixed. Therefore, Overconfidence-\(k\) has heuristically more overconfidence bias as \(k\) increases.

Note that in the edge case where \(c = p \in (0, 1)\), \(E_{pen}(p, p) = 0\) due to Unnormalized Overconfidence-\(k\) satisfying weak Proper Scoring.

\subsubsection{Unnormalized Underconfidence-\(k\)}

Unnormalized Underconfidence-\(k\) has \(f(c) = kc+\ln(1-\frac{kc}{1+k})\) and \(g(c) = kc+(k+1)\ln(1-\frac{kc}{1+k})\).

\(E(c, p)\)

\(= E(p, p) - E(c, p)\)

\(= p f(p) + (1-p) g(p) - (p f(c) + (1-p) g(c))\)

\(= p(f(p) - f(c)) + (1-p)(g(p) - g(c))\)

\(= p(kp+\ln(1-\frac{kp}{1+k}) - (kc+\ln(1-\frac{kc}{1+k}))) + (1-p)((kp+(k+1)\ln(1-\frac{kp}{1+k})) - (kc+(k+1)\ln(1-\frac{kc}{1+k})))\)

\(= p(kp - kc + \ln(1-\frac{kp}{1+k}) - \ln(1-\frac{kc}{1+k})) + (1-p)(kp - kc + (k+1)\ln(1-\frac{kp}{1+k})-(k+1)\ln(1-\frac{kc}{1+k}))\)

\(= (p + (1-p))(kp - kc) + (p + (k+1)(1-p))(\ln(1-\frac{kp}{1+k}) - \ln(1-\frac{kc}{1+k}))\)

\(= (kp - kc) + (p + (k+1) - (kp+p))(\ln(1-\frac{kp}{1+k}) - \ln(1-\frac{kc}{1+k}))\)

\(= (kp - kc) + (k + 1 - kp)(\ln(1-\frac{kp}{1+k}) - \ln(1-\frac{kc}{1+k}))\)

\(= (p - c)k + (1 + (1-p)k)(\ln(1-p(1-\frac{1}{1+k})) - \ln(1-c(1-\frac{1}{1+k})))\)

Next, for all \(c, p \in (0, 1)\), we aim to show that the miscalibration penalty increases as \(k\) increases provided \(c \neq p\) by computing the first partial derivative of the miscalibration penalty with respect to \(k\) and showing that it is always positive provided \(c \neq p\).

\(\frac{\partial}{\partial k}(E_{pen}(c, p))\)

\(= (p-c) + (1-p)(\ln(1-p(1-\frac{1}{1+k})) - \ln(1-c(1-\frac{1}{1+k}))) + (1 + (1-p)k)\left(\frac{\frac{-p}{(1+k)^2}}{1-p(1-\frac{1}{1+k})} - \frac{\frac{-c}{(1+k)^2}}{1-c(1-\frac{1}{1+k})}\right) \)

\(= (p-c) + (1-p)(\ln(1-p(1-\frac{1}{1+k})) - \ln(1-c(1-\frac{1}{1+k}))) + (1 + (1-p)k)\left(\frac{\frac{-p}{(1+k)^2}}{1-\frac{kp}{1+k}} - \frac{\frac{-c}{(1+k)^2}}{1-\frac{kc}{1+k}}\right) \)

\(= (p-c) + (1-p)(\ln(1-p(1-\frac{1}{1+k})) - \ln(1-c(1-\frac{1}{1+k}))) + (1 + (1-p)k)\left(\frac{\frac{-p}{(1+k)^2}}{\frac{1+k-kp}{1+k}} - \frac{\frac{-c}{(1+k)^2}}{\frac{1+k-kc}{1+k}}\right) \)

\(= (p-c) + (1-p)(\ln(1-p(1-\frac{1}{1+k})) - \ln(1-c(1-\frac{1}{1+k}))) + (1 + (1-p)k)\left(\frac{\frac{-p}{1+k}}{1+k-kp} - \frac{\frac{-c}{1+k}}{1+k-kc}\right) \)

\(= (p-c) + (1-p)(\ln(1-p(1-\frac{1}{1+k})) - \ln(1-c(1-\frac{1}{1+k}))) + (1 + (1-p)k)\left(\frac{c}{(1+k)(1+k-kc)} - \frac{p}{(1+k)(1+k-kp)}\right) \)

Note that whenever \(c, p \in (0,1)\) such that \(c = p\), \(\frac{\partial}{\partial k}(E_{pen}(c, p)) = 0\).

\(\frac{\partial}{\partial c\partial k}(E_{pen}(c, p))\)

\(= -1 + \frac{\partial}{\partial c}((1-p)(- \ln(1-c(1-\frac{1}{1+k}))) + (1 + (1-p)k)(\frac{c}{(1+k)(1+k-kc)})) \)

\(= -1 + (1-p)\frac{\partial}{\partial c}(- \ln(1-c(1-\frac{1}{1+k})) + \frac{1 + (1-p)k}{1+k} \frac{\partial}{\partial c}\left(\frac{c}{1+k-kc}\right) \)

\(= -1 + (1-p)(-1)\left(\frac{-(1-\frac{1}{1+k})}{1-c(1-\frac{1}{1+k})}\right) + \frac{1 + (1-p)k}{1+k} \left(\frac{(1+k-kc) - c(-k)}{(1+k-kc)^2}\right) \)

\(= -1 + (1-p)\left(\frac{\frac{k}{1+k}}{1-\frac{ck}{1+k}}\right) + \frac{1 + (1-p)k}{1+k} \left(\frac{1+k}{(1+k-kc)^2}\right) \)

\(= -1 + (1-p)\left(\frac{k}{1+k-ck}\right) + \frac{1 + (1-p)k}{(1+k-kc)^2} \)

\(= -1 + \frac{k(1-p)}{1+k-ck} + \frac{1 + (1-p)k}{(1+k-kc)^2} \)

\(= \frac{- (1+k-kc)^2 + k(1-p)(1+k-kc) + 1 + (1-p)k}{(1+k-kc)^2} \)

\(= \frac{- (1+k-kc)^2 + k(1-p)(2+k-kc) + 1}{(1+k-kc)^2} \)

\(= \frac{- (1+k)^2 + 2(1+k)(kc) - k^2c^2 + (k-kp)(2+k-kc) + 1}{(1+k-kc)^2} \)

\(= \frac{- 1 - 2k - k^2 + 2kc + 2k^2c - k^2c^2 + 2k - 2kp + k^2 - k^2p - k^2c + k^2cp + 1}{(1+k-kc)^2} \)

\(= \frac{2kc + 2k^2c - k^2c^2 - 2kp - k^2p - k^2c + k^2cp}{(1+k-kc)^2} \)

\(= \frac{2kc + k^2c - k^2c^2 - 2kp - k^2p + k^2cp}{(1+k-kc)^2} \)

\(= \frac{c(2k + k^2 - k^2c) - p(2k + k^2 + k^2c)}{(1+k-kc)^2} \)

\(= \frac{(c - p)(2k + k^2 - k^2c)}{(1+k-kc)^2} \)

\(= \frac{(c - p)(2k + k^2(1-c))}{(1+k(1-c))^2} \)

For all \(c, p \in (0, 1)\) and \(k \in (0, \infty)\), \(2k + k^2(1-c) > 0\) since \(1-c \in (0, 1)\) and \((1+k(1-c))^2 > 0\). 

Therefore, \(\frac{\partial^2}{\partial c \partial k}(E_{pen}(c, p)) < 0\) if \(c < p\), \(\frac{\partial^2}{\partial c \partial k}(E_{pen}(c, p)) > 0\) if \(c > p\) and \(\frac{\partial^2}{\partial c \partial k}(E_{pen}(c, p)) = 0\) if \(c = p\). Since \(\frac{\partial}{\partial k}(E_{pen}(p, p)) = 0\), as \(c\) increases from 0 to \(p\) keeping \(p\) fixed, \(\frac{\partial}{\partial k}(E_{pen}(c, p))\) decreases towards 0 and as \(c\) increases from \(p\) to 1 keeping \(p\) fixed, \(\frac{\partial}{\partial k}(E_{pen}(c, p))\) increases from 0.

Since \(\frac{\partial}{\partial k}(E_{pen}(c, p)) > 0\) for all \(c, p \in (0, 1)\) and \(k \in (0, \infty)\) such that \(c \neq p\), \(E_{pen}(c, p)\) strictly increases in \(k\) keeping \(c\) and \(p\) fixed, i.e. miscalibration penalty increases as \(k\) increases, keeping \(c\) and \(p\) fixed. Therefore, Underconfidence-\(k\) has heuristically more underconfidence bias as \(k\) increases.

Note that in the edge case where \(c = p \in (0, 1)\), \(E_{pen}(p, p) = 0\) due to Unnormalized Underconfidence-\(k\) satisfying weak Proper Scoring.

\section{Experimental Setting Details} \label{appendix:experimental-setting-details}

\subsection{Experimental Procedure} \label{appendix:experimental-procedure}

We used Qwen 2.5 (3B) Instruct \citep{qwen2, qwen2.5}, Qwen 3 (4B) \citep{qwen3technicalreport}. In addition to the LLMs mentioned in the main paper, we also tested our code on Gemma 4 (E2B) Instruct \citep{gemmateam2026gemma4} in our experiments. Reasoning tokens, if supported by the LLM, are disabled in our experiments.

For the datasets, we used BigMath \citep{albalak2025bigmathlargescalehighqualitymath}, DeepMath-103K \citep{deepmath}, HotpotQA \citep{yang2018hotpotqa} and HotpotQA-Modified, a custom modified version of HotpotQA based on the preprocessing methodology of \cite{damani2026beyond}, 

We ran experiments of Qwen 2.5 (3B) Instruct for all four datasets. Due to limitations of compute resources, we only run experiments of Qwen 3 (4B) Instruct and Gemma 4 (E2B) Instruct on selected datasets. For Qwen 3 (4B), we ran experiments only on HotpotQA-Modified and DeepMath-103K. For Gemma 4 (E2B) Instruct, we ran experiments only on HotpotQA-Modified. 

Our experiment methodology can be described as follows:

\textbf{Stage 1 - Dataset preparation}: For each dataset obtained, we split the data into training and testing sets, and classified the questions according to their difficulty in easy, medium and hard, as elaborated in Appendix \ref{appendix:dataset-details-and-preprocessing}. For Qwen 3 (4B) and Gemma 4 (E2B) Instruct experiments, steps 2 and 3 are skipped since we found it easier for them to obey the JSON format than XML. 

\textbf{Stage 2 - Fine-tuning preparation}: For each dataset, the LLM is tasked to answer 1024 questions in the training set and to provide its confidence score in JSON format. Its responses are collected for the next stage. More details on how this is done can be found on Appendix \ref{appendix:fine-tuning-preparation}.

\textbf{Stage 3 - Supervised Fine-tuning (SFT)}: The responses from the previous step are reformatted to an XML-like format to help the LLM follow instructions in the next step. The LLM is fine-tuned on the reformatted responses using LoRA \citep{hu2021loralowrankadaptationlarge} for 1 epoch. More details can be found on Appendix \ref{appendix:sft-training-details}.

\textbf{Stage 4 - RL}: We fine-tuned the LLM to perform confidence calibration using RL. The LLM is tasked to answer the question and to provide a confidence value \(\tilde{c}\), an integer in [0, 100], to its answer, along with some reasoning on how it arrived at the answer and \(\tilde{c}\). To determine the reward, the answer is graded for correctness and  \(\tilde{c}\) is normalized into the implied confidence \(c\) by setting \(c  = \frac{\tilde{c}+0.5}{101}\).

For most LLMs, instead of Group Relative Policy Optimization (GRPO) \citep{shao2024deepseekmathpushinglimitsmathematical}, we used GRPO Done Right (Dr GRPO) \citep{liu2025understandingr1zeroliketrainingcritical} loss, which gives each token equal weightage. For Gemma 4 (E2B) Instruct, due to training instability issues when using Dr GRPO, we use DAPO \citep{yu2025dapoopensourcellmreinforcement}. In both cases, we remove the standard deviation normalization term to eliminate the confidence calibration bias found in GRPO \citep{bereket2025uncalibratedreasoninggrpoinduces}. 

For Qwen 2.5 (3B) Instruct experiments, the LLM is tasked to output in XML format. For Qwen 3 (4B) and Gemma 4 (E2B) Instruct experiments, the LLM is tasked to output in JSON format. 

We ran our experiments with 8 generations per question and 64 questions per batch. Since we aim to train the LLM in both confidence calibration and reasoning improvements, the loss is taken over the entire output. More details can be found in Appendix \ref{appendix:rl-training-details}.

\textbf{Step 5 - Evaluation}: The LLM is then evaluated using the test set of the dataset, with the answers and verbalized confidences recorded. The LLM samples 16 responses for each question, except for Qwen 3 (4B) when evaluated on DeepMath-103K, where the LLM samples 2 responses for each question. The base models and, if applicable, the model after Stage 3 (i.e., the SFT model) are evaluated as baselines. We evaluated the model using accuracy, Expected Calibration Error (10 bins) \citep{guo2017calibrationmodernneuralnetworks}, Area Under Receiver Operating Characteristic Curve \citep{marcum1960auroc, peterson1954auroc}, Brier score \citep{brier1950verification}, the average Brier-1 reward \citep{damani2026beyond} and calibration bias. For Qwen 3 (4B) and Gemma 4 (E2B) Instruct experiments, we additionally evaluate whether the LLM has abstained from answering the question or admitted unanswerability, with GPT-OSS-20B \citep{openai2025gptoss120bgptoss20bmodel} and Qwen 3 (30B-A3B) FP8 \citep{qwen3technicalreport} as LLM judges. Details regarding the evaluation protocol can be found in Appendix \ref{appendix:evaluation-protocol-details}. Details regarding the evaluation metrics can be found in Appendix \ref{appendix:evaluation-metrics-used}. 

Details regarding the exact prompt formats can be found in Appendix \ref{appendix:prompt-formats}.

\subsection{Dataset details and preprocessing} \label{appendix:dataset-details-and-preprocessing}

% Had to rely on Google Search Gemini AI to find the relevant commands for bulleted list.
\begin{itemize}
\item \textbf{BigMath \citep{albalak2025bigmathlargescalehighqualitymath}}:
We took the filtered dataset from Hugging Face found at open-r1/Big-Math-RL-Verified-Processed\footnote{https://huggingface.co/datasets/open-r1/Big-Math-RL-Verified-Processed}. The filtered dataset contains only questions where the answers are parseable by math-verify\footnote{https://github.com/huggingface/Math-Verify} and the Llama 3.1 (8B) solve rate is provided by \cite{albalak2025bigmathlargescalehighqualitymath}.

We labeled the questions with the highest tercile of Llama 3.1 (8B) solve rate as easy, middle tercile of Llama 3.1 (8B) solve rate as medium and lower tercile of Llama 3.1 (8B) solve rate as difficult. The easy questions have Llama 3.1 (8B) solve rates greater than 0.765625, the hard questions have corresponding solve rates below 0.1875, while the medium questions have solve rates in between 0.1875 and 0.765625 inclusive.

We randomly split the dataset into train and test splits in an approximately 97:3 ratio. 

To verify LLM answers, we used math-verify, with an accuracy tolerance of two decimal places for non-exact answers represented in decimal.

\item \textbf{DeepMath-103K \citep{deepmath}}: This dataset consists of difficult math questions and has been used to finetune LLM via RL to improve reasoning performance \citep{deepmath}. The difficulty rating is based on estimates from GPT-4o after being prompted with the difficulty level guidance in the Art of Problem Solving\footnote{https://artofproblemsolving.com/wiki/index.php/AoPS\_Wiki:Competition\_ratings}\citep{deepmath}.

We labeled the questions with the lowest tercile of difficulty rating as easy, middle tercile of difficulty rating as medium and highest tercile of difficulty rating as difficult. The easy questions have difficulty under 5, the hard questions have difficulty above 6 while the medium questions have difficulty in between 5 and 6 inclusive. 

The train and test dataset splits were taken from trl-lib/DeepMath-103K at Hugging Face\footnote{https://huggingface.co/datasets/trl-lib/DeepMath-103K}. The questions were matched with the original dataset from zwhe99/DeepMath-103K\footnote{https://huggingface.co/datasets/zwhe99/DeepMath-103K} to obtain the difficulty values.

We attempted to filter off multiple choice questions and yes/no questions by filtering off the questions whose case-insensitive answer after removing the dollar sign in LaTeX or the full stop at the back is in the set \{yes, no, true, false, a, b, c, d\}. In such questions, the question format and the ground truth format are often inconsistent, complicating reliable grading of the answers. 

In particular, the yes/no questions often require proof to ensure that true understanding of mathematical concepts is rewarded. Automating the grading of the reasoning accuracy in proofs is difficult and can be unreliable.

The verification method is the same as in BigMath.

\item \textbf{HotpotQA \citep{yang2018hotpotqa}}: This is a textual reasoning dataset where LLMs are given multiple sources and are asked reading comprehension questions about the sources. The entries in the dataset have 10 sources, where 2 sources are relevant and 8 are irrelevant. The difficulties of easy, medium and hard are already provided in the dataset. 

We used the distractor subset.

For training, we used the provided training set, while for validation, we used the provided validation set. The training set contains easy, medium and hard questions while the validation set only contains hard questions.

An answer is marked as correct if the ROUGE-1 \citep{lin2004rouge} score exceeds 0.7. 

\item \textbf{HotpotQA-Modified}: Modified from the procedure in \cite{damani2026beyond} with the most important difference being the inclusion of dataset difficulty, for each dataset entry in the HotpotQA dataset, we perform the following at random:
\begin{itemize}
\item With probability \(\frac{1}{3}\), 2 irrelevant sources are removed. These are relabeled as easy questions.
\item With probability \(\frac{1}{3}\), 1 irrelevant source and 1 relevant source are removed. These are relabeled as medium questions.
\item With probability \(\frac{1}{3}\), both relevant sources are removed, leaving the LLM with no choice but to deduce from prior knowledge. These are relabeled as hard questions.
\end{itemize}
The evaluation method is the same as in HotpotQA.

\end{itemize}

\subsection{Procedure details for fine-tuning preparation}\label{appendix:fine-tuning-preparation}

This step is applicable to Qwen 2.5 (3B) Instruct experiments and not applicable to Gemma 4 (E2B) Instruct and Qwen 3 (4B) experiments.

The LLM was instructed to answer a 1024-question subset of questions in the training set in JavaScript Object Notation (JSON) format. A format enforcer in vLLM \citep{kwon2023efficient} was used to guide the LLM to produce valid a JSON response. The JSON contains the reasoning, answer, confidence analysis and confidence attributes. The LLM is instructed to place its reasoning on the reasoning attribute, its final answer on the answer attribute, its justification of its confidence value on the confidence analysis attribute and the confidence value on the confidence attribute. The LLM is instructed to output an integer value between 0 and 100 inclusive in the confidence attribute. For BigMath and DeepMath-103K, the output token limit was 2048. For HotpotQA and HotpotQA-Modified, the output token limit was 1024. The LLM temperature used is 1.

If a valid JSON response was produced, each of the attributes were extracted and rearranged into an Extensible Markup Language (XML) format for the SFT stage. As the LLM may not realize that its LaTeX expressions involving backslashes were being escaped, we attempt to undo the escapes to recover the original string. 

If the JSON response was invalid, the question will be discarded and its response will not be included in the SFT dataset.

As shown in Table \ref{table:valid-json-responses-sanity-check}, our experiments in Qwen 2.5 (3B) Instruct had the LLM output valid JSON responses for over 95\% of the 1024-question subset for each of the datasets tested, ensuring a sufficient dataset size for the SFT stage. 

\begin{table}
    \centering
    \caption{Number of questions with valid JSON responses in Qwen 2.5 (3B) Instruct out of the 1024 questions in the fine-tuning preparation stage.}
    \begin{tabular}{cc}
        \toprule
         Dataset& Questions with valid JSON\\
         \midrule
         BigMath& 1019\\
         DeepMath-103K& 1015\\
         HotpotQA& 1015\\
         HotpotQA-Modified& 1019\\
         \bottomrule
    \end{tabular}
    
    \label{table:valid-json-responses-sanity-check}
\end{table}

For the exact prompt formats used, please refer to Appendix \ref{appendix:prompt-formats}.

\subsection{Training details for SFT}\label{appendix:sft-training-details}

This step is applicable to Qwen 2.5 (3B) Instruct experiments and not applicable to Gemma 4 (E2B) Instruct and Qwen 3 (4B) experiments.

The LLM was loaded in 16-bit. We used LoRA \citep{hu2021loralowrankadaptationlarge} with a rank of 32 and an \(\alpha\) of 32. We used the AdamW \citep{loshchilov2018decoupled} optimizer with a weight decay of 0.01. The learning rate was set to 0.0002 with 5 warmup steps. A linear learning rate schedule was used. The effective batch size is 16. We trained over the SFT dataset, prepared from the previous fine-tuning preparation stage, for 1 epoch.

\subsection{Training details for RL}\label{appendix:rl-training-details}

The LLM was loaded in 16-bit. We used FP16 \citep{ieee7542008} since it was found to be more stable during RL training \citep{qi2025defeatingtraininginferencemismatchfp16} compared to BF16 \citep{bf16-8877390}. However, we used BF16 for Gemma 4 (E2B) Instruct due to the lack of support of FP16 for RL training for the LLM.

We used the AdamW \citep{loshchilov2018decoupled} 8-bit optimizer with \(\beta_1 = 0.9\) and \(\beta_2 = 0.99\). A constant learning rate scheduler was used. The maximum gradient norm was set to 0.1. We used an effective batch size of 512 with 8 generations per question and 64 questions per batch. The LLM temperature was set to 1.

For consistency with the SFT stage, we used LoRA \citep{hu2021loralowrankadaptationlarge} with a rank of 32 and an \(\alpha\) of 32. For Qwen 2.5 (3B) Instruct, the LLM was initialized from the LoRA adapter in the supervised finetuning step because the LLM has been fine-tuned in the supervised finetuning step to follow the XML output format which is used in the current RL step. For Qwen 3 (4B) and Gemma 4 (E2B) Instruct, since the SFT step is skipped, the LLM was initialized from the base model.

The RL code was implemented using Unsloth \citep{unsloth}, significantly decreasing the video random access memory (VRAM) requirements.

The learning rates\footnote{Note: For Qwen 2.5 (3B), while the initial experimental design included 25 warmup steps, a configuration override resulted in 0 warmup steps during execution. The code repository reflects this executed configuration.}, the number of steps and the output token limits are chosen depending on the LLM and the dataset. They are as described in Table \ref{table:rl-hyperparameters}.

\begin{table}
\centering
\caption{Number of steps, learning rate and output token limit values that are used in the experiments}
\label{table:rl-hyperparameters}
\resizebox{\textwidth}{!}{
\begin{tabular}{l l l l l}
\toprule
LLM & Dataset & Step Count & Learning Rate & Output Token Limit \\
\midrule
Qwen 2.5 (3B) Instruct & BigMath & 500 & \(10^{-5}\) & 2048 \\
Qwen 2.5 (3B) Instruct & DeepMath-103K & 500 & \(10^{-5}\) & 2048 \\
Qwen 2.5 (3B) Instruct & HotpotQA & 500 & \(10^{-5}\) & 1024 \\
Qwen 2.5 (3B) Instruct & HotpotQA-Modified & 500 & \(10^{-5}\) & 1024 \\ 
\midrule
Qwen 3 (4B) & HotpotQA-Modified & 200 & \(10^{-5}\) & 4096 \\
Qwen 3 (4B) & DeepMath-103K & 150 & \(2 \times 10^{-5}\) & 8192 \\
\midrule
Gemma 4 (E2B) Instruct & HotpotQA-Modified & 60 & \(10^{-4}\) & 4096 \\
\bottomrule
\end{tabular}
}
\end{table}

For consistency, for Qwen 2.5 (3B) Instruct, the output token limits were the same as the fine-tuning preparation stage.

Note that Qwen 2.5 (3B) Instruct experiments utilized the XML format while the Qwen 3 (4B) and Gemma 4 (E2B) Instruct experiments utilized the JSON format. The formatting rewards for the XML and JSON formats are handled differently, explained in the following subsubsections. For the exact prompt formats used, please refer to Appendix \ref{appendix:prompt-formats}.

\subsubsection{XML Formatting Reward}

For the XML format, our implementation of RL contained additional formatting rewards in addition to the reward in the reward scheme. These were graded independently of the answer. Our formatting reward was based on an Unsloth GRPO tutorial \citep{unsloth}\footnote{https://colab.research.google.com/github/unslothai/notebooks/blob/main/nb/Qwen2.5\_(3B)-GRPO.ipynb}. To discourage overly long answers which would highly likely be graded as incorrect, a valid answer reward of 0.5 was provided if the length did not exceed 1000 characters. A reward of 0.5 was awarded for fully following the XML format. An additional reward of 0.1 was awarded for each opening and closing tag in the XML format that appears exactly once. A reward of 1 was provided by providing a valid confidence value, i.e. an integer between 0 and 100 inclusive. In the event of an invalid confidence value, the worst possible confidence value for the answer was assigned, further discouraging reward hacking of the LLM by not following the format.

\subsubsection{JSON Formatting Reward}\label{appendix:json-formatting-reward}

For the JSON format, our implementation of RL contained an additional formatting reward, which is awarded if and only if the format is considered to be obeyed. The formatting reward is 0.1 for Qwen 3 (4B) HotpotQA-Modified experiments and 0.5 for Qwen 3 (4B) DeepMath-103K and 5 for Gemma 4 (E2B) Instruct HotpotQA-Modified experiments. We empirically noticed that a high formatting reward is required to mitigate the risk of training collapse in Gemma 4 (E2B) Instruct.

For Qwen 3 (4B) experiments, the LLM is guided to output valid JSON using the format enforcer. But there is no guarantee that the JSON produced in the output will be valid due to the output token limit. However, format enforcers are found to worsen the invalid formatting rate of Gemma 4 (E2B) Instruct and are therefore not used in Gemma 4 (E2B) Instruct experiments.

For both Qwen 3 (4B) DeepMath-103K and Gemma 4 (E2B), we attempt to manually fix the broken JSON from the LLM before reading the JSON. First, we attempt to detect the JSON using the first opening and closing brackets, i.e. extra output will only be counted towards the output token limit instead of being outright penalized, provided the JSON can be read. Second, we look for unescaped double-quotes and manually escape them. Then, we use fast-json-repair \footnote{The citation refers to the original source. The actual library that we use is the Rust version for Python: https://pypi.org/project/fast-json-repair/0.2.3/} \citep{Baccianella_JSON_Repair_-_2025} to fix the remaining unescaped characters and other potential syntax errors. As fast-json-repair does not take into account all unescaped double-quotes, we have to handle them separately. 

If the confidence value is invalid, the format is considered invalid and the worst possible confidence value for the answer is assigned. If the answer cannot be read, the format is considered invalid and the answer is considered incorrect with the worst possible assigned confidence value of 100. If the reasoning field or the confidence analysis field are the only fields that cannot be read, then the answer is still graded with the confidence provided, but the formatting reward will not be awarded.

\subsection{Details on Evaluation Protocol}\label{appendix:evaluation-protocol-details}

\begin{table}[h]
\centering
\caption{LLM sampling parameters used during the evaluation stage. The LLMs used for RL training are listed at the top, while the LLMs used for abstention judging are listed at the bottom, demarcated with a row separator. Top-k of `All' means that all tokens are considered for the sampling stage.}
\label{table:llm-sampling-params}
\begin{tabular}{l c c c}
\toprule
LLM & Temperature & Top-p & Top-k \\
\midrule
Qwen 2.5 (3B) Instruct & 1 & 1 & All \\
Qwen 3 (4B) & 0.7 & 0.8 & 20 \\
Gemma 4 (E2B) Instruct & 1 & 0.95 & 64 \\
\midrule
GPT-OSS-20B & 1 & 1 & All \\
Qwen 3 (30B-A3B) FP8 & 0.6 & 0.95 & 20 \\
\bottomrule
\end{tabular}

\end{table}

During the evaluation stage, we tested the base model on the JSON format for Qwen 2.5 (3B) Instruct, Qwen 3 (4B) and Gemma 4 (E2B) Instruct, and XML format for Qwen 2.5 (3B) Instruct. For each LLM, we also evaluated on the SFT model, if applicable, and each of the RL models, corresponding to the different reward schemes. There are two stages. The first stage is to grade the LLM for confidence calibration and accuracy. The second stage, applicable to Qwen 3 (4B) and Gemma 4 (E2B), uses an LLM judge, GPT-OSS-20B or Qwen 3 (30B-A3B) FP8 to judge whether an LLM has abstained from answering a question. For the exact prompt formats used, refer to Appendix \ref{appendix:prompt-formats}.

For the exact sampling parameters used by each LLM, refer to Table \ref{table:llm-sampling-params}. The parameters for Qwen 3 (4B), Gemma 4 (E2B) Instruct, GPT-OSS-20B and Qwen 3 (30B-A3B) FP8 have been adopted based on the recommendations of \cite{qwen3technicalreport}, \cite{gemmateam2026gemma4}, \cite{unsloth}\footnote{https://unsloth.ai/docs/models/gpt-oss-how-to-run-and-fine-tune} and \cite{qwen3technicalreport} respectively. 

\subsubsection{Evaluation for accuracy}

For Qwen 2.5 (3B) Instruct, the JSON format was evaluated to ensure that the SFT training in XML format did not significantly decrease the performance of the base model. The XML format was evaluated to demonstrate that the base model is prompt-sensitive, i.e. performance varies by the exact prompt format used. 

For both the JSON and XML formats, we provided a system prompt instructing the LLM on the format to use and used a format enforcer in vLLM \citep{kwon2023efficient} to guide the LLM towards following the instructed format. The exception is Gemma 4 (E2B) Instruct where using a format enforcer was found to have worsened its ability to follow the format, and therefore not used. For consistency, the output token limits are the same as in the RL stage and the fine-tuning preparation stage.

Inspired by \cite{damani2026beyond}, in the event of invalid output, we gave the LLM a second chance during evaluation and probe them for the final answer and confidence. Without this step, the outputs that cannot be parsed will be judged as wrong. If the invalid formatting rate is high, accuracy statistics may overemphasize whether the LLM can follow the format rather than whether it can perform the instructed task correctly.

If the model was instructed to output in XML, in the event of an invalid XML output, the evaluation code will attempt to extract valid XML tags for answer and confidence as parts of the output may be salvaged. In the event of multiple valid tags, only the last tag will be taken. If both the answer and the confidence can be retrieved, the output will be graded as normal. If the answer cannot be retrieved, the LLM will be asked for the answer in a follow-up. If a valid confidence cannot be retrieved, the LLM will be asked for the confidence in a follow-up. If both the answer and the confidence cannot be taken, the LLM will be asked for the answer, followed by a separate prompt asking for the confidence in a follow-up.

If the model was instructed to format in JSON, the LLM will additionally be guided to provide a valid confidence value, i.e. an integer in between 0 and 100 inclusive and an answer. For Qwen 2.5 (3B) Instruct, we added an additional formatting restriction that the answer cannot exceed 1000 characters. Nearly all ground truths in the BigMath, DeepMath-103K and the HotpotQA dataset are under 1000 characters long. Since an invalid JSON string cannot be parsed, in the event the JSON string is invalid, the LLM will be asked to output its answer in a follow-up prompt, then the confidence. For Qwen 3 (4B) and Gemma 4 (E2B) Instruct, however, we attempt to fix broken JSON before parsing the JSON in the same way as in Appendix \ref{appendix:json-formatting-reward}. Moreover, in Gemma 4 (E2B) Instruct, just as in Appendix \ref{appendix:json-formatting-reward}, a format enforcer was not used. Therefore, it is possible that the answer or the confidence values are missing. In this case, we handle it the same way as we handled the missing answer and confidence values for the XML format.

The follow-up prompts are identical for both JSON and XML formats. Since the initial insructions to output in JSON and XML no longer apply, the system prompt was removed with the remaining conversational trace left intact.

To prevent the LLM from taking extra time to answer the question via extra reasoning beyond the original given token limit, the LLM was tasked to output only the final answer when asked in a follow-up, and in LaTeX format for mathematical datasets. In addition, to reduce the chances of LLM failing to follow instructions, we prefixed the response with "Final Answer:" using a format enforcer. The output token limit, including the prefix, was set to 64 in BigMath and DeepMath-103K and 32 for HotpotQA and HotpotQA-Modified.

When the LLM is asked for its confidence in a follow-up, the LLM will be guided using a format enforcer to output an integer in between 0 and 100 inclusive, representing the confidence value.

Similarly to \cite{damani2026beyond}, once the answer and confidence values were retrieved, the same dataset-dependent correctness checking algorithm is used, allowing a fair comparison.

\subsubsection{Evaluation for abstention}

In addition, for Qwen 3 (4B) and Gemma 4 (E2B) Instruct, we also asked two different LLM judges whether the LLM has abstained from answering the question. The LLM judge is instructed to answer with ``unanswerable'' if the LLM implied that the question is unanswerable, defined as when the LLM implies it has insufficient information as interpreted by the judge. If the LLM abstained in other ways, the LLM judge is instructed to respond with ``abstain''. If the LLM has attempted to answer the question, the LLM judge is instructed to output ``answer''.

The LLM judges used are GPT-OSS-20B and Qwen 3 (30B-A3B) FP8. GPT-OSS-20B was used with low reasoning effort and Qwen 3 (30B-A3B) was used with thinking mode set on and a thinking budget of 2048.

\subsection{Evaluation Metrics Used} \label{appendix:evaluation-metrics-used}

To evaluate our model, we used the following metrics. 
\begin{itemize}
\item \textbf{Accuracy}: The proportion of sampled responses that were judged as correct answers to the corresponding questions.

\item \textbf{Expected Calibration Error (ECE)} \citep{guo2017calibrationmodernneuralnetworks}: Following \cite{damani2026beyond}, let \(M\) be the number of bins and \(N\) questions in total. The ith bin is denoted as \(B_i\), defined as the set of question-answer-confidence tuples whose confidence lies within the ith bin. Let the average confidence of the tuples in the ith bin be denoted as \(\text{conf}(B_i)\) and the accuracy of the tuples in the ith bin be denoted as \(\text{acc}(B_i)\). Then, the ECE is defined as \(\sum_{i=1}^{M} \frac{|B_i|}{N}|\text{acc}(B_i) - \text{conf}(B_i)|\) \citep{damani2026beyond}. Following \cite{damani2026beyond}, we used \(M=10\) bins. In addition, we use equally spaced bins.

\item \textbf{Area Under Receiver Operating Characteristic Curve (AUROC)} \citep{marcum1960auroc, peterson1954auroc}: Since a zero ECE can be obtained by uninformedly aligning the confidence score with the accuracy of the model \citep{devic2025tracelengthsimpleuncertainty}, we use the AUROC metric to measure the ability of the LLM to differentiate between the answers that are likely correct and the answers that are likely wrong via its verbalized confidence. Where TPR stands for the true positive rate and FPR stands for the false positive rate, AUROC is defined as \(\int^{1}_0 \text{TPR}(\text{FPR}^{-1}(t)) \: dt \) \citep{damani2026beyond}. 

\item \textbf{Brier score} \citep{brier1950verification}: Let \(N\) be the number of questions. For question \(i\), let the question be denoted as \(q_i\), the corresponding answer denoted as \(a_i\), and the confidence, normalized in the range \([0, 1]\), be denoted as \(c_i\). Let \(\text{isCorrect}(q_i, a_i)\) denote the function that validates whether \(a_i\) is the correct answer to \(q_i\), returning 1 if the answer is correct and 0 if the answer is wrong. Brier score is a proper scoring metric that is used to evaluate probability calibration abilities, given by the formula: \(\frac{1}{N} \sum_{i=1}^N (\text{isCorrect}(q_i, a_i) - c_i)^2\).

\item \textbf{Average Brier-1 reward} \citep{damani2026beyond}: A strictly non-hackable confidence reward scheme repurposed as an integrated metric that takes into account both the accuracy and the confidence calibration of the model. The average Brier-1 reward is equal to the difference between the accuracy and the Brier score since Brier-1 offers an additional reward of 1 to the negative of the Brier score if and only if the answer is correct.

\item \textbf{Calibration bias}: Calibration bias is defined as the accuracy minus the average confidence over the evaluated answers. The average confidence over the evaluated answers is the expected accuracy assuming the confidence values are well-calibrated. A negative value means that the model is overconfident on average, while a positive value means that the model is underconfident on average.

\end{itemize}

\subsection{Prompt formats} \label{appendix:prompt-formats}

For Qwen 2.5 (3B) Instruct, the JSON system prompt format is used for the fine-tuning preparation stage, and when evaluating base models. The XML prompt format is used during the SFT stage to train the LLM to follow instructions for the RL stage. It is also used during the RL stage and the XML evaluation for consistency.

For Qwen 3 (4B) and Gemma 4 (E2B) Instruct experiments, we use an updated version of the JSON system prompt format for both RL training and evaluation.

The JSON and XML system prompt formats are as shown below.

\begin{mdframed}[linecolor=black, backgroundcolor=gray!10, roundcorner=20pt, innertopmargin=5pt, innerbottommargin=5pt] 

\textbf{JSON system prompt format (Qwen 2.5 (3B) Instruct)}

\nointerlineskip

\rule{\linewidth}{0.5pt}

\nointerlineskip

Respond in the JSON format: \{"reasoning": string, "answer": string, "confidence\_analysis": string, "confidence": int\}

\mbox{}

Provide your internal thought process in "reasoning" attribute.

Provide only your final answer in "answer" attribute.

For mathematical questions, provide your final answer in LaTeX format.
Provide your analysis on the uncertainty of your answer in "confidence\_analysis" attribute, taking into account various factors that may lead to your answer being different or incorrect.

In "confidence" attribute, express your confidence as an integer between 0 and 100 inclusive.

\end{mdframed}

\mbox{}

\begin{mdframed}[linecolor=black, backgroundcolor=gray!10, roundcorner=20pt, innertopmargin=5pt, innerbottommargin=5pt] 

\textbf{JSON system prompt format (Qwen 3 (4B) and Gemma 4 (E2B) Instruct)}

\nointerlineskip

\rule{\linewidth}{0.5pt}

\nointerlineskip

Respond in the JSON format: \{"reasoning": string, "answer": string, "confidence\_analysis": string, "confidence": int\}

\mbox{}

Provide your detailed internal thought process in "reasoning" attribute.
Provide only your final answer in "answer" attribute.
For mathematical questions, provide your final answer in LaTeX format.
Provide your analysis on the uncertainty of your answer in "confidence\_analysis" attribute, taking into account various factors that may lead to your answer being different or incorrect, such as errors in calculation and ambiguities in answer or grading.
In "confidence" attribute, express your confidence as an integer between 0 and 100 inclusive, representing the subjective probability that your answer is correct.

You score points based on whether your answer correctly and how well-calibrated your confidence is. Your objective is to maximize the total number of points scored over all the questions asked.

\end{mdframed}

\mbox{}

\begin{mdframed}[linecolor=black, backgroundcolor=gray!10, roundcorner=20pt, innertopmargin=5pt, innerbottommargin=5pt] 

\textbf{XML system prompt format}

\nointerlineskip

\rule{\linewidth}{0.5pt}

\nointerlineskip

When answering questions, follow these instructions:

1) Enclose your internal thought process with <reasoning> and </reasoning> tags.

2) Enclose your final answer with <answer> and </answer> tags. For mathematical answers, answer in LaTeX format.

3) Enclose your analysis on the uncertainty of your answer with <confidence\_analysis> and </confidence\_analysis> tags, taking into account various factors that may lead to your answer being different or incorrect.

4) Enclose your confidence with <confidence> and </confidence> tags. Confidence is an integer between 0 and 100 inclusive, with higher values indicating higher confidence. Higher confidence means higher score if answer is correct but lower score if answer is incorrect. Your aim is to maximize your score considering the confidence of your answer given your internal thought process to the question.

Respond in the following format:

<reasoning>

...

</reasoning>

<answer>

...

</answer>

<confidence\_analysis>

...

</confidence\_analysis>

<confidence>

...

</confidence>

\end{mdframed}

\mbox{}

For BigMath and DeepMath-103K, the prompt is simply the question in the dataset. For HotpotQA and HotpotQA-Modified, there are multiple sources which need to be presented to form a prompt. The sources are presented in random order to mitigate the effects of data leakage. The LLM prompt for questions in the HotpotQA and HotpotQA-Modified datasets are as follows, where \(N\) is the number of sources in the question:

\begin{mdframed}[linecolor=black, backgroundcolor=gray!10, roundcorner=20pt, innertopmargin=5pt, innerbottommargin=5pt] 

\textbf{Question prompt format for HotpotQA and HotpotQA-Modified}

\nointerlineskip

\rule{\linewidth}{0.5pt}

\nointerlineskip

Based on the information sources given below and your existing knowledge, answer the following question: [QUESTION]

\mbox{}

Source 1: [TITLE OF SOURCE 1]

[TEXT OF SOURCE 1]

\mbox{}

Source 2: [TITLE OF SOURCE 2]

[TEXT OF SOURCE 2]

...

Source [N]: [TITLE OF SOURCE N]

[TEXT OF SOURCE N]

\end{mdframed}

For the evaluation, sometimes, the LLM may produce an invalid response despite the help of JSON repair \citep{Baccianella_JSON_Repair_-_2025} for Gemma 4 (E2B) Instruct and a format enforcer for other LLMs. Therefore, the LLM may be prompted a second time for the answer or the confidence. Their respective prompt formats are as shown below.

\begin{mdframed}[linecolor=black, backgroundcolor=gray!10, roundcorner=20pt, innertopmargin=5pt, innerbottommargin=5pt] 

\textbf{Prompt format for second LLM probe for answer (BigMath and DeepMath-103K)}

\nointerlineskip

\rule{\linewidth}{0.5pt}

\nointerlineskip

Reasoning token limit reached. Please output only your final answer within 64 tokens. Express your answer in LaTeX.

\end{mdframed}

\begin{mdframed}[linecolor=black, backgroundcolor=gray!10, roundcorner=20pt, innertopmargin=5pt, innerbottommargin=5pt] 

\textbf{Prompt format for second LLM probe for answer (HotpotQA and HotpotQA-Modified)} 

\nointerlineskip

\rule{\linewidth}{0.5pt}

\nointerlineskip

Reasoning token limit reached. Please output only your final answer within 32 tokens.

\end{mdframed}

Note that the token limit for the second probe of the answer is 64 for BigMath and DeepMath-103K and 32 for HotpotQA and HotpotQA-Modified.

\begin{mdframed}[linecolor=black, backgroundcolor=gray!10, roundcorner=20pt, innertopmargin=5pt, innerbottommargin=5pt] 

\textbf{Prompt format for second LLM probe for confidence}

\nointerlineskip

\rule{\linewidth}{0.5pt}

\nointerlineskip

Please output your confidence as an integer between 0 and 100 inclusive.

\end{mdframed}

For Qwen 3 (4B) and Gemma 4 (E2B) Instruct, we additionally evaluated the LLM output for abstention. The prompt format is shown below.

\begin{mdframed}[linecolor=black, backgroundcolor=gray!10, roundcorner=20pt, innertopmargin=5pt, innerbottommargin=5pt] 

\textbf{Prompt format for abstention judge (i.e. GPT-OSS-20B or Qwen 3 (30B-A3B) FP8)}

\nointerlineskip

\rule{\linewidth}{0.5pt}

\nointerlineskip

Given the question and the answer written in JSON format, determine whether the LLM opted to abstain from answering the question.

\mbox{}

Respond with the following (without the quotes):

(i) "unanswerable" if the LLM has responded that the question is unanswerable due to having insufficient information, regardless of whether the question is actually answerable.

(ii) "abstain" if the LLM has abstained from answering the question in a manner that does not otherwise qualify as "unanswerable". If the LLM abstains by implying that the question is unanswerable, respond with "unanswerable".

(iii) "answer" if the LLM has answered the question, regardless of whether the question is answered correctly.

\mbox{}

Answer with only "unanswerable", "abstain" or "answer", without the quotes. Do not attempt to answer or solve the question given in JSON because this is irrelevant in judging whether or how the LLM has abstained from answering the question.

\mbox{}

\{"question": [FORMATTED QUESTION IN DATASET], "answer": [PROVIDED LLM ANSWER]\}

\end{mdframed}

\subsection{Hardware used} \label{appendix:hardware-used}

Only central processing unit (CPU) power is required for dataset preparation and result analysis. An NVIDIA GPU was used for RL training and evaluation, with the exact configuration depending on the LLM and dataset as described in the following subsubsections.

\subsubsection{Qwen 2.5 (3B) Instruct}

For our finetuning dataset preparation, SFT, RL and evaluation stages, we used a single instance of 3g.47gb in an NVIDIA H100 graphics processing unit (GPU). For finetuning dataset preparation, supervised finetuning, RL, we used 48 GiB of CPU random access memory (RAM). For the evaluation stage, we used 128 GiB of CPU RAM.

For Qwen 2.5 (3B) Instruct, the preprocessing stage and the finetuning stage each took under 10 minutes per dataset. The RL experiment took anywhere between 20 hours and 3 days depending on the dataset and the reward scheme. The evaluation stage completed in under 3 days per dataset, with execution time varying depending on dataset.

\subsubsection{Qwen 3 (4B)}

For Qwen 3 (4B), a single instance of 3g.47gb in an NVIDIA H100 GPU is used for RL training for HotpotQA-Modified. A single instance of 3g.47gb in an NVIDIA H100 GPU is used for evaluation on HotpotQA-Modified, except when evaluating for abstention using Qwen 3 (30B-A3B), where an NVIDIA H100 GPU is used. One NVIDIA H100 GPU is used for RL training and evaluation for DeepMath-103K. The CPU RAM limit is set to 128 GiB in all Qwen 3 (4B) RL experiments and their corresponding evaluations.

\subsubsection{Gemma 4 (E2B) Instruct}

For Gemma 4 (E2B) Instruct, one NVIDIA H200 graphics processing unit (GPU) is used for RL training and evaluation for HotpotQA-Modified, with the CPU RAM limit set to 128 GiB. 

\subsection{Dataset and Model Licenses} \label{appendix:dataset-and-model-licenses}

DeepMath-103K \citep{deepmath} is MIT licensed, BigMath \citep{albalak2025bigmathlargescalehighqualitymath} has Apache 2.0 license, and HotpotQA \citep{yang2018hotpotqa} has CC-BY-SA-4.0 licence, which allows derivatives of the work such as HotpotQA-Modified. 

Qwen 2.5 (3B) Instruct \citep{qwen2.5} is licensed using Qwen Research Agreement license, which permits use only for research and evaluation purposes. Llama 3.2 (3B) Instruct \citep{grattafiori2024llama3herdmodels} is licensed using the Llama 3.2 Community License Agreement. Llama 3.1 (8B) Instruct \citep{grattafiori2024llama3herdmodels} is licensed using the Llama 3.1 Community License Agreement. 

Both Qwen 3 (4B) \citep{qwen3technicalreport} and Gemma 4 (E2B) Instruct \citep{gemmateam2026gemma4} are licensed under the Apache 2.0 license.

All the licenses involved permit use for research.

\section{Full results}\label{appendix:full-results}

This section contains the results of the experiments that we have completed. It includes accuracy, AUROC, Brier-1, Brier score, ECE (10-bins), calibration biases, and abstention rates (if assessed) of the experiments that we conduct. We split this section into subsections by LLM, i.e. Qwen 2.5 (3B) Instruct, Qwen 3 (4B) and Gemma 4 (E2B) Instruct. In the final subsection, we briefly discuss the invalid formatting rates recorded by the model.

\subsection{Qwen 2.5 (3B) Instruct}\label{appendix:qwen2.5-3b-results}

This subsection provides the accuracy, AUROC, Brier-1, Brier score, ECE (10-bins) and calibration bias of all datasets used in the experiment for Qwen 2.5 (3B) Instruct. For accuracy, AUROC and Brier-1, higher is better, while for Brier score, ECE (10-bins) and calibration bias, closer to 0 is better.

The results are split by dataset difficulty. Note that the HotpotQA test set only contains questions labeled as hard, therefore, only the overall results in HotpotQA are shown.

The AUROC is not shown if the accuracy is outside the range \((0.001, 0.999)\) because the small number of correct answers or the small number of wrong answers may not form a reliable sample of the ability of the LLM to differentiate correct answers from wrong answers. 

For all tables in the section, base models and SFT models are included for reference at the top, while the middle section contains hackable confidence reward schemes and the bottom contains the reward schemes in the overconfidence to underconfidence spectrum, arranged from the overconfident end at the top to the underconfident end at the bottom.

The SFT (XML) model performed similarly to the Base (JSON) model. The SFT (XML) model was fine-tuned via SFT from reformatted outputs of Base (JSON) to XML format. This shows that the SFT stage did not result in a significant loss of accuracy from the original model.

The Base (XML) format was shown to highlight the prompt sensitivity of LLM, demonstrating the importance of the exact prompt format used in evaluation. This draws parallels to \cite{liu2025understandingr1zeroliketrainingcritical}, which highlighted that including a prompt template for Qwen 2.5 models negatively impacts performance prior to RL.

Tables \ref{table:bigmath-qwen2.5-3b-overall-results}, \ref{table:bigmath-qwen2.5-3b-easy-results},  \ref{table:bigmath-qwen2.5-3b-medium-results} and \ref{table:bigmath-qwen2.5-3b-hard-results} respectively show the performance statistics of Qwen 2.5 (3B) Instruct when trained on BigMath and tested on the entire test set, easy subset, medium subset and hard subset, of the test set of BigMath.

Tables \ref{table:deepmath-103k-qwen2.5-3b-overall-results}, \ref{table:deepmath-103k-qwen2.5-3b-easy-results},  \ref{table:deepmath-103k-qwen2.5-3b-medium-results} and \ref{table:deepmath-103k-qwen2.5-3b-hard-results} respectively show the performance statistics of Qwen 2.5 (3B) Instruct when trained on DeepMath-103K and tested on the entire test set, easy subset, medium subset and hard subset, of the test set of DeepMath-103K.

Table \ref{table:hotpotqa-qwen2.5-3b-overall-results} shows the performance statistics of Qwen 2.5 (3B) Instruct when trained on HotpotQA and tested on the test set of HotpotQA.

Tables \ref{table:hotpotqa-modified-qwen2.5-3b-overall-results}, \ref{table:hotpotqa-modified-qwen2.5-3b-easy-results},  \ref{table:hotpotqa-modified-qwen2.5-3b-medium-results} and \ref{table:hotpotqa-modified-qwen2.5-3b-hard-results} respectively show the performance statistics of Qwen 2.5 (3B) Instruct when trained on HotpotQA-Modified and tested on the entire test set, easy subset, medium subset and hard subset, of the test set of HotpotQA-Modified.

For most cases, the invalid formatting rates are low and are therefore unlikely to affect the results. More details on invalid formatting rates can be found on subsection \ref{appendix:invalid-formatting-rates}.

\begingroup
\renewcommand{\arraystretch}{0.9}

\begin{table}[h!]
\centering
\caption{Accuracy, AUROC, Brier-1, Brier, ECE (10 bins) and calibration bias statistics for the test set of BigMath when Qwen 2.5 (3B) Instruct is trained with BigMath dataset. Calib. Bias is short for Calibration bias.}
\label{table:bigmath-qwen2.5-3b-overall-results}
\begin{tabular}{l c c c c c c} 
\toprule
Model/Reward Scheme& Accuracy & AUROC & Brier-1 & Brier & ECE & Calib. Bias\\
\midrule
Base (JSON) & 0.2531& 0.5908& -0.4212& 0.6743& 0.6988& -0.6983
\\
Base (XML) & 0.5644& 0.5982& 0.1736& 0.3908& 0.3910& -0.3844
\\
SFT (XML) & 0.2499& 0.5959& -0.4266& 0.6765& 0.7018& -0.7017
\\
\midrule
Log Loss & 0.0001& -& 0.0000& 0.0001& 0.0049& -0.0049
\\
Brier Score & 0.0001& -& 0.0000& 0.0002& 0.0049& -0.0049
\\
Log-1 & 0.6336& 0.8602& 0.4841& 0.1495& 0.0826& 0.0826
\\
Brier-2 & 0.6713& 0.8434& 0.5278& 0.1436& 0.0388& 0.0028
\\
Log-\(\frac{1}{\ln 202}\)& 0.6656& 0.7849& 0.4885& 0.1771& 0.0649& -0.0503
\\
\midrule
Correctness-only & 0.6739& 0.5277& 0.3565& 0.3174& 0.3171& -0.3171
\\
Overconfidence-1000 & 0.6704& 0.6681& 0.4364& 0.2340& 0.1826& -0.1825
\\
Overconfidence-4 & 0.6685& 0.7492& 0.4908& 0.1777& 0.0564& -0.0376
\\
Overconfidence-1 & 0.6723& 0.7669& 0.5080& 0.1643& 0.0560& -0.0151
\\
Brier-1 & 0.6676& 0.8195& 0.5151& 0.1525& 0.0244& -0.0243
\\
Underconfidence-1 & 0.6666& 0.8282& 0.5132& 0.1534& 0.0307& -0.0192
\\
Underconfidence-4 & 0.6769& 0.8343& 0.5272& 0.1498& 0.0441& 0.0423
\\
Brier-log Hybrid & 0.6682& 0.8521& 0.5253& 0.1429& 0.0474& -0.0054
\\
\midrule

\end{tabular}

\end{table}

\begin{table}[h!]
\centering
\caption{Accuracy, AUROC, Brier-1, Brier, ECE (10 bins) and calibration bias statistics for the easy subset of the test set of BigMath when Qwen 2.5 (3B) Instruct is trained with BigMath dataset. Calib. Bias is short for Calibration bias.}
\label{table:bigmath-qwen2.5-3b-easy-results}
\begin{tabular}{l c c c c c c} 
\toprule
Model/Reward Scheme& Accuracy & AUROC & Brier-1 & Brier & ECE & Calib. Bias\\
\midrule
Base (JSON) & 0.4391& 0.5595& -0.0842& 0.5233& 0.5321& -0.5315
\\
Base (XML) & 0.8587& 0.5524& 0.7228& 0.1359& 0.1175& -0.1093
\\
SFT (XML) & 0.4494& 0.5681& -0.0628& 0.5122& 0.5219& -0.5217
\\
\midrule
Log Loss & 0.0001& -& 0.0000& 0.0002& 0.0048& -0.0048
\\
Brier Score & 0.0003& -& 0.0000& 0.0004& 0.0047& -0.0047
\\
Log-1 & 0.9388& 0.7847& 0.8574& 0.0814& 0.1645& 0.1645
\\
Brier-2 & 0.9530& 0.7412& 0.8981& 0.0549& 0.0866& 0.0743
\\
Log-\(\frac{1}{\ln 202}\)& 0.9501& 0.6862& 0.8837& 0.0664& 0.1351& 0.1351
\\
\midrule
Correctness-only & 0.9540& 0.5181& 0.9090& 0.0450& 0.0405& -0.0405
\\
Overconfidence-1000 & 0.9529& 0.5987& 0.9035& 0.0495& 0.0670& 0.0621
\\
Overconfidence-4 & 0.9509& 0.6201& 0.8804& 0.0705& 0.1540& 0.1533
\\
Overconfidence-1 & 0.9521& 0.6179& 0.8836& 0.0685& 0.1597& 0.1596
\\
Brier-1 & 0.9529& 0.6756& 0.9012& 0.0517& 0.0923& 0.0921
\\
Underconfidence-1 & 0.9505& 0.7292& 0.8926& 0.0579& 0.1001& 0.0996
\\
Underconfidence-4 & 0.9552& 0.7240& 0.8925& 0.0626& 0.1233& 0.1225
\\
Brier-log Hybrid & 0.9513& 0.7566& 0.8994& 0.0519& 0.1010& 0.1009
\\
\midrule

\end{tabular}

\end{table}

\begin{table}[h!]
\centering
\caption{Accuracy, AUROC, Brier-1, Brier, ECE (10 bins) and calibration bias statistics for the medium subset of the test set of BigMath when Qwen 2.5 (3B) Instruct is trained with BigMath dataset. Calib. Bias is short for Calibration bias.}
\label{table:bigmath-qwen2.5-3b-medium-results}
\begin{tabular}{l c c c c c c} 
\toprule
Model/Reward Scheme& Accuracy & AUROC & Brier-1 & Brier & ECE & Calib. Bias\\
\midrule
Base (JSON) & 0.2387& 0.5701& -0.4519& 0.6905& 0.7135& -0.7130
\\
Base (XML) & 0.6096& 0.5734& 0.2559& 0.3538& 0.3472& -0.3381
\\
SFT (XML) & 0.2262& 0.5702& -0.4767& 0.7030& 0.7263& -0.7262
\\
\midrule
Log Loss & 0.0001& -& 0.0000& 0.0002& 0.0048& -0.0048
\\
Brier Score & 0.0001& -& 0.0000& 0.0001& 0.0050& -0.0050
\\
Log-1 & 0.6981& 0.7655& 0.5074& 0.1907& 0.1417& 0.1417
\\
Brier-2 & 0.7520& 0.7580& 0.5902& 0.1618& 0.0929& 0.0735
\\
Log-\(\frac{1}{\ln 202}\)& 0.7427& 0.6964& 0.5712& 0.1715& 0.0311& 0.0227
\\
\midrule
Correctness-only & 0.7545& 0.5231& 0.5149& 0.2395& 0.2376& -0.2374
\\
Overconfidence-1000 & 0.7506& 0.6171& 0.5610& 0.1896& 0.1088& -0.1080
\\
Overconfidence-4 & 0.7474& 0.6718& 0.5761& 0.1712& 0.0354& 0.0272
\\
Overconfidence-1 & 0.7532& 0.6829& 0.5903& 0.1629& 0.0428& 0.0391
\\
Brier-1 & 0.7480& 0.7244& 0.5823& 0.1657& 0.0507& 0.0318
\\
Underconfidence-1 & 0.7444& 0.7373& 0.5785& 0.1660& 0.0461& 0.0445
\\
Underconfidence-4 & 0.7573& 0.7436& 0.5832& 0.1741& 0.1143& 0.1134
\\
Brier-log Hybrid & 0.7460& 0.7600& 0.5878& 0.1582& 0.0532& 0.0532
\\
\midrule

\end{tabular}

\end{table}

\begin{table}[h!]
\centering
\caption{Accuracy, AUROC, Brier-1, Brier, ECE (10 bins) and calibration bias statistics for the hard subset of the test set of BigMath when Qwen 2.5 (3B) Instruct is trained with BigMath dataset. Calib. Bias is short for Calibration bias.}
\label{table:bigmath-qwen2.5-3b-hard-results}
\begin{tabular}{l c c c c c c} 
\toprule
Model/Reward Scheme& Accuracy & AUROC & Brier-1 & Brier & ECE & Calib. Bias\\
\midrule
Base (JSON) & 0.0892& 0.5631& -0.7134& 0.8027& 0.8440& -0.8436
\\
Base (XML) & 0.2358& 0.5598& -0.4376& 0.6734& 0.6995& -0.6957
\\
SFT (XML) & 0.0823& 0.5723& -0.7250& 0.8073& 0.8496& -0.8496
\\
\midrule
Log Loss & 0.0000& -& 0.0000& 0.0000& 0.0050& -0.0050
\\
Brier Score & 0.0000& -& 0.0000& 0.0000& 0.0050& -0.0050
\\
Log-1 & 0.2748& 0.7581& 0.1019& 0.1729& 0.0569& -0.0561
\\
Brier-2 & 0.3185& 0.7265& 0.1085& 0.2100& 0.1387& -0.1380
\\
Log-\(\frac{1}{\ln 202}\)& 0.3139& 0.6824& 0.0248& 0.2891& 0.3028& -0.3028
\\
\midrule
Correctness-only & 0.3229& 0.5225& -0.3355& 0.6584& 0.6639& -0.6639
\\
Overconfidence-1000 & 0.3172& 0.6105& -0.1393& 0.4565& 0.4935& -0.4934
\\
Overconfidence-4 & 0.3169& 0.6635& 0.0296& 0.2873& 0.2869& -0.2869
\\
Overconfidence-1 & 0.3211& 0.6949& 0.0633& 0.2577& 0.2382& -0.2382
\\
Brier-1 & 0.3117& 0.7084& 0.0758& 0.2359& 0.1933& -0.1932
\\
Underconfidence-1 & 0.3146& 0.7136& 0.0822& 0.2324& 0.1991& -0.1984
\\
Underconfidence-4 & 0.3279& 0.7136& 0.1191& 0.2087& 0.1081& -0.1073
\\
Brier-log Hybrid & 0.3169& 0.7330& 0.1022& 0.2148& 0.1674& -0.1671
\\
\midrule

\end{tabular}

\end{table}

\begin{table}[h!]
\centering
\caption{Accuracy, AUROC, Brier-1, Brier, ECE (10 bins) and calibration bias statistics for the test set of DeepMath-103K when Qwen 2.5 (3B) Instruct is trained with DeepMath-103K dataset. Calib. Bias is short for Calibration bias.}
\label{table:deepmath-103k-qwen2.5-3b-overall-results}
\begin{tabular}{l c c c c c c} 
\toprule
Model/Reward Scheme& Accuracy & AUROC & Brier-1 & Brier & ECE & Calib. Bias\\
\midrule
Base (JSON) & 0.1796& 0.5713& -0.5268& 0.7063& 0.7427& -0.7415
\\
Base (XML) & 0.3228& 0.5577& -0.2743& 0.5971& 0.6145& -0.6078
\\
SFT (XML) & 0.1634& 0.5732& -0.5626& 0.7260& 0.7636& -0.7633
\\
\midrule
Log Loss & 0.0001& -& 0.0000& 0.0002& 0.0049& -0.0049
\\
Brier Score & 0.0001& -& 0.0000& 0.0001& 0.0049& -0.0049
\\
Log-1 & 0.0000& -& 0.0000& 0.0001& 0.0050& -0.0050
\\
Brier-2 & 0.0002& -& 0.0000& 0.0002& 0.0048& -0.0048
\\
Log-\(\frac{1}{\ln 202}\)& 0.4327& 0.6119& 0.1952& 0.2375& 0.0556& 0.0527
\\
\midrule
Correctness-only & 0.4238& 0.5034& -0.1437& 0.5675& 0.5686& -0.5684
\\
Overconfidence-1000 & 0.4204& 0.5446& 0.1737& 0.2468& 0.0822& -0.0821
\\
Overconfidence-4 & 0.4294& 0.6609& 0.2052& 0.2242& 0.0155& -0.0140
\\
Overconfidence-1 & 0.4138& 0.6752& 0.1930& 0.2207& 0.0341& 0.0299
\\
Brier-1 & 0.4063& 0.6558& 0.1828& 0.2235& 0.0468& 0.0082
\\
Underconfidence-1 & 0.4159& 0.6375& 0.1870& 0.2289& 0.0296& -0.0153
\\
Underconfidence-4 & 0.3921& 0.6452& 0.1684& 0.2237& 0.0422& 0.0360
\\
Brier-log Hybrid & 0.2635& 0.6302& 0.0781& 0.1854& 0.0273& 0.0271
\\
\midrule

\end{tabular}

\end{table}

\begin{table}[h!]
\centering
\caption{Accuracy, AUROC, Brier-1, Brier, ECE (10 bins) and calibration bias statistics for the easy subset of the test set of DeepMath-103K when Qwen 2.5 (3B) Instruct is trained with DeepMath-103K dataset. Calib. Bias is short for Calibration bias.}
\label{table:deepmath-103k-qwen2.5-3b-easy-results}
\begin{tabular}{l c c c c c c} 
\toprule
Model/Reward Scheme& Accuracy & AUROC & Brier-1 & Brier & ECE & Calib. Bias\\
\midrule
Base (JSON) & 0.1958& 0.5777& -0.5148& 0.7107& 0.7412& -0.7404
\\
Base (XML) & 0.4331& 0.5909& -0.0689& 0.5020& 0.5118& -0.5070
\\
SFT (XML) & 0.1785& 0.5796& -0.5570& 0.7355& 0.7660& -0.7660
\\
\midrule
Log Loss & 0.0000& -& 0.0000& 0.0000& 0.0051& -0.0051
\\
Brier Score & 0.0001& -& -0.0001& 0.0001& 0.0051& -0.0051
\\
Log-1 & 0.0001& -& 0.0000& 0.0001& 0.0050& -0.0050
\\
Brier-2 & 0.0000& -& 0.0000& 0.0000& 0.0050& -0.0050
\\
Log-\(\frac{1}{\ln 202}\)& 0.5334& 0.6897& 0.2933& 0.2401& 0.1270& 0.1249
\\
\midrule
Correctness-only & 0.5308& 0.5040& 0.0689& 0.4619& 0.4618& -0.4617
\\
Overconfidence-1000 & 0.5262& 0.5899& 0.2904& 0.2358& 0.0120& 0.0039
\\
Overconfidence-4 & 0.5324& 0.7571& 0.3292& 0.2032& 0.0786& 0.0359
\\
Overconfidence-1 & 0.5205& 0.7646& 0.3118& 0.2087& 0.0938& 0.0860
\\
Brier-1 & 0.5108& 0.7142& 0.2843& 0.2265& 0.1289& 0.0834
\\
Underconfidence-1 & 0.5191& 0.6986& 0.2907& 0.2284& 0.1143& 0.0695
\\
Underconfidence-4 & 0.4948& 0.7209& 0.2692& 0.2256& 0.1062& 0.1061
\\
Brier-log Hybrid & 0.3313& 0.7306& 0.1347& 0.1966& 0.0750& 0.0649
\\
\midrule

\end{tabular}

\end{table}

\begin{table}[h!]
\centering
\caption{Accuracy, AUROC, Brier-1, Brier, ECE (10 bins) and calibration bias statistics for the medium subset of the test set of DeepMath-103K when Qwen 2.5 (3B) Instruct is trained with DeepMath-103K dataset. Calib. Bias is short for Calibration bias.}
\label{table:deepmath-103k-qwen2.5-3b-medium-results}
\begin{tabular}{l c c c c c c} 
\toprule
Model/Reward Scheme& Accuracy & AUROC & Brier-1 & Brier & ECE & Calib. Bias\\
\midrule
Base (JSON) & 0.1508& 0.5724& -0.5790& 0.7298& 0.7706& -0.7695
\\
Base (XML) & 0.3016& 0.5557& -0.3137& 0.6153& 0.6351& -0.6280
\\
SFT (XML) & 0.1398& 0.5677& -0.6080& 0.7479& 0.7876& -0.7874
\\
\midrule
Log Loss & 0.0000& -& 0.0000& 0.0001& 0.0049& -0.0049
\\
Brier Score & 0.0001& -& 0.0000& 0.0001& 0.0049& -0.0049
\\
Log-1 & 0.0000& -& 0.0000& 0.0001& 0.0049& -0.0049
\\
Brier-2 & 0.0001& -& 0.0000& 0.0001& 0.0049& -0.0049
\\
Log-\(\frac{1}{\ln 202}\)& 0.4059& 0.6051& 0.1732& 0.2328& 0.0328& 0.0283
\\
\midrule
Correctness-only & 0.3940& 0.5037& -0.2030& 0.5970& 0.5986& -0.5984
\\
Overconfidence-1000 & 0.3886& 0.5359& 0.1399& 0.2487& 0.1123& -0.1123
\\
Overconfidence-4 & 0.4054& 0.6549& 0.1817& 0.2237& 0.0335& -0.0335
\\
Overconfidence-1 & 0.3855& 0.6683& 0.1681& 0.2174& 0.0253& 0.0057
\\
Brier-1 & 0.3796& 0.6575& 0.1611& 0.2185& 0.0489& -0.0170
\\
Underconfidence-1 & 0.3876& 0.6344& 0.1616& 0.2260& 0.0402& -0.0401
\\
Underconfidence-4 & 0.3613& 0.6433& 0.1450& 0.2163& 0.0180& 0.0088
\\
Brier-log Hybrid & 0.2263& 0.6178& 0.0568& 0.1695& 0.0079& -0.0039
\\
\midrule
\end{tabular}
\end{table}

\begin{table}[h!]
\centering
\caption{Accuracy, AUROC, Brier-1, Brier, ECE (10 bins) and calibration bias statistics for the hard subset of the test set of DeepMath-103K when Qwen 2.5 (3B) Instruct is trained with DeepMath-103K dataset. Calib. Bias is short for Calibration bias.}
\label{table:deepmath-103k-qwen2.5-3b-hard-results}
\begin{tabular}{l c c c c c c} 
\toprule
Model/Reward Scheme& Accuracy & AUROC & Brier-1 & Brier & ECE & Calib. Bias\\
\midrule
Base (JSON) & 0.2204& 0.5671& -0.4435& 0.6639& 0.6957& -0.6941
\\
Base (XML) & 0.3015& 0.5334& -0.3140& 0.6155& 0.6331& -0.6260
\\
SFT (XML) & 0.1958& 0.5793& -0.4876& 0.6835& 0.7212& -0.7207
\\
\midrule
Log Loss & 0.0003& -& 0.0000& 0.0004& 0.0047& -0.0047
\\
Brier Score & 0.0001& -& 0.0000& 0.0002& 0.0048& -0.0048
\\
Log-1 & 0.0000& -& 0.0000& 0.0000& 0.0050& -0.0050
\\
Brier-2 & 0.0004& -& 0.0000& 0.0004& 0.0046& -0.0046
\\
Log-\(\frac{1}{\ln 202}\)& 0.4260& 0.5671& 0.1818& 0.2441& 0.0586& 0.0567
\\
\midrule
Correctness-only & 0.4189& 0.5024& -0.1533& 0.5722& 0.5731& -0.5728
\\
Overconfidence-1000 & 0.4197& 0.5273& 0.1704& 0.2493& 0.0776& -0.0754
\\
Overconfidence-4 & 0.4167& 0.6007& 0.1807& 0.2360& 0.0304& -0.0068
\\
Overconfidence-1 & 0.4063& 0.6217& 0.1736& 0.2327& 0.0476& 0.0420
\\
Brier-1 & 0.3976& 0.6062& 0.1670& 0.2306& 0.0184& 0.0122
\\
Underconfidence-1 & 0.4105& 0.5993& 0.1762& 0.2342& 0.0179& -0.0172
\\
Underconfidence-4 & 0.3911& 0.5921& 0.1559& 0.2352& 0.0523& 0.0459
\\
Brier-log Hybrid & 0.2919& 0.5800& 0.0850& 0.2069& 0.0637& 0.0603
\\
\midrule

\end{tabular}

\end{table}

\begin{table}[h!]
\centering
\caption{Accuracy, AUROC, Brier-1, Brier, ECE (10 bins) and calibration bias statistics for the test set of HotpotQA when Qwen 2.5 (3B) Instruct is trained with HotpotQA dataset. Calib. Bias is short for Calibration bias.}
\label{table:hotpotqa-qwen2.5-3b-overall-results}
\begin{tabular}{l c c c c c c} 
\toprule
Model/Reward Scheme& Accuracy & AUROC & Brier-1 & Brier & ECE & Calib. Bias\\
\midrule
Base (JSON) & 0.4404& 0.5850& -0.0286& 0.4689& 0.4730& -0.4693
\\
Base (XML) & 0.0479& 0.5315& -0.7309& 0.7788& 0.8399& -0.8399
\\
SFT (XML) & 0.4362& 0.6044& -0.0208& 0.4570& 0.4644& -0.4634
\\
\midrule
Log Loss & 0.0000& -& 0.0000& 0.0001& 0.0050& -0.0049
\\
Brier Score & 0.0000& -& -0.0001& 0.0001& 0.0058& -0.0058
\\
Log-1 & 0.6472& 0.6420& 0.4277& 0.2194& 0.0777& -0.0776
\\
Brier-2 & 0.6507& 0.5920& 0.4221& 0.2286& 0.1057& -0.1057
\\
Log-\(\frac{1}{\ln 202}\)& 0.6537& 0.5455& 0.4077& 0.2461& 0.1491& -0.1491
\\
\midrule
Correctness-only & 0.6571& 0.5095& 0.3207& 0.3363& 0.3346& -0.3345
\\
Overconfidence-1000 & 0.6548& 0.5273& 0.3220& 0.3328& 0.3307& -0.3305
\\
Overconfidence-4 & 0.6526& 0.5725& 0.4134& 0.2393& 0.1319& -0.1319
\\
Overconfidence-1 & 0.6507& 0.5854& 0.4059& 0.2448& 0.1563& -0.1562
\\
Brier-1 & 0.6504& 0.5790& 0.4144& 0.2360& 0.1225& -0.1225
\\
Underconfidence-1 & 0.6551& 0.6046& 0.4238& 0.2313& 0.1175& -0.1175
\\
Underconfidence-4 & 0.6529& 0.6158& 0.4242& 0.2287& 0.1093& -0.1093
\\
Brier-log Hybrid & 0.6439& 0.6098& 0.4128& 0.2311& 0.1081& -0.1080
\\
\midrule

\end{tabular}

\end{table}

\begin{table}[h!]
\centering
\caption{Accuracy, AUROC, Brier-1, Brier, ECE (10 bins) and calibration bias statistics for the test set of HotpotQA-Modified when Qwen 2.5 (3B) Instruct is trained with HotpotQA-Modified dataset. Calib. Bias is short for Calibration bias.}
\label{table:hotpotqa-modified-qwen2.5-3b-overall-results}
\begin{tabular}{l c c c c c c} 
\toprule
Model/Reward Scheme& Accuracy & AUROC & Brier-1 & Brier & ECE & Calib. Bias\\
\midrule
Base (JSON) & 0.2806& 0.6559& -0.2226& 0.5032& 0.5335& -0.5317
\\
Base (XML) & 0.0363& 0.6050& -0.6437& 0.6800& 0.7703& -0.7703
\\
SFT (XML) & 0.2787& 0.6472& -0.2270& 0.5057& 0.5347& -0.5325
\\
\midrule
Log Loss & 0.0000& -& 0.0000& 0.0001& 0.0050& -0.0050
\\
Brier Score & 0.0000& -& 0.0000& 0.0001& 0.0050& -0.0050
\\
Log-1 & 0.0950& 0.9681& 0.0536& 0.0415& 0.0235& -0.0235
\\
Brier-2 & 0.2862& 0.8900& 0.1677& 0.1185& 0.0414& -0.0414
\\
Log-\(\frac{1}{\ln 202}\)& 0.4186& 0.7081& 0.1922& 0.2264& 0.1038& -0.1026
\\
\midrule
Correctness-only & 0.4204& 0.5292& -0.1282& 0.5486& 0.5508& -0.5470
\\
Overconfidence-1000 & 0.4135& 0.6856& 0.1402& 0.2734& 0.2300& -0.2290
\\
Overconfidence-4 & 0.4166& 0.7405& 0.2149& 0.2016& 0.0223& -0.0129
\\
Overconfidence-1 & 0.4204& 0.7332& 0.2012& 0.2192& 0.1182& -0.1152
\\
Brier-1 & 0.4173& 0.7474& 0.2131& 0.2043& 0.0637& -0.0608
\\
Underconfidence-1 & 0.4181& 0.7519& 0.2140& 0.2041& 0.0794& -0.0775
\\
Underconfidence-4 & 0.4118& 0.7481& 0.2088& 0.2030& 0.0616& -0.0565
\\
Brier-log Hybrid & 0.4007& 0.7440& 0.1994& 0.2013& 0.0503& -0.0459
\\
\midrule

\end{tabular}

\end{table}

\begin{table}[h!]
\centering
\caption{Accuracy, AUROC, Brier-1, Brier, ECE (10 bins) and calibration bias statistics for the easy subset of the test set of HotpotQA-Modified when Qwen 2.5 (3B) Instruct is  trained with HotpotQA-Modified dataset. Calib. Bias is short for Calibration bias.}
\label{table:hotpotqa-modified-qwen2.5-3b-easy-results}
\begin{tabular}{l c c c c c c} 
\toprule
Model/Reward Scheme& Accuracy & AUROC & Brier-1 & Brier & ECE & Calib. Bias\\
\midrule
Base (JSON) & 0.4493& 0.5908& -0.0125& 0.4618& 0.4672& -0.4642
\\
Base (XML) & 0.0588& 0.5386& -0.7216& 0.7804& 0.8369& -0.8368
\\
SFT (XML) & 0.4421& 0.5898& -0.0213& 0.4634& 0.4680& -0.4655
\\
\midrule
Log Loss & 0.0001& -& 0.0000& 0.0001& 0.0050& -0.0049
\\
Brier Score & 0.0000& -& 0.0000& 0.0001& 0.0050& -0.0050
\\
Log-1 & 0.2022& 0.9386& 0.1275& 0.0747& 0.0271& -0.0271
\\
Brier-2 & 0.5296& 0.7723& 0.3546& 0.1750& 0.0162& -0.0162
\\
Log-\(\frac{1}{\ln 202}\)& 0.6271& 0.6098& 0.3953& 0.2318& 0.0725& -0.0320
\\
\midrule
Correctness-only & 0.6273& 0.5059& 0.2612& 0.3661& 0.3648& -0.3634
\\
Overconfidence-1000 & 0.6228& 0.5957& 0.3723& 0.2506& 0.1348& -0.1299
\\
Overconfidence-4 & 0.6248& 0.6326& 0.3958& 0.2290& 0.0698& 0.0697
\\
Overconfidence-1 & 0.6328& 0.6198& 0.4081& 0.2247& 0.0506& -0.0310
\\
Brier-1 & 0.6317& 0.6359& 0.4082& 0.2235& 0.0419& 0.0025
\\
Underconfidence-1 & 0.6245& 0.6511& 0.4033& 0.2212& 0.0414& -0.0083
\\
Underconfidence-4 & 0.6199& 0.6488& 0.3972& 0.2226& 0.0433& 0.0044
\\
Brier-log Hybrid & 0.6111& 0.6461& 0.3864& 0.2247& 0.0371& 0.0210
\\
\midrule

\end{tabular}

\end{table}

\begin{table}[h!]
\centering
\caption{Accuracy, AUROC, Brier-1, Brier, ECE (10 bins) and calibration bias statistics for the medium subset of the test set of HotpotQA-Modified when Qwen 2.5 (3B) Instruct is trained with HotpotQA-Modified dataset. Calib. Bias is short for Calibration bias.}
\label{table:hotpotqa-modified-qwen2.5-3b-medium-results}
\begin{tabular}{l c c c c c c} 
\toprule
Model/Reward Scheme& Accuracy & AUROC & Brier-1 & Brier & ECE & Calib. Bias\\
\midrule
Base (JSON) & 0.2726& 0.6221& -0.2435& 0.5161& 0.5411& -0.5382
\\
Base (XML) & 0.0352& 0.5796& -0.6395& 0.6746& 0.7681& -0.7681
\\
SFT (XML) & 0.2726& 0.6152& -0.2452& 0.5178& 0.5407& -0.5381
\\
\midrule
Log Loss & 0.0000& -& 0.0000& 0.0000& 0.0050& -0.0050
\\
Brier Score & 0.0001& -& 0.0000& 0.0001& 0.0050& -0.0049
\\
Log-1 & 0.0619& 0.9708& 0.0272& 0.0347& 0.0268& -0.0268
\\
Brier-2 & 0.2323& 0.8839& 0.1149& 0.1174& 0.0583& -0.0583
\\
Log-\(\frac{1}{\ln 202}\)& 0.4203& 0.6393& 0.1704& 0.2499& 0.1243& -0.0881
\\
\midrule
Correctness-only & 0.4250& 0.5209& -0.1237& 0.5487& 0.5498& -0.5447
\\
Overconfidence-1000 & 0.4135& 0.6110& 0.1207& 0.2929& 0.2229& -0.2217
\\
Overconfidence-4 & 0.4187& 0.6659& 0.1923& 0.2264& 0.0541& 0.0033
\\
Overconfidence-1 & 0.4238& 0.6575& 0.1805& 0.2433& 0.1206& -0.1053
\\
Brier-1 & 0.4184& 0.6743& 0.1875& 0.2309& 0.0896& -0.0424
\\
Underconfidence-1 & 0.4197& 0.6762& 0.1902& 0.2295& 0.0777& -0.0615
\\
Underconfidence-4 & 0.4158& 0.6766& 0.1869& 0.2289& 0.0816& -0.0363
\\
Brier-log Hybrid & 0.4003& 0.6699& 0.1747& 0.2257& 0.0691& -0.0296
\\
\midrule

\end{tabular}

\end{table}

\begin{table}[h!]
\centering
\caption{Accuracy, AUROC, Brier-1, Brier, ECE (10 bins) and calibration bias statistics for the hard subset of the test set of HotpotQA-Modified when Qwen 2.5 (3B) Instruct is trained with HotpotQA-Modified dataset. Calib. Bias is short for Calibration bias.}
\label{table:hotpotqa-modified-qwen2.5-3b-hard-results}
\begin{tabular}{l c c c c c c} 
\toprule
Model/Reward Scheme& Accuracy & AUROC & Brier-1 & Brier & ECE & Calib. Bias\\
\midrule
Base (JSON) & 0.1203& 0.6546& -0.4118& 0.5320& 0.5926& -0.5926
\\
Base (XML) & 0.0149& 0.6180& -0.5702& 0.5851& 0.7062& -0.7062
\\
SFT (XML) & 0.1217& 0.6452& -0.4144& 0.5361& 0.5953& -0.5938
\\
\midrule
Log Loss & 0.0000& -& 0.0000& 0.0000& 0.0050& -0.0050
\\
Brier Score & 0.0000& -& 0.0000& 0.0001& 0.0050& -0.0050
\\
Log-1 & 0.0204& 0.9829& 0.0054& 0.0149& 0.0166& -0.0166
\\
Brier-2 & 0.0958& 0.9398& 0.0325& 0.0633& 0.0503& -0.0503
\\
Log-\(\frac{1}{\ln 202}\)& 0.2093& 0.7020& 0.0109& 0.1984& 0.1896& -0.1870
\\
\midrule
Correctness-only & 0.2100& 0.5276& -0.5202& 0.7302& 0.7370& -0.7321
\\
Overconfidence-1000 & 0.2050& 0.6807& -0.0722& 0.2772& 0.3357& -0.3346
\\
Overconfidence-4 & 0.2070& 0.7824& 0.0567& 0.1504& 0.1160& -0.1110
\\
Overconfidence-1 & 0.2055& 0.7701& 0.0152& 0.1904& 0.2106& -0.2085
\\
Brier-1 & 0.2028& 0.7788& 0.0434& 0.1594& 0.1448& -0.1416
\\
Underconfidence-1 & 0.2110& 0.7925& 0.0486& 0.1624& 0.1638& -0.1618
\\
Underconfidence-4 & 0.2008& 0.7764& 0.0424& 0.1584& 0.1413& -0.1366
\\
Brier-log Hybrid & 0.1915& 0.7612& 0.0371& 0.1544& 0.1326& -0.1283
\\
\midrule

\end{tabular}

\end{table}

\clearpage

\subsection{Qwen 3 (4B)}\label{appendix:qwen3-4b-results}

Similar to Appendix \ref{appendix:qwen2.5-3b-results}, this subsection provides the accuracy, AUROC, Brier-1, Brier score, ECE (10-bins) and calibration bias of all datasets used in the experiment for Qwen 3 (4B). In addition, this subsection provides the abstention statistics, as judged by GPT-OSS-20B and Qwen 3 (30B-A3B) FP8. All statistics are broken down by question difficulty, i.e. easy, medium and hard, and overall statistics that combine all difficulty levels are also provided.

Tables \ref{table:hotpotqa-modified-qwen3-4b-overall-results}, \ref{table:hotpotqa-modified-qwen3-4b-easy-results},  \ref{table:hotpotqa-modified-qwen3-4b-medium-results} and \ref{table:hotpotqa-modified-qwen3-4b-hard-results} respectively show the performance statistics of Qwen 3 (4B) when trained on HotpotQA-Modified and tested on the entire test set, easy subset, medium subset and hard subset, of the test set of HotpotQA-Modified.

Tables \ref{table:deepmath-103k-qwen3-4b-overall-results}, \ref{table:deepmath-103k-qwen3-4b-easy-results},  \ref{table:deepmath-103k-qwen3-4b-medium-results} and \ref{table:deepmath-103k-qwen3-4b-hard-results} respectively show the performance statistics of Qwen 3 (4B) when trained on DeepMath-103K and tested on the entire test set, easy subset, medium subset and hard subset, of the test set of DeepMath-103K.

Tables \ref{table:hotpotqa-modified-qwen3-4b-abstention-gpt-oss-20b} and \ref{table:hotpotqa-modified-qwen3-4b-abstention-qwen3-30b-a3b-fp8} show the abstention rates for Qwen 3 (4B) when trained and evaluated on HotpotQA-Modified when judged by GPT-OSS-20B and Qwen 3 (30B-A3B) respectively, including a breakdown by difficulty.

Tables \ref{table:deepmath-103k-qwen3-4b-abstention-gpt-oss-20b} and \ref{table:deepmath-103k-qwen3-4b-abstention-qwen3-30b-a3b-fp8} show the abstention rates for Qwen 3 (4B) when trained and evaluated on DeepMath-103K when judged by GPT-OSS-20B and Qwen 3 (30B-A3B) respectively, including a breakdown by difficulty.

\begin{table}[h]
\centering

\caption{Accuracy, AUROC, Brier-1, Brier, ECE (10 bins) and calibration bias statistics for the test set of HotpotQA-Modified when Qwen 3 (4B) is trained with HotpotQA-Modified dataset. Calib. Bias is short for Calibration bias.}
\label{table:hotpotqa-modified-qwen3-4b-overall-results}

\begin{tabular}{l c c c c c c} 
\toprule
Model/Reward Scheme & Accuracy & AUROC & Brier-1 & Brier & ECE & Calib. Bias \\ \midrule
Base (JSON) & 0.3145 & 0.7606 & -0.0649 & 0.3794 & 0.4345 & -0.4345 \\ \midrule
Log Loss & 0.0006 & - & 0.0004 & 0.0003 & 0.0053 & -0.0052 \\
Brier Score & 0.0100 & 0.9829 & 0.0064 & 0.0036 & 0.0071 & -0.0056 \\
Log-1 & 0.3326 & 0.8812 & 0.1913 & 0.1413 & 0.0904 & -0.0904 \\
Brier-2 & 0.3697 & 0.8452 & 0.1950 & 0.1747 & 0.1302 & -0.1302 \\
Log-\(\frac{1}{\ln 202}\) & 0.4323 & 0.7443 & 0.2092 & 0.2231 & 0.1472 & -0.1472 \\
\midrule
Correctness-only & 0.4374 & 0.7366 & 0.0875 & 0.3499 & 0.3768 & -0.3768 \\
Overconfidence-1000 & 0.4325 & 0.7480 & 0.1162 & 0.3163 & 0.3327 & -0.3327 \\
Overconfidence-4 & 0.4319 & 0.7521 & 0.1829 & 0.2490 & 0.2103 & -0.2103 \\
Overconfidence-1 & 0.4202 & 0.7773 & 0.2092 & 0.2110 & 0.1398 & -0.1398 \\
Brier-1 & 0.4166 & 0.7842 & 0.2059 & 0.2107 & 0.1517 & -0.1517 \\
Underconfidence-1 & 0.4192 & 0.7864 & 0.2125 & 0.2067 & 0.1398 & -0.1398 \\
Underconfidence-4 & 0.4174 & 0.7889 & 0.2061 & 0.2113 & 0.1634 & -0.1634 \\
Brier-log Hybrid & 0.4103 & 0.7872 & 0.1952 & 0.2151 & 0.1773 & -0.1773 \\
\midrule
\end{tabular}

\end{table}

\begin{table}[h]
\centering
\caption{Accuracy, AUROC, Brier-1, Brier, ECE (10 bins) and calibration bias statistics for the easy subset of the test set of HotpotQA-Modified when Qwen 3 (4B) is trained with HotpotQA-Modified dataset. Calib. Bias is short for Calibration bias.}
\label{table:hotpotqa-modified-qwen3-4b-easy-results}
\begin{tabular}{l c c c c c c} 
\toprule
Model/Reward Scheme & Accuracy & AUROC & Brier-1 & Brier & ECE & Calib. Bias \\
\midrule
Base (JSON) & 0.5394 & 0.5930 & 0.1355 & 0.4039 & 0.4074 & -0.4074 \\
\midrule
Log Loss & 0.0018 & 0.9998 & 0.0012 & 0.0006 & 0.0056 & -0.0053 \\
Brier Score & 0.0250 & 0.9806 & 0.0169 & 0.0081 & 0.0089 & -0.0053 \\
Log-1 & 0.6279 & 0.6988 & 0.4182 & 0.2097 & 0.1152 & -0.1143 \\
Brier-2 & 0.6486 & 0.6599 & 0.4217 & 0.2268 & 0.1422 & -0.1400 \\
Log-\(\frac{1}{\ln 202}\) & 0.6611 & 0.6208 & 0.4384 & 0.2226 & 0.0863 & -0.0863 \\
\midrule
Correctness-only & 0.6657 & 0.6086 & 0.3771 & 0.2886 & 0.2740 & -0.2740 \\
Overconfidence-1000 & 0.6682 & 0.6026 & 0.3823 & 0.2859 & 0.2710 & -0.2710 \\
Overconfidence-4 & 0.6673 & 0.6259 & 0.4164 & 0.2509 & 0.1893 & -0.1875 \\
Overconfidence-1 & 0.6648 & 0.6379 & 0.4364 & 0.2283 & 0.1243 & -0.1193 \\
Brier-1 & 0.6638 & 0.6364 & 0.4359 & 0.2279 & 0.1238 & -0.1227 \\
Underconfidence-1 & 0.6654 & 0.6481 & 0.4411 & 0.2243 & 0.1229 & -0.1187 \\
Underconfidence-4 & 0.6644 & 0.6458 & 0.4383 & 0.2261 & 0.1318 & -0.1300 \\
Brier-log Hybrid & 0.6594 & 0.6327 & 0.4312 & 0.2283 & 0.1264 & -0.1260 \\
\midrule
\end{tabular}

\end{table}

\begin{table}[h]
\centering

\caption{Accuracy, AUROC, Brier-1, Brier, ECE (10 bins) and calibration bias statistics for the medium subset of the test set of HotpotQA-Modified when Qwen 3 (4B) is trained with HotpotQA-Modified dataset. Calib. Bias is short for Calibration bias.}
\label{table:hotpotqa-modified-qwen3-4b-medium-results}

\begin{tabular}{l c c c c c c} 
\toprule
Model/Reward Scheme & Accuracy & AUROC & Brier-1 & Brier & ECE & Calib. Bias \\
\midrule
Base (JSON) & 0.2905 & 0.7091 & -0.1030 & 0.3935 & 0.4416 & -0.4416 \\
\midrule
Log Loss & 0.0000 & - & -0.0001 & 0.0002 & 0.0051 & -0.0051 \\
Brier Score & 0.0035 & 0.9879 & 0.0017 & 0.0018 & 0.0064 & -0.0058 \\
Log-1 & 0.2634 & 0.8756 & 0.1250 & 0.1383 & 0.0887 & -0.0873 \\
Brier-2 & 0.3217 & 0.8187 & 0.1399 & 0.1818 & 0.1244 & -0.1244 \\
Log-\(\frac{1}{\ln 202}\) & 0.4280 & 0.6494 & 0.1784 & 0.2496 & 0.1242 & -0.1242 \\
\midrule
Correctness-only & 0.4343 & 0.6467 & 0.0678 & 0.3664 & 0.3668 & -0.3668 \\
Overconfidence-1000 & 0.4233 & 0.6629 & 0.0835 & 0.3397 & 0.3245 & -0.3245 \\
Overconfidence-4 & 0.4247 & 0.6613 & 0.1512 & 0.2736 & 0.1837 & -0.1837 \\
Overconfidence-1 & 0.4051 & 0.7030 & 0.1710 & 0.2341 & 0.1177 & -0.1170 \\
Brier-1 & 0.3976 & 0.7186 & 0.1666 & 0.2310 & 0.1385 & -0.1385 \\
Underconfidence-1 & 0.4015 & 0.7203 & 0.1749 & 0.2266 & 0.1193 & -0.1193 \\
Underconfidence-4 & 0.4017 & 0.7177 & 0.1696 & 0.2321 & 0.1468 & -0.1468 \\
Brier-log Hybrid & 0.3907 & 0.7224 & 0.1563 & 0.2344 & 0.1680 & -0.1680 \\
\midrule
\end{tabular}

\end{table}

\begin{table}[h]
\centering
\caption{Accuracy, AUROC, Brier-1, Brier, ECE (10 bins) and calibration bias statistics for the hard subset of the test set of HotpotQA-Modified when Qwen 3 (4B) is trained with HotpotQA-Modified dataset. Calib. Bias is short for Calibration bias.}
\label{table:hotpotqa-modified-qwen3-4b-hard-results}
\begin{tabular}{l c c c c c c} 
\toprule
Model/Reward Scheme & Accuracy & AUROC & Brier-1 & Brier & ECE & Calib. Bias \\
\midrule
Base (JSON) & 0.1137 & 0.8009 & -0.2276 & 0.3413 & 0.4546 & -0.4546 \\
\midrule
Log Loss & 0.0001 & - & 0.0000 & 0.0001 & 0.0051 & -0.0050 \\
Brier Score & 0.0013 & 0.9793 & 0.0003 & 0.0009 & 0.0061 & -0.0057 \\
Log-1 & 0.1053 & 0.9354 & 0.0294 & 0.0759 & 0.0696 & -0.0696 \\
Brier-2 & 0.1384 & 0.8937 & 0.0225 & 0.1159 & 0.1260 & -0.1260 \\
Log-\(\frac{1}{\ln 202}\) & 0.2085 & 0.7645 & 0.0106 & 0.1979 & 0.2302 & -0.2302 \\
\midrule
Correctness-only & 0.2132 & 0.7339 & -0.1819 & 0.3951 & 0.4888 & -0.4888 \\
Overconfidence-1000 & 0.2066 & 0.7550 & -0.1173 & 0.3239 & 0.4022 & -0.4022 \\
Overconfidence-4 & 0.2043 & 0.7585 & -0.0191 & 0.2234 & 0.2585 & -0.2585 \\
Overconfidence-1 & 0.1913 & 0.7959 & 0.0197 & 0.1715 & 0.1823 & -0.1823 \\
Brier-1 & 0.1887 & 0.8104 & 0.0148 & 0.1739 & 0.1934 & -0.1934 \\
Underconfidence-1 & 0.1909 & 0.8056 & 0.0211 & 0.1698 & 0.1807 & -0.1807 \\
Underconfidence-4 & 0.1865 & 0.8127 & 0.0102 & 0.1764 & 0.2127 & -0.2127 \\
Brier-log Hybrid & 0.1813 & 0.8158 & -0.0021 & 0.1834 & 0.2373 & -0.2373 \\
\midrule
\end{tabular}

\end{table}

\begin{table}[h]
\centering
\caption{Accuracy, AUROC, Brier-1, Brier, ECE (10 bins) and calibration bias statistics for the test set of DeepMath-103K when Qwen 3 (4B) is trained with DeepMath-103K dataset. Calib. Bias is short for Calibration bias.}
\label{table:deepmath-103k-qwen3-4b-overall-results}
\begin{tabular}{l c c c c c c} 
\toprule
Model/Reward Scheme & Accuracy & AUROC & Brier-1 & Brier & ECE & Calib. Bias \\
\midrule
Base (JSON) & 0.3265 & 0.6754 & -0.1641 & 0.4907 & 0.5332 & -0.5325 \\
\midrule
Log Loss & 0.0006 & - & 0.0001 & 0.0005 & 0.0050 & -0.0049 \\
Brier Score & 0.0060 & 0.9366 & 0.0019 & 0.0041 & 0.0070 & -0.0067 \\
Log-1 & 0.7868 & 0.7324 & 0.6404 & 0.1464 & 0.0370 & -0.0368 \\
Brier-2 & 0.6913 & 0.6631 & 0.4884 & 0.2029 & 0.0909 & -0.0903 \\
Log-\(\frac{1}{\ln 202}\) & 0.7771 & 0.7077 & 0.6218 & 0.1553 & 0.0460 & -0.0460 \\
\midrule
Correctness-only & 0.4528 & 0.7240 & 0.0163 & 0.4364 & 0.4544 & -0.4543 \\
Overconfidence-1000 & 0.7143 & 0.6582 & 0.5153 & 0.1990 & 0.0966 & -0.0966 \\
Overconfidence-4 & 0.6106 & 0.5808 & 0.3399 & 0.2707 & 0.1949 & -0.1863 \\
Overconfidence-1 & 0.5127 & 0.5618 & 0.1858 & 0.3269 & 0.2796 & -0.2749 \\
Brier-1 & 0.7046 & 0.6429 & 0.5089 & 0.1956 & 0.0667 & -0.0643 \\
Underconfidence-1 & 0.4465 & 0.5512 & 0.0764 & 0.3702 & 0.3486 & -0.3472 \\
Underconfidence-4 & 0.7189 & 0.6912 & 0.5354 & 0.1834 & 0.0729 & -0.0729 \\
Brier-log Hybrid & 0.7416 & 0.7262 & 0.5702 & 0.1714 & 0.0613 & -0.0613 \\
\midrule
\end{tabular}

\end{table}

\begin{table}[h]
\centering
\caption{Accuracy, AUROC, Brier-1, Brier, ECE (10 bins) and calibration bias statistics for the easy subset of the test set of DeepMath-103K when Qwen 3 (4B) is trained with DeepMath-103K dataset. Calib. Bias is short for Calibration bias.}
\label{table:deepmath-103k-qwen3-4b-easy-results}
\begin{tabular}{l c c c c c c} 
\toprule
Model/Reward Scheme & Accuracy & AUROC & Brier-1 & Brier & ECE & Calib. Bias \\
\midrule
Base (JSON) & 0.4402 & 0.7433 & 0.0100 & 0.4302 & 0.4582 & -0.4582 \\
\midrule
Log Loss & 0.0015 & 1.0000 & 0.0014 & 0.0001 & 0.0052 & -0.0047 \\
Brier Score & 0.0174 & 0.9980 & 0.0128 & 0.0046 & 0.0093 & -0.0050 \\
Log-1 & 0.8561 & 0.8028 & 0.7540 & 0.1020 & 0.0263 & -0.0128 \\
Brier-2 & 0.7758 & 0.7529 & 0.6261 & 0.1496 & 0.0513 & -0.0481 \\
Log-\(\frac{1}{\ln 202}\) & 0.8508 & 0.7610 & 0.7403 & 0.1104 & 0.0284 & -0.0161 \\
\midrule
Correctness-only & 0.5614 & 0.7956 & 0.2063 & 0.3550 & 0.3673 & -0.3673 \\
Overconfidence-1000 & 0.7985 & 0.7321 & 0.6544 & 0.1441 & 0.0468 & -0.0468 \\
Overconfidence-4 & 0.6848 & 0.6912 & 0.4624 & 0.2224 & 0.1599 & -0.1501 \\
Overconfidence-1 & 0.6144 & 0.6698 & 0.3491 & 0.2653 & 0.2135 & -0.2091 \\
Brier-1 & 0.7636 & 0.7377 & 0.6076 & 0.1561 & 0.0714 & -0.0561 \\
Underconfidence-1 & 0.5386 & 0.6486 & 0.2224 & 0.3162 & 0.2868 & -0.2868 \\
Underconfidence-4 & 0.8091 & 0.7532 & 0.6796 & 0.1294 & 0.0287 & -0.0287 \\
Brier-log Hybrid & 0.8159 & 0.7882 & 0.6905 & 0.1254 & 0.0409 & -0.0260 \\
\midrule
\end{tabular}

\end{table}

\begin{table}[h]
\centering
\caption{Accuracy, AUROC, Brier-1, Brier, ECE (10 bins) and calibration bias statistics for the medium subset of the test set of DeepMath-103K when Qwen 3 (4B) is trained with DeepMath-103K dataset. Calib. Bias is short for Calibration bias.}
\label{table:deepmath-103k-qwen3-4b-medium-results}
\begin{tabular}{l c c c c c c} 
\toprule
Model/Reward Scheme & Accuracy & AUROC & Brier-1 & Brier & ECE & Calib. Bias \\
\midrule
Base (JSON) & 0.3001 & 0.6690 & -0.2112 & 0.5113 & 0.5589 & -0.5585 \\
Log Loss & 0.0002 & - & 0.0000 & 0.0003 & 0.0047 & -0.0047 \\
Brier Score & 0.0028 & 0.7896 & -0.0008 & 0.0036 & 0.0061 & -0.0061 \\
Log-1 & 0.8020 & 0.7364 & 0.6649 & 0.1371 & 0.0306 & -0.0303 \\
Brier-2 & 0.7026 & 0.6439 & 0.5016 & 0.2010 & 0.0863 & -0.0842 \\
Log-\(\frac{1}{\ln 202}\) & 0.7907 & 0.6983 & 0.6419 & 0.1487 & 0.0393 & -0.0393 \\
\midrule
Correctness-only & 0.4279 & 0.7287 & -0.0297 & 0.4576 & 0.4798 & -0.4795 \\
Overconfidence-1000 & 0.7278 & 0.6389 & 0.5336 & 0.1943 & 0.0856 & -0.0850 \\
Overconfidence-4 & 0.6183 & 0.5706 & 0.3533 & 0.2650 & 0.1842 & -0.1770 \\
Overconfidence-1 & 0.5032 & 0.5306 & 0.1686 & 0.3347 & 0.2856 & -0.2813 \\
Brier-1 & 0.7098 & 0.6177 & 0.5128 & 0.1970 & 0.0718 & -0.0643 \\
Underconfidence-1 & 0.4311 & 0.5298 & 0.0487 & 0.3824 & 0.3636 & -0.3614 \\
Underconfidence-4 & 0.7331 & 0.6871 & 0.5570 & 0.1762 & 0.0630 & -0.0630 \\
Brier-log Hybrid & 0.7539 & 0.7178 & 0.5873 & 0.1667 & 0.0529 & -0.0529 \\
\midrule
\end{tabular}

\end{table}

\begin{table}[h]
\centering
\caption{Accuracy, AUROC, Brier-1, Brier, ECE (10 bins) and calibration bias statistics for the hard subset of the test set of DeepMath-103K when Qwen 3 (4B) is trained with DeepMath-103K dataset. Calib. Bias is short for Calibration bias.}
\label{table:deepmath-103k-qwen3-4b-hard-results}
\begin{tabular}{l c c c c c c} 
\toprule
Model/Reward Scheme & Accuracy & AUROC & Brier-1 & Brier & ECE & Calib. Bias \\
\midrule
Base (JSON) & 0.3124 & 0.6294 & -0.1746 & 0.4869 & 0.5282 & -0.5266 \\
\midrule
Log Loss & 0.0008 & - & -0.0002 & 0.0010 & 0.0054 & -0.0054 \\
Brier Score & 0.0055 & 0.9612 & 0.0008 & 0.0047 & 0.0088 & -0.0086 \\
Log-1 & 0.7245 & 0.6706 & 0.5388 & 0.1856 & 0.0605 & -0.0605 \\
Brier-2 & 0.6279 & 0.6274 & 0.3938 & 0.2341 & 0.1236 & -0.1228 \\
Log-\(\frac{1}{\ln 202}\) & 0.7154 & 0.6718 & 0.5253 & 0.1901 & 0.0743 & -0.0732 \\
\midrule
Correctness-only & 0.4386 & 0.6664 & -0.0040 & 0.4427 & 0.4571 & -0.4565 \\
Overconfidence-1000 & 0.6473 & 0.6291 & 0.4114 & 0.2358 & 0.1426 & -0.1426 \\
Overconfidence-4 & 0.5586 & 0.5238 & 0.2529 & 0.3057 & 0.2321 & -0.2211 \\
Overconfidence-1 & 0.4759 & 0.5392 & 0.1300 & 0.3459 & 0.3038 & -0.2984 \\
Brier-1 & 0.6647 & 0.6242 & 0.4507 & 0.2140 & 0.0695 & -0.0687 \\
Underconfidence-1 & 0.4248 & 0.5193 & 0.0474 & 0.3774 & 0.3552 & -0.3544 \\
Underconfidence-4 & 0.6473 & 0.6423 & 0.4232 & 0.2240 & 0.1146 & -0.1130 \\
Brier-log Hybrid & 0.6817 & 0.6919 & 0.4781 & 0.2036 & 0.0941 & -0.0941 \\
\midrule

\end{tabular}

\end{table}

\begin{table}[h]
\centering
\caption{Qwen 3 (4B) abstention statistics when trained with HotpotQA-Modified dataset, with GPT-OSS-20B as the LLM judge for determining abstention. Unans. stands for unanswerability, Other refers to LLM abstained without implying unanswerability, while Total refers to total abstention rate of the LLM.}
\label{table:hotpotqa-modified-qwen3-4b-abstention-gpt-oss-20b}
\resizebox{\textwidth}{!}{
\begin{tabular}{l l l l l l l l l l l l l}
\toprule
\multirow{2.4}{*}{Model/Reward Scheme} & \multicolumn{3}{c}{Easy} & \multicolumn{3}{c}{Medium} & \multicolumn{3}{c}{Hard} & \multicolumn{3}{c}{Overall} \\
\cmidrule(lr){2-4} \cmidrule(lr){5-7} \cmidrule(lr){8-10} \cmidrule(lr){11-13}
 & Unans. & Other & Total & Unans. & Other & Total & Unans. & Other & Total & Unans. & Other & Total \\

\midrule
Base (JSON) & 0.0066 & 0.0012 & 0.0077 & 0.1109 & 0.0130 & 0.1239 & 0.2474 & 0.0188 & 0.2661 & 0.1219 & 0.0110 & 0.1329 \\
\midrule
Log Loss & 0.8399 & 0.1085 & 0.9484 & 0.8449 & 0.1074 & 0.9522 & 0.8455 & 0.1091 & 0.9546 & 0.8434 & 0.1083 & 0.9518 \\
Brier Score & 0.8140 & 0.1142 & 0.9281 & 0.8424 & 0.1098 & 0.9522 & 0.8447 & 0.1099 & 0.9546 & 0.8336 & 0.1113 & 0.9449 \\
Log-1 & 0.0597 & 0.0086 & 0.0683 & 0.3963 & 0.0584 & 0.4547 & 0.5916 & 0.0725 & 0.6641 & 0.3491 & 0.0464 & 0.3955 \\
Brier-2 & 0.0283 & 0.0042 & 0.0325 & 0.2867 & 0.0322 & 0.3188 & 0.4844 & 0.0413 & 0.5257 & 0.2666 & 0.0258 & 0.2924 \\
Log-\(\frac{1}{\ln 202}\) & 0.0011 & 0.0014 & 0.0025 & 0.0240 & 0.0082 & 0.0322 & 0.0693 & 0.0180 & 0.0872 & 0.0316 & 0.0092 & 0.0408 \\
\midrule
Correctness-only & 0.0011 & 0.0006 & 0.0017 & 0.0163 & 0.0041 & 0.0205 & 0.0538 & 0.0062 & 0.0600 & 0.0239 & 0.0036 & 0.0275 \\
Overconfidence-1000 & 0.0021 & 0.0011 & 0.0031 & 0.0366 & 0.0073 & 0.0439 & 0.0984 & 0.0128 & 0.1112 & 0.0459 & 0.0071 & 0.0529 \\
Overconfidence-4 & 0.0016 & 0.0012 & 0.0028 & 0.0374 & 0.0095 & 0.0469 & 0.1050 & 0.0184 & 0.1235 & 0.0482 & 0.0097 & 0.0579 \\
Overconfidence-1 & 0.0046 & 0.0018 & 0.0064 & 0.0868 & 0.0146 & 0.1014 & 0.1969 & 0.0244 & 0.2213 & 0.0963 & 0.0136 & 0.1099 \\
Brier-1 & 0.0054 & 0.0022 & 0.0076 & 0.1002 & 0.0169 & 0.1171 & 0.2224 & 0.0282 & 0.2505 & 0.1096 & 0.0158 & 0.1253 \\
Underconfidence-1 & 0.0059 & 0.0024 & 0.0083 & 0.0973 & 0.0155 & 0.1128 & 0.2169 & 0.0270 & 0.2440 & 0.1070 & 0.0150 & 0.1219 \\
Underconfidence-4 & 0.0055 & 0.0020 & 0.0075 & 0.0919 & 0.0161 & 0.1080 & 0.2109 & 0.0254 & 0.2363 & 0.1030 & 0.0145 & 0.1175 \\
Brier-log Hybrid & 0.0060 & 0.0020 & 0.0080 & 0.1112 & 0.0134 & 0.1246 & 0.2402 & 0.0201 & 0.2603 & 0.1194 & 0.0118 & 0.1312 \\
\midrule
\end{tabular}
}
\end{table}

\begin{table}[h]
\centering

\caption{Qwen 3 (4B) abstention statistics when trained with HotpotQA-Modified dataset, with Qwen 3 (30B-A3B) FP8 as the LLM judge for determining abstention. Unans. stands for unanswerability, Other refers to LLM abstained without implying unanswerability, while Total refers to total abstention rate of the LLM.}
\label{table:hotpotqa-modified-qwen3-4b-abstention-qwen3-30b-a3b-fp8}
\resizebox{\textwidth}{!}{
\begin{tabular}{l l l l l l l l l l l l l}
\toprule
\multirow{2.4}{*}{Model/Reward Scheme} & \multicolumn{3}{c}{Easy} & \multicolumn{3}{c}{Medium} & \multicolumn{3}{c}{Hard} & \multicolumn{3}{c}{Overall} \\
\cmidrule(lr){2-4} \cmidrule(lr){5-7} \cmidrule(lr){8-10} \cmidrule(lr){11-13}
 & Unans. & Other & Total & Unans. & Other & Total & Unans. & Other & Total & Unans. & Other & Total \\

\midrule
Base (JSON) & 0.0106 & 0.0000 & 0.0106 & 0.1658 & 0.0002 & 0.1659 & 0.3419 & 0.0002 & 0.3421 & 0.1731 & 0.0001 & 0.1732 \\
\midrule
Log Loss & 0.9556 & 0.0158 & 0.9714 & 0.9718 & 0.0097 & 0.9816 & 0.9730 & 0.0129 & 0.9859 & 0.9668 & 0.0128 & 0.9796 \\
Brier Score & 0.9331 & 0.0109 & 0.9440 & 0.9735 & 0.0052 & 0.9787 & 0.9817 & 0.0042 & 0.9859 & 0.9627 & 0.0068 & 0.9694 \\
Log-1 & 0.0728 & 0.0005 & 0.0733 & 0.4787 & 0.0015 & 0.4802 & 0.6997 & 0.0021 & 0.7018 & 0.4168 & 0.0014 & 0.4182 \\
Brier-2 & 0.0355 & 0.0003 & 0.0358 & 0.3422 & 0.0007 & 0.3429 & 0.5626 & 0.0010 & 0.5636 & 0.3135 & 0.0007 & 0.3141 \\
Log-\(\frac{1}{\ln 202}\) & 0.0020 & 0.0000 & 0.0021 & 0.0412 & 0.0002 & 0.0413 & 0.1115 & 0.0008 & 0.1123 & 0.0518 & 0.0003 & 0.0521 \\

\midrule
Correctness-only & 0.0017 & 0.0000 & 0.0017 & 0.0249 & 0.0001 & 0.0249 & 0.0730 & 0.0002 & 0.0732 & 0.0334 & 0.0001 & 0.0334 \\
Overconfidence-1000 & 0.0030 & 0.0000 & 0.0030 & 0.0507 & 0.0002 & 0.0509 & 0.1318 & 0.0006 & 0.1324 & 0.0620 & 0.0002 & 0.0623 \\
Overconfidence-4 & 0.0026 & 0.0000 & 0.0026 & 0.0550 & 0.0004 & 0.0554 & 0.1487 & 0.0007 & 0.1493 & 0.0690 & 0.0003 & 0.0694 \\
Overconfidence-1 & 0.0066 & 0.0001 & 0.0067 & 0.1150 & 0.0004 & 0.1154 & 0.2534 & 0.0011 & 0.2545 & 0.1253 & 0.0005 & 0.1258 \\
Brier-1 & 0.0077 & 0.0001 & 0.0079 & 0.1323 & 0.0005 & 0.1328 & 0.2813 & 0.0008 & 0.2821 & 0.1407 & 0.0005 & 0.1412 \\
Underconfidence-1 & 0.0087 & 0.0001 & 0.0088 & 0.1276 & 0.0004 & 0.1281 & 0.2748 & 0.0006 & 0.2754 & 0.1373 & 0.0004 & 0.1377 \\
Underconfidence-4 & 0.0078 & 0.0002 & 0.0080 & 0.1220 & 0.0004 & 0.1224 & 0.2705 & 0.0008 & 0.2714 & 0.1337 & 0.0005 & 0.1342 \\
Brier-log Hybrid & 0.0087 & 0.0001 & 0.0088 & 0.1404 & 0.0002 & 0.1406 & 0.2959 & 0.0008 & 0.2967 & 0.1486 & 0.0003 & 0.1490 \\
\midrule

\end{tabular}
}
\end{table}

\begin{table}[h]
\centering
\caption{Qwen 3 (4B) abstention statistics when trained with DeepMath-103K dataset, with GPT-OSS-20B as the LLM judge for determining abstention. Unans. stands for unanswerability, Other refers to LLM abstained without implying unanswerability, while Total refers to total abstention rate of the LLM.}
\label{table:deepmath-103k-qwen3-4b-abstention-gpt-oss-20b}
\resizebox{\textwidth}{!}{
\begin{tabular}{l c c c c c c c c c c c c}
\toprule
\multirow{2.4}{*}{Model/Reward Scheme} & \multicolumn{3}{c}{Easy} & \multicolumn{3}{c}{Medium} & \multicolumn{3}{c}{Hard} & \multicolumn{3}{c}{Overall} \\
\cmidrule(lr){2-4} \cmidrule(lr){5-7} \cmidrule(lr){8-10} \cmidrule(lr){11-13}
 & Unans. & Other & Total & Unans. & Other & Total & Unans. & Other & Total & Unans. & Other & Total \\
\midrule
Base (JSON) & 0.0159 & 0.0061 & 0.0220 & 0.0210 & 0.0076 & 0.0287 & 0.0329 & 0.0083 & 0.0412 & 0.0239 & 0.0076 & 0.0314 \\
Log Loss & 0.0015 & 0.9962 & 0.9977 & 0.0016 & 0.9979 & 0.9995 & 0.0024 & 0.9956 & 0.9980 & 0.0018 & 0.9969 & 0.9988 \\
Brier Score & 0.0053 & 0.9697 & 0.9750 & 0.0030 & 0.9898 & 0.9928 & 0.0032 & 0.9830 & 0.9861 & 0.0034 & 0.9845 & 0.9879 \\
Log-1 & 0.0015 & 0.0030 & 0.0045 & 0.0007 & 0.0044 & 0.0051 & 0.0020 & 0.0048 & 0.0067 & 0.0012 & 0.0043 & 0.0055 \\
Brier-2 & 0.0030 & 0.0030 & 0.0061 & 0.0012 & 0.0051 & 0.0062 & 0.0036 & 0.0044 & 0.0079 & 0.0022 & 0.0045 & 0.0067 \\
Log-\(\frac{1}{\ln 202}\) & 0.0008 & 0.0068 & 0.0076 & 0.0018 & 0.0074 & 0.0092 & 0.0024 & 0.0071 & 0.0095 & 0.0018 & 0.0072 & 0.0091 \\
\midrule

Correctness-only & 0.0015 & 0.0083 & 0.0098 & 0.0002 & 0.0136 & 0.0139 & 0.0032 & 0.0123 & 0.0154 & 0.0013 & 0.0124 & 0.0137 \\
Overconfidence-1000 & 0.0030 & 0.0053 & 0.0083 & 0.0007 & 0.0053 & 0.0060 & 0.0024 & 0.0067 & 0.0091 & 0.0016 & 0.0057 & 0.0073 \\
Overconfidence-4 & 0.0008 & 0.0144 & 0.0152 & 0.0012 & 0.0116 & 0.0127 & 0.0020 & 0.0099 & 0.0119 & 0.0013 & 0.0115 & 0.0128 \\
Overconfidence-1 & 0.0023 & 0.0220 & 0.0242 & 0.0002 & 0.0185 & 0.0187 & 0.0012 & 0.0146 & 0.0158 & 0.0009 & 0.0179 & 0.0187 \\
Brier-1 & 0.0030 & 0.0076 & 0.0106 & 0.0016 & 0.0074 & 0.0090 & 0.0044 & 0.0075 & 0.0119 & 0.0027 & 0.0075 & 0.0102 \\
Underconfidence-1 & 0.0015 & 0.0182 & 0.0197 & 0.0012 & 0.0148 & 0.0159 & 0.0012 & 0.0107 & 0.0119 & 0.0012 & 0.0141 & 0.0153 \\
Underconfidence-4 & 0.0045 & 0.0083 & 0.0129 & 0.0030 & 0.0055 & 0.0085 & 0.0087 & 0.0083 & 0.0170 & 0.0050 & 0.0069 & 0.0119 \\
Brier-log Hybrid & 0.0008 & 0.0015 & 0.0023 & 0.0000 & 0.0018 & 0.0018 & 0.0012 & 0.0024 & 0.0036 & 0.0005 & 0.0020 & 0.0024 \\
\midrule
\end{tabular}
}
\end{table}

\begin{table}[h]
\centering

\caption{Qwen 3 (4B) abstention statistics when trained with DeepMath-103K dataset, with Qwen 3 (30B-A3B) FP8 as the LLM judge for determining abstention. Unans. stands for unanswerability, Other refers to LLM abstained without implying unanswerability, while Total refers to total abstention rate of the LLM.}
\label{table:deepmath-103k-qwen3-4b-abstention-qwen3-30b-a3b-fp8}
\resizebox{\textwidth}{!}{
\begin{tabular}{l c c c c c c c c c c c c}
\toprule
\multirow{2.4}{*}{Model/Reward Scheme}& \multicolumn{3}{c}{Easy} & \multicolumn{3}{c}{Medium} & \multicolumn{3}{c}{Hard} & \multicolumn{3}{c}{Overall} \\
\cmidrule(lr){2-4} \cmidrule(lr){5-7} \cmidrule(lr){8-10} \cmidrule(lr){11-13}
 & Unans. & Other & Total & Unans. & Other & Total & Unans. & Other & Total & Unans. & Other & Total \\
\midrule
Base (JSON) & 0.0197 & 0.0030 & 0.0227 & 0.0206 & 0.0076 & 0.0282 & 0.0333 & 0.0055 & 0.0388 & 0.0243 & 0.0062 & 0.0306 \\
Log Loss & 0.0106 & 0.9803 & 0.9909 & 0.0116 & 0.9727 & 0.9843 & 0.0036 & 0.9834 & 0.9869 & 0.0089 & 0.9772 & 0.9862 \\
Brier Score & 0.0114 & 0.9568 & 0.9682 & 0.0092 & 0.9702 & 0.9794 & 0.0115 & 0.9628 & 0.9743 & 0.0103 & 0.9657 & 0.9760 \\
Log-1 & 0.0023 & 0.0038 & 0.0061 & 0.0007 & 0.0044 & 0.0051 & 0.0024 & 0.0059 & 0.0083 & 0.0015 & 0.0048 & 0.0062 \\
Brier-2 & 0.0045 & 0.0045 & 0.0091 & 0.0014 & 0.0044 & 0.0058 & 0.0048 & 0.0028 & 0.0075 & 0.0029 & 0.0039 & 0.0069 \\
Log-\(\frac{1}{\ln 202}\) & 0.0008 & 0.0061 & 0.0068 & 0.0021 & 0.0079 & 0.0099 & 0.0051 & 0.0067 & 0.0119 & 0.0028 & 0.0072 & 0.0100 \\
\midrule

Correctness-only & 0.0015 & 0.0106 & 0.0121 & 0.0014 & 0.0090 & 0.0104 & 0.0036 & 0.0103 & 0.0139 & 0.0021 & 0.0097 & 0.0117 \\
Overconfidence-1000 & 0.0038 & 0.0053 & 0.0091 & 0.0009 & 0.0065 & 0.0074 & 0.0020 & 0.0055 & 0.0075 & 0.0017 & 0.0060 & 0.0077 \\
Overconfidence-4 & 0.0038 & 0.0159 & 0.0197 & 0.0016 & 0.0092 & 0.0109 & 0.0028 & 0.0099 & 0.0127 & 0.0023 & 0.0105 & 0.0128 \\
Overconfidence-1 & 0.0045 & 0.0189 & 0.0235 & 0.0012 & 0.0159 & 0.0171 & 0.0024 & 0.0127 & 0.0150 & 0.0021 & 0.0154 & 0.0175 \\
Brier-1 & 0.0030 & 0.0061 & 0.0091 & 0.0021 & 0.0076 & 0.0097 & 0.0067 & 0.0063 & 0.0131 & 0.0037 & 0.0070 & 0.0106 \\
Underconfidence-1 & 0.0030 & 0.0152 & 0.0182 & 0.0032 & 0.0152 & 0.0185 & 0.0016 & 0.0119 & 0.0135 & 0.0027 & 0.0142 & 0.0169 \\
Underconfidence-4 & 0.0061 & 0.0098 & 0.0159 & 0.0044 & 0.0046 & 0.0090 & 0.0071 & 0.0083 & 0.0154 & 0.0055 & 0.0066 & 0.0121 \\
Brier-log Hybrid & 0.0008 & 0.0023 & 0.0030 & 0.0002 & 0.0025 & 0.0028 & 0.0024 & 0.0012 & 0.0036 & 0.0010 & 0.0021 & 0.0031 \\
\midrule
\end{tabular}
}
\end{table}

\clearpage

\subsection{Gemma 4 (E2B) Instruct}\label{appendix:gemma4-e2b-it-results}

Similar to Appendix \ref{appendix:qwen3-4b-results}, this subsection provides the accuracy, AUROC, Brier-1, Brier score, ECE (10-bins) and calibration bias of all datasets used in the experiment for Qwen 3 (4B), Likewise, this subsection provides the abstention statistics, as judged by GPT-OSS-20B and Qwen 3 (30B-A3B) FP8. All statistics are broken down by question difficulty, i.e. easy, medium and hard, and overall statistics that combine all difficulty levels are also provided.

Tables \ref{table:hotpotqa-modified-gemma4-e2b-it-overall-results}, \ref{table:hotpotqa-modified-gemma4-e2b-it-easy-results},  \ref{table:hotpotqa-modified-gemma4-e2b-it-medium-results} and \ref{table:hotpotqa-modified-gemma4-e2b-it-hard-results} respectively show the performance statistics of Gemma 4 (E2B) Instruct when trained on HotpotQA-Modified and tested on the entire test set, easy subset, medium subset and hard subset, of the test set of HotpotQA-Modified.

Tables \ref{table:hotpotqa-modified-gemma4-e2b-it-abstention-gpt-oss-20b} and \ref{table:hotpotqa-modified-gemma4-e2b-it-abstention-qwen3-30b-a3b-fp8} show the abstention rates for Gemma 4 (E2B) Instruct when trained and evaluated on HotpotQA-Modified when judged by GPT-OSS-20B and Qwen 3 (30B-A3B) respectively, including a breakdown by difficulty.

\begin{table}[h]
\centering
\caption{Accuracy, AUROC, Brier-1, Brier, ECE (10 bins) and calibration bias statistics for the test set of HotpotQA-Modified when Gemma 4 (E2B) Instruct is trained with HotpotQA-Modified dataset. Calib. Bias is short for Calibration bias.}

\label{table:hotpotqa-modified-gemma4-e2b-it-overall-results}
\begin{tabular}{l c c c c c c} 
\toprule
Model/Reward Scheme & Accuracy & AUROC & Brier-1 & Brier & ECE & Calib. Bias \\
\midrule
Base (JSON) & 0.3009 & 0.7444 & -0.1077 & 0.4085 & 0.4633 & -0.4633 \\
\midrule
Log Loss & 0.0000 & 0.0000 & 0.0000 & 0.0000 & 0.0050 & -0.0050 \\
Brier Score & 0.0000 & 0.0000 & -0.0030 & 0.0030 & 0.0545 & -0.0545 \\
Log-1 & 0.4200 & 0.7660 & 0.2225 & 0.1975 & 0.0787 & -0.0787 \\
Brier-2 & 0.2984 & 0.9038 & 0.1749 & 0.1235 & 0.1079 & -0.0885 \\
Log-\(\frac{1}{\ln 202}\) & 0.4391 & 0.7606 & 0.2297 & 0.2094 & 0.1038 & -0.1031 \\
\midrule
Correctness-only & 0.3978 & 0.5243 & -0.1843 & 0.5821 & 0.5858 & -0.5858 \\
Overconfidence-1000 & 0.4308 & 0.7359 & 0.0662 & 0.3647 & 0.3931 & -0.3931 \\
Overconfidence-4 & 0.4302 & 0.7447 & 0.2154 & 0.2149 & 0.0888 & -0.0888 \\
Overconfidence-1 & 0.4158 & 0.8129 & 0.2300 & 0.1857 & 0.1008 & -0.0888 \\
Brier-1 & 0.4189 & 0.7273 & 0.2066 & 0.2122 & 0.0735 & 0.0599 \\
Underconfidence-1 & 0.4132 & 0.8035 & 0.2242 & 0.1890 & 0.0941 & -0.0929 \\
Underconfidence-4 & 0.4352 & 0.7733 & 0.2347 & 0.2005 & 0.0850 & -0.0178 \\
Brier-log Hybrid & 0.4066 & 0.7804 & 0.2179 & 0.1887 & 0.0562 & -0.0190 \\
\midrule
\end{tabular}
\end{table}

\begin{table}
\centering
\caption{Accuracy, AUROC, Brier-1, Brier, ECE (10 bins) and calibration bias statistics for the easy subset of the test set of HotpotQA-Modified when Gemma 4 (E2B) Instruct is trained with HotpotQA-Modified dataset. Calib. Bias is short for Calibration bias.}
\label{table:hotpotqa-modified-gemma4-e2b-it-easy-results}
\begin{tabular}{l l l l l l l}
\toprule
Model/Reward Scheme & Accuracy & AUROC & Brier-1 & Brier & ECE & Calib. Bias \\
\midrule
Base (JSON) & 0.5491 & 0.5862 & 0.1454 & 0.4037 & 0.4070 & -0.4070 \\
\midrule
Log Loss & 0.0000 & 0.0000 & 0.0000 & 0.0000 & 0.0050 & -0.0050 \\
Brier Score & 0.0000 & 0.0000 & -0.0030 & 0.0030 & 0.0545 & -0.0545 \\
Log-1 & 0.6649 & 0.6131 & 0.4499 & 0.2149 & 0.0550 & -0.0450 \\
Brier-2 & 0.6081 & 0.7200 & 0.4050 & 0.2031 & 0.1457 & -0.1148 \\
Log-\(\frac{1}{\ln 202}\) & 0.6797 & 0.6193 & 0.4602 & 0.2195 & 0.1068 & -0.0880 \\
\midrule
Correctness-only & 0.6354 & 0.5197 & 0.2805 & 0.3550 & 0.3519 & -0.3519 \\
Overconfidence-1000 & 0.6557 & 0.5992 & 0.3582 & 0.2975 & 0.2800 & -0.2800 \\
Overconfidence-4 & 0.6434 & 0.6367 & 0.4049 & 0.2385 & 0.1399 & -0.0884 \\
Overconfidence-1 & 0.6841 & 0.6379 & 0.4643 & 0.2197 & 0.1227 & -0.1094 \\
Brier-1 & 0.6452 & 0.5805 & 0.3851 & 0.2601 & 0.1766 & 0.1750 \\
Underconfidence-1 & 0.6762 & 0.6125 & 0.4486 & 0.2276 & 0.1365 & -0.1122 \\
Underconfidence-4 & 0.6737 & 0.6232 & 0.4478 & 0.2259 & 0.1182 & -0.0346 \\
Brier-log Hybrid & 0.6527 & 0.5980 & 0.4374 & 0.2154 & 0.0162 & 0.0156 \\
\midrule
\end{tabular}
\end{table}

\begin{table}
\centering
\caption{Accuracy, AUROC, Brier-1, Brier, ECE (10 bins) and calibration bias statistics for the medium subset of the test set of HotpotQA-Modified when Gemma 4 (E2B) Instruct is trained with HotpotQA-Modified dataset. Calib. Bias is short for Calibration bias.}
\label{table:hotpotqa-modified-gemma4-e2b-it-medium-results}
\begin{tabular}{l l l l l l l}
\toprule
Model/Reward Scheme & Accuracy & AUROC & Brier-1 & Brier & ECE & Calib. Bias \\
\midrule
Base (JSON) & 0.2603 & 0.6993 & -0.1624 & 0.4227 & 0.4768 & -0.4768 \\
\midrule
Log Loss & 0.0000 & 0.0000 & 0.0000 & 0.0000 & 0.0050 & -0.0050 \\
Brier Score & 0.0000 & 0.0000 & -0.0030 & 0.0030 & 0.0545 & -0.0545 \\
Log-1 & 0.4148 & 0.6764 & 0.1895 & 0.2253 & 0.0587 & -0.0395 \\
Brier-2 & 0.2187 & 0.9084 & 0.1036 & 0.1151 & 0.1113 & -0.0796 \\
Log-\(\frac{1}{\ln 202}\) & 0.4361 & 0.6726 & 0.1981 & 0.2380 & 0.1003 & -0.0615 \\
\midrule
Correctness-only & 0.3894 & 0.5096 & -0.2016 & 0.5910 & 0.5943 & -0.5943 \\
Overconfidence-1000 & 0.4291 & 0.6477 & 0.0536 & 0.3755 & 0.3803 & -0.3803 \\
Overconfidence-4 & 0.4326 & 0.6764 & 0.1967 & 0.2359 & 0.0911 & -0.0414 \\
Overconfidence-1 & 0.3854 & 0.7707 & 0.1897 & 0.1957 & 0.0869 & -0.0605 \\
Brier-1 & 0.4013 & 0.6603 & 0.1743 & 0.2270 & 0.0729 & 0.0649 \\
Underconfidence-1 & 0.3931 & 0.7440 & 0.1819 & 0.2112 & 0.1022 & -0.0501 \\
Underconfidence-4 & 0.4315 & 0.6823 & 0.1893 & 0.2421 & 0.1413 & 0.0423 \\
Brier-log Hybrid & 0.3892 & 0.7333 & 0.1775 & 0.2117 & 0.1006 & 0.0034 \\
\midrule
\end{tabular}
\end{table}

\begin{table}
\centering
\caption{Accuracy, AUROC, Brier-1, Brier, ECE (10 bins) and calibration bias statistics for the hard subset of the test set of HotpotQA-Modified when Gemma 4 (E2B) Instruct is trained with HotpotQA-Modified dataset. Calib. Bias is short for Calibration bias.}
\label{table:hotpotqa-modified-gemma4-e2b-it-hard-results}
\begin{tabular}{l l l l l l l}
\toprule
Model/Reward Scheme & Accuracy & AUROC & Brier-1 & Brier & ECE & Calib. Bias \\
\midrule
Base (JSON) & 0.0929 & 0.7674 & -0.3068 & 0.3997 & 0.5062 & -0.5062 \\
\midrule
Log Loss & 0.0000 & 0.0000 & 0.0000 & 0.0000 & 0.0050 & -0.0050 \\
Brier Score & 0.0000 & 0.0000 & -0.0030 & 0.0030 & 0.0545 & -0.0545 \\
Log-1 & 0.1813 & 0.7327 & 0.0279 & 0.1534 & 0.1502 & -0.1502 \\
Brier-2 & 0.0670 & 0.9569 & 0.0146 & 0.0524 & 0.0762 & -0.0710 \\
Log-\(\frac{1}{\ln 202}\) & 0.2023 & 0.7366 & 0.0307 & 0.1716 & 0.1584 & -0.1584 \\
\midrule
Correctness-only & 0.1692 & 0.4419 & -0.6304 & 0.7996 & 0.8105 & -0.8105 \\
Overconfidence-1000 & 0.2085 & 0.7150 & -0.2126 & 0.4210 & 0.5182 & -0.5182 \\
Overconfidence-4 & 0.2156 & 0.7083 & 0.0445 & 0.1711 & 0.1349 & -0.1349 \\
Overconfidence-1 & 0.1780 & 0.8039 & 0.0356 & 0.1423 & 0.1195 & -0.0957 \\
Brier-1 & 0.2104 & 0.7228 & 0.0601 & 0.1503 & 0.0610 & -0.0596 \\
Underconfidence-1 & 0.1706 & 0.8171 & 0.0415 & 0.1291 & 0.1151 & -0.1151 \\
Underconfidence-4 & 0.2011 & 0.7888 & 0.0662 & 0.1349 & 0.0673 & -0.0591 \\
Brier-log Hybrid & 0.1782 & 0.7863 & 0.0382 & 0.1400 & 0.0877 & -0.0751 \\
\midrule
\end{tabular}
\end{table}

\begin{table}
\centering
\caption{Gemma 4 (E2B) Instruct abstention statistics when trained with HotpotQA-Modified dataset, with GPT-OSS-20B as the LLM judge for determining abstention. Unans. stands for unanswerability, Other refers to LLM abstained without implying unanswerability, while Total refers to total abstention rate of the LLM.}
\label{table:hotpotqa-modified-gemma4-e2b-it-abstention-gpt-oss-20b}
\resizebox{\textwidth}{!}{
\begin{tabular}{l c c c c c c c c c c c c}
\toprule
\multirow{2.4}{*}{Model/Reward Scheme} & \multicolumn{3}{c}{Easy} & \multicolumn{3}{c}{Medium} & \multicolumn{3}{c}{Hard} & \multicolumn{3}{c}{Overall} \\
\cmidrule(lr){2-4} \cmidrule(lr){5-7} \cmidrule(lr){8-10} \cmidrule(lr){11-13}
 & Unans. & Other & Total & Unans. & Other & Total & Unans. & Other & Total & Unans. & Other & Total \\
\midrule
Base (JSON) & 0.0131 & 0.0026 & 0.0156 & 0.2093 & 0.0226 & 0.2319 & 0.3899 & 0.0270 & 0.4169 & 0.2043 & 0.0173 & 0.2216 \\
\midrule
Log Loss & 0.9219 & 0.0523 & 0.9742 & 0.9184 & 0.0598 & 0.9782 & 0.9191 & 0.0576 & 0.9767 & 0.9198 & 0.0565 & 0.9763 \\
Brier Score & 0.9557 & 0.0432 & 0.9989 & 0.9483 & 0.0506 & 0.9989 & 0.9397 & 0.0583 & 0.9981 & 0.9479 & 0.0507 & 0.9986 \\
Log-1 & 0.0002 & 0.0023 & 0.0025 & 0.0057 & 0.0244 & 0.0301 & 0.0221 & 0.0693 & 0.0914 & 0.0094 & 0.0321 & 0.0415 \\
Brier-2 & 0.0989 & 0.0109 & 0.1098 & 0.4523 & 0.0470 & 0.4993 & 0.6120 & 0.0607 & 0.6727 & 0.3874 & 0.0395 & 0.4269 \\
Log-\(\frac{1}{\ln 202}\) & 0.0001 & 0.0013 & 0.0013 & 0.0013 & 0.0094 & 0.0108 & 0.0043 & 0.0237 & 0.0279 & 0.0019 & 0.0115 & 0.0134 \\
\midrule
Correctness-only & 0.0000 & 0.0015 & 0.0016 & 0.0002 & 0.0029 & 0.0031 & 0.0001 & 0.0028 & 0.0029 & 0.0001 & 0.0024 & 0.0025 \\
Overconfidence-1000 & 0.0001 & 0.0007 & 0.0008 & 0.0003 & 0.0015 & 0.0018 & 0.0002 & 0.0021 & 0.0023 & 0.0002 & 0.0014 & 0.0016 \\
Overconfidence-4 & 0.0001 & 0.0033 & 0.0034 & 0.0002 & 0.0021 & 0.0023 & 0.0007 & 0.0036 & 0.0043 & 0.0003 & 0.0030 & 0.0033 \\
Overconfidence-1 & 0.0015 & 0.0204 & 0.0219 & 0.0206 & 0.2013 & 0.2219 & 0.0389 & 0.3531 & 0.3920 & 0.0204 & 0.1917 & 0.2121 \\
Brier-1 & 0.0003 & 0.0015 & 0.0017 & 0.0027 & 0.0115 & 0.0142 & 0.0071 & 0.0256 & 0.0326 & 0.0033 & 0.0129 & 0.0162 \\
Underconfidence-1 & 0.0088 & 0.0042 & 0.0130 & 0.1157 & 0.0393 & 0.1550 & 0.2456 & 0.0704 & 0.3160 & 0.1236 & 0.0380 & 0.1616 \\
Underconfidence-4 & 0.0007 & 0.0016 & 0.0023 & 0.0141 & 0.0145 & 0.0286 & 0.0391 & 0.0329 & 0.0720 & 0.0180 & 0.0164 & 0.0344 \\
Brier-log Hybrid & 0.0011 & 0.0122 & 0.0134 & 0.0159 & 0.1251 & 0.1410 & 0.0326 & 0.2313 & 0.2639 & 0.0166 & 0.1230 & 0.1396 \\
\midrule
\end{tabular}
}
\end{table}

\begin{table}
\centering
\caption{Gemma 4 (E2B) Instruct abstention statistics when trained with HotpotQA-Modified dataset, with Qwen 3 (30B-A3B) FP8 as the LLM judge for determining abstention. Unans. stands for unanswerability, Other refers to LLM abstained without implying unanswerability, while Total refers to total abstention rate of the LLM.}
\label{table:hotpotqa-modified-gemma4-e2b-it-abstention-qwen3-30b-a3b-fp8}
\resizebox{\textwidth}{!}{
\begin{tabular}{l c c c c c c c c c c c c}
\toprule
\multirow{2.4}{*}{Model/Reward Scheme} & \multicolumn{3}{c}{Easy} & \multicolumn{3}{c}{Medium} & \multicolumn{3}{c}{Hard} & \multicolumn{3}{c}{Overall} \\
\cmidrule(lr){2-4} \cmidrule(lr){5-7} \cmidrule(lr){8-10} \cmidrule(lr){11-13}
 & Unans. & Other & Total & Unans. & Other & Total & Unans. & Other & Total & Unans. & Other & Total \\
\midrule
Base (JSON) & 0.0264 & 0.0000 & 0.0264 & 0.3441 & 0.0001 & 0.3442 & 0.5902 & 0.0001 & 0.5903 & 0.3204 & 0.0001 & 0.3204 \\
\midrule
Log Loss & 0.9954 & 0.0004 & 0.9959 & 0.9983 & 0.0002 & 0.9985 & 0.9985 & 0.0002 & 0.9986 & 0.9974 & 0.0002 & 0.9976 \\
Brier Score & 0.9999 & 0.0001 & 1.0000 & 0.9999 & 0.0001 & 1.0000 & 1.0000 & 0.0000 & 1.0000 & 0.9999 & 0.0001 & 1.0000 \\
Log-1 & 0.0008 & 0.0003 & 0.0011 & 0.0317 & 0.0032 & 0.0349 & 0.0995 & 0.0087 & 0.1082 & 0.0442 & 0.0041 & 0.0483 \\
Brier-2 & 0.1310 & 0.0004 & 0.1314 & 0.6192 & 0.0004 & 0.6197 & 0.8369 & 0.0003 & 0.8372 & 0.5285 & 0.0004 & 0.5289 \\
Log-\(\frac{1}{\ln 202}\) & 0.0004 & 0.0001 & 0.0005 & 0.0112 & 0.0016 & 0.0129 & 0.0285 & 0.0046 & 0.0331 & 0.0134 & 0.0021 & 0.0155 \\
\midrule
Correctness-only & 0.0001 & 0.0006 & 0.0007 & 0.0031 & 0.0008 & 0.0039 & 0.0056 & 0.0004 & 0.0060 & 0.0029 & 0.0006 & 0.0035 \\
Overconfidence-1000 & 0.0001 & 0.0000 & 0.0001 & 0.0031 & 0.0000 & 0.0031 & 0.0060 & 0.0000 & 0.0060 & 0.0031 & 0.0000 & 0.0031 \\
Overconfidence-4 & 0.0001 & 0.0001 & 0.0002 & 0.0032 & 0.0001 & 0.0033 & 0.0073 & 0.0002 & 0.0075 & 0.0035 & 0.0001 & 0.0037 \\
Overconfidence-1 & 0.0109 & 0.0062 & 0.0171 & 0.1555 & 0.0437 & 0.1992 & 0.2785 & 0.0812 & 0.3597 & 0.1484 & 0.0437 & 0.1921 \\
Brier-1 & 0.0012 & 0.0001 & 0.0013 & 0.0204 & 0.0012 & 0.0216 & 0.0529 & 0.0028 & 0.0556 & 0.0249 & 0.0014 & 0.0263 \\
Underconfidence-1 & 0.0127 & 0.0010 & 0.0137 & 0.1740 & 0.0074 & 0.1814 & 0.3636 & 0.0128 & 0.3763 & 0.1837 & 0.0071 & 0.1908 \\
Underconfidence-4 & 0.0018 & 0.0002 & 0.0019 & 0.0368 & 0.0017 & 0.0385 & 0.0996 & 0.0045 & 0.1041 & 0.0462 & 0.0021 & 0.0484 \\
Brier-log Hybrid & 0.0073 & 0.0034 & 0.0106 & 0.1077 & 0.0237 & 0.1314 & 0.2110 & 0.0407 & 0.2517 & 0.1088 & 0.0226 & 0.1314 \\
\midrule
\end{tabular}
}
\end{table}

\clearpage

\subsection{Invalid Formatting Rates}\label{appendix:invalid-formatting-rates}

For Qwen 2.5 (3B) Instruct, the base (JSON) and SFT models have invalid format rates of under 0.03, while the RL models have invalid format rates of under 0.01. However, the base (XML) model has formatting error rates of between 0.05 and 0.20, with higher error rates observed in DeepMath-103K and BigMath, which is worse than the base (JSON) model. On the other hand, XML output is easier to parse during the reinforcement learning stage. This justifies our decision to finetune Qwen 2.5 (3B) Instruct to output in XML, using the structured format in JSON to read the initial answers for formatting into XML.

For Qwen 3 (4B), the invalid formatting rates are below 0.04 for all reward schemes and the base model for HotpotQA-Modified. However, the invalid formatting rates are higher in DeepMath-103K. While most reward schemes in DeepMath-103K recorded an invalid formatting rate of under 0.2, Correctness-only, Underconfidence-1 and Overconfidence-1 reward schemes recorded an invalid formatting rate of above 0.5 and recorded the most anomalous accuracy drops and calibration biases. This is despite formatting being at least 90\% obeyed at the end of the training. As a format enforcer was used for this LLM during evaluation, a possible reason is that the length of the output was too long and was truncated and that the LLM is struggling to generalize beyond the training data. Another possible reason is that the LLM has been RL trained with a different temperature to encourage exploration, and evaluated on a lower temperature, but due to significant changes in model parameters, the temperatures that work well in the base model no longer work well after training. We leave the investigation on the underlying causes to future work.

For Gemma 4 (E2B) Instruct, the invalid formatting rates are below 0.001 after attempting to repair broken JSON for all reward schemes and the base model for HotpotQA-Modified.

The LLM abstention judge output an invalid format less than 0.0002 of the time for each dataset and each reward scheme examined, including the base model. Therefore, invalid formats are unlikely to significantly affect the abstention statistics.

\section{Case study of selective confidence reward hacking}\label{appendix:reward-hacking-examples}

This appendix section shows a case study of selective confidence reward hacking in Brier-2 reward scheme for the HotpotQA-Modified dataset. This was recorded during the 500th and final RL training step of Qwen 2.5 (3B) Instruct. Table \ref{table:reward-hacking-responses} shows the responses of a cherry-picked sample of 3 questions in the final RL training batch of the LLM but representative of the scenarios on how the LLM may respond when selective confidence reward hacking is present. Note that the exact question is irrelevant because the main focus is to show how the LLM chooses not to provide an educated guess to maximize its chances of getting the correct answer.

\begin{table}[h]
\centering

\caption{Table of responses to three example questions in the final RL training batch of Qwen 2.5 (3B) Instruct for Brier-2 in the HotpotQA-Modified dataset. Con. stands for the corresponding confidence level of the response.}
\begin{tabular}{l c l c l c}    
\toprule
\multicolumn{2}{c}{Example Question A} & \multicolumn{2}{c}{Example Question B} & \multicolumn{2}{c}{Example Question C} \\
\cmidrule(lr){1-2} \cmidrule(lr){3-4} \cmidrule(lr){5-6}
Response & Con. & Response & Con. & Response & Con. \\ \midrule
W. Chump \& Sons Limited & 70 & Unknown & 0 & Unknown & 0 \\
W. Chump \& Sons Limited & 60 & No information & 0 & Not provided & 0 \\
W. Chump \& Sons Limited & 70 & Hog & 60 & Unknown & 0 \\
W. Chump \& Sons & 70 & Unknown & 0 & Not enough information & 0 \\
W. Chump \& Sons Limited & 70 & Hog & 70 & Unknown & 0 \\
W. Chump \& Sons & 70 & Unknown & 0 & Unknown & 0 \\
W. Chump \& Sons Limited & 70 & unspecified & 0 & Unknown & 0 \\
W. Chump \& Sons Limited & 70 & Unknown & 0 & Not provided & 0 \\ \bottomrule
\label{table:reward-hacking-responses}
\end{tabular}
\end{table}

The LLM chose to answer Example Question A as it was moderately confident of the correct answer. The confidence values have largely converged to 70, showing a high degree of internal consistency of its confidence.

The LLM has selectively given up on answering Example Question B in 6 out of the 8 responses. This shows that the decision to be confident about an incorrect answer may not be deterministic.

In the third example question, the LLM has decided to give up on Example Question C as there was insufficient information to answer the question. This occurs in HotpotQA-Modified as relevant sources can be removed during dataset preprocessing.

% These example questions correspond to the first, second and fourth questions of step 500 of our training run respectively.

\section{Other Experimental Observations}\label{appendix:other-experimental-observations}

In this appendix section, we discuss other observed experimental trends, namely the presence of LLM abstention in confidence reward hacking by other LLMs, the accuracy-calibration tradeoffs and the speed of confidence calibration. An investigation of the underlying causes of the experimentally observed phenomena will be left to further research.

\subsection{Presence of Confidence Reward Hacking in other LLMs}\label{appendix:presence-of-confidence-reward-hacking}

\begin{table}[t]
\centering
\caption{Accuracy statistics and abstention statistics, split by dataset difficulty, when Qwen 3 (4B) and GPT-OSS-20B are trained with HotpotQA-Modified. LLM abstention is judged by GPT-OSS-20B. The models are subdivided into three categories, namely the base model, RL models trained with hackable confidence reward schemes and RL models trained with reward schemes in \(R_{spec}\).}
\label{table:results-hotpotqa-modified-two-llms}
\smallskip
\resizebox{\textwidth}{!}{
\begin{tabular}{l c c c c c c c c c c c c}
\toprule
\multirow{3}{*}{Model/Reward Scheme} & \multicolumn{6}{c}{Qwen 3 (4B)} & \multicolumn{6}{c}{Gemma 4 (E2B) Instruct} \\ 
\cmidrule(lr){2-7} \cmidrule(lr){8-13}
 & \multicolumn{3}{c}{Accuracy} & \multicolumn{3}{c}{Abstention Rates} & \multicolumn{3}{c}{Accuracy} & \multicolumn{3}{c}{Abstention Rates} \\
\cmidrule(lr){2-4} \cmidrule(lr){5-7} \cmidrule(lr){8-10} \cmidrule(lr){11-13}
 & easy & medium & hard & easy & medium & hard & easy & medium & hard & easy & medium & hard \\
\midrule
Base Model & 0.5394 & 0.2905 & 0.1137 & 0.0077 & 0.1239 & 0.2661 & 0.5491 & 0.2603 & 0.0929 & 0.0156 & 0.2319 & 0.4169 \\ \midrule
Log Loss & 0.0018 & 0.0000 & 0.0001 & 0.9484 & 0.9522 & 0.9546 & 0.0000 & 0.0000 & 0.0000 & 0.9742 & 0.9782 & 0.9767 \\
Brier Score & 0.0250 & 0.0035 & 0.0013 & 0.9281 & 0.9522 & 0.9546 & 0.0000 & 0.0000 & 0.0000 & 0.9989 & 0.9989 & 0.9981 \\
Log-1 & 0.6279 & 0.2634 & 0.1053 & 0.0683 & 0.4547 & 0.6641 & 0.6649 & 0.4148 & 0.1813 & 0.0025 & 0.0301 & 0.0914 \\
Brier-2 & 0.6486 & 0.3217 & 0.1384 & 0.0325 & 0.3188 & 0.5257 & 0.6081 & 0.2187 & 0.0670 & 0.1098 & 0.4993 & 0.6727 \\
Log-\(\frac{1}{\ln 202}\) & 0.6611 & 0.4280 & 0.2085 & 0.0025 & 0.0322 & 0.0872 & 0.6797 & 0.4361 & 0.2023 & 0.0013 & 0.0108 & 0.0279 \\ \midrule
Correctness-only & 0.6657 & 0.4343 & 0.2132 & 0.0017 & 0.0205 & 0.0600 & 0.6354 & 0.3894 & 0.1692 & 0.0016 & 0.0031 & 0.0029 \\
Overconfidence-1000 & 0.6682 & 0.4233 & 0.2066 & 0.0031 & 0.0439 & 0.1112 & 0.6557 & 0.4291 & 0.2085 & 0.0008 & 0.0018 & 0.0023 \\
Overconfidence-4 & 0.6673 & 0.4247 & 0.2043 & 0.0028 & 0.0469 & 0.1235 & 0.6434 & 0.4326 & 0.2156 & 0.0034 & 0.0023 & 0.0043 \\
Overconfidence-1 & 0.6648 & 0.4051 & 0.1913 & 0.0064 & 0.1014 & 0.2213 & 0.6841 & 0.3854 & 0.1780 & 0.0219 & 0.2219 & 0.3920 \\
Brier-1 & 0.6638 & 0.3976 & 0.1887 & 0.0076 & 0.1171 & 0.2505 & 0.6452 & 0.4013 & 0.2104 & 0.0017 & 0.0142 & 0.0326 \\
Underconfidence-1 & 0.6654 & 0.4015 & 0.1909 & 0.0083 & 0.1128 & 0.2440 & 0.6762 & 0.3931 & 0.1706 & 0.0130 & 0.1550 & 0.3160 \\
Underconfidence-4 & 0.6644 & 0.4017 & 0.1865 & 0.0075 & 0.1080 & 0.2363 & 0.6737 & 0.4315 & 0.2011 & 0.0023 & 0.0286 & 0.0720 \\
Brier-log Hybrid & 0.6594 & 0.3907 & 0.1813 & 0.0080 & 0.1246 & 0.2603 & 0.6527 & 0.3892 & 0.1782 & 0.0134 & 0.1410 & 0.2639 \\
\bottomrule
\end{tabular}
}
\end{table}

Table \ref{table:results-hotpotqa-modified-two-llms} shows the accuracy and abstention statistics broken down by question difficulty when Qwen 3 (4B) and Gemma 4 (E2B) Instruct are trained with HotpotQA-Modified and under different reward schemes. LLM abstention was judged by GPT-OSS-20B \citep{openai2025gptoss120bgptoss20bmodel}. We consider LLM abstention statistics because in some benchmarks such as SQuAD 2.0 \citep{rajpurkar-etal-2016-squad, rajpurkar-etal-2018-know}, the LLM is supposed to abstain when unconfident of the answer. Note that in our experiments, an abstention is considered a wrong answer as the ground truths do not show abstention.

Observe that as the accuracy of the LLMs decrease, the abstention rates of the LLMs generally increases, which shows that the LLM mostly use abstention as a means to answer incorrectly for confidence reward hacking. A deeper investigation on why the LLM tends to abstain when unsure instead of attempting to predict the answer will be left to future work.

The results largely match the theory. Log Loss and Brier Score had near-zero accuracies and abstained from answering almost all the questions. Log-1 and Brier-2, being hackable confidence, showed accuracy drops apart from Log-1 when Gemma 4 (E2B) Instruct is trained. Log-\(\frac{1}{\ln 202}\), being strictly non-hackable confidence on \((\frac{1}{203}, 1)\), did not show reward hacking, as the lowest representable confidence in our experiments is \(\frac{1}{202}\). Weakly non-hackable confidence reward schemes in our overconfidence to underconfidence spectrum had generally decreased abstention rates, in line with the theory that such reward schemes do not incentivize answering incorrectly to obtain a greater reward, with abstention being one example of an incorrect answer.

We also repeated our abstention judging on Qwen 3 (30B-A3B) FP8 \citep{qwen3technicalreport} on the same LLM experimental output, and the full results can be found in Appendix \ref{appendix:full-results}.

\subsection{Accuracy-Confidence Calibration Tradeoffs}

In this subsection, we demonstrate using Figures \ref{fig:higher-is-better-line-chart} and \ref{fig:closer-to-0-is-better-line-chart} that reward schemes nearer to the overconfident end of our proposed overconfidence to underconfidence spectrum prioritize accuracy while reward schemes nearer to the underconfident end prioritize confidence calibration.

\begin{figure}[t]
    \centering
    \includegraphics[width=1\linewidth]{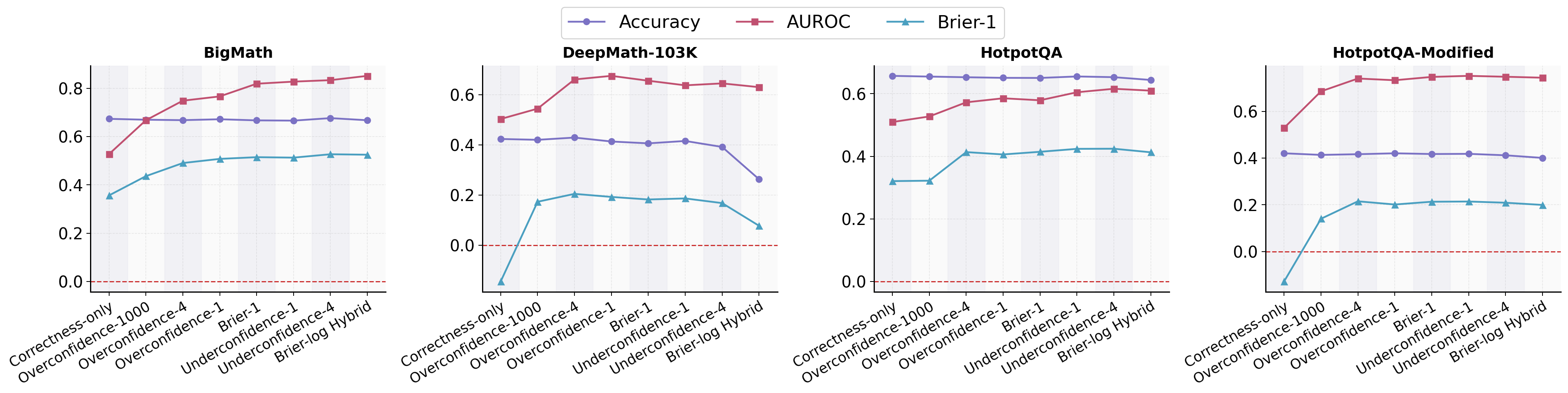}
    \caption{Accuracy, AUROC and Brier-1 performance metrics for Qwen 2.5 (3B) Instruct when trained with BigMath, DeepMath, HotpotQA and HotpotQA-Modified datasets (from left to right) for weakly non-hackable confidence reward schemes ordered in the spectrum from the overconfident end to the underconfident end (from left to right). For all three performance metrics, higher is better.}
    \label{fig:higher-is-better-line-chart}
\end{figure}

Figure \ref{fig:higher-is-better-line-chart} shows the accuracy, AUROC and Brier-1 reward metric for the reward schemes in \(R_{spec}\), as defined in Section \ref{section:reward-schemes-used}, when ordered from the overconfident end to the underconfident end of the spectrum (from left to right of plot). As one moves towards the underconfident end of the spectrum, the accuracy is either maintained or generally decreases. The decrease in accuracy is most notable in DeepMath-103K and is the least pronounced in BigMath. In contrast, AUROC and Brier-1 initially improve, but deteriorate or level off as one moves towards the underconfident end of the spectrum depending on the dataset. 

We believe that the general decrease in accuracy could be linked to a general increase in abstention rates in the model as one moves towards the underconfident end of the spectrum, as shown in Table \ref{table:results-hotpotqa-modified-two-llms} in Section \ref{section:results-rq1}. The increase in abstention rates is not consistent across datasets, for instance when Qwen 3 (4B) is trained with DeepMath-103K (refer to Tables \ref{table:deepmath-103k-qwen3-4b-overall-results}, \ref{table:deepmath-103k-qwen3-4b-abstention-gpt-oss-20b} and \ref{table:deepmath-103k-qwen3-4b-abstention-qwen3-30b-a3b-fp8} for the relevant full results), where confidence reward hacking is not observed to occur for the Log-1 and Brier-2 reward schemes and the final accuracies are high. This could be due to lower confidence responses from the LLM being more likely abstentions than attempted responses. More investigation is required to understand the underlying causes.

\begin{figure}[t]
    \centering
    \includegraphics[width=1\linewidth]{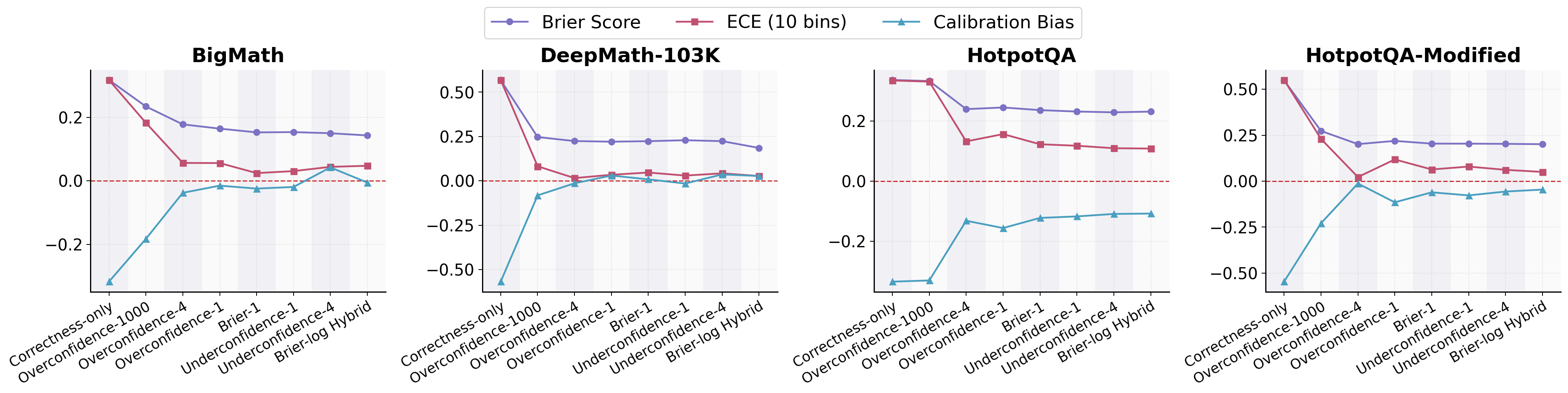}
        \vspace{0.05in}
    \caption{Brier score, ECE (10 bins) and Calibration bias statistics for Qwen 2.5 (3B) Instruct when trained with BigMath, DeepMath-103K, HotpotQA and HotpotQA-Modified datasets (from left to right) for weakly non-hackable confidence reward schemes ordered in the spectrum from the overconfident end to the underconfident end (from left to right). For all three performance metrics, closer to 0 (red dotted line) is better.}
    \label{fig:closer-to-0-is-better-line-chart}
\end{figure}

Figure \ref{fig:closer-to-0-is-better-line-chart} shows the Brier score, ECE (10 bins) and calibration bias for the reward schemes in \(R_{spec}\) when ordered from the overconfident end of the spectrum to the underconfident end (from left to right of plot). As one moves towards the underconfident end of the spectrum, both Brier score and ECE generally improve before leveling off around Overconfidence-4. 

\subsection{Speed of Confidence Calibration}\label{appendix:speed-of-confidence-calibration}

In this section, we demonstrate that the reward scheme influences the training speed of the LLM in both accuracy and confidence calibration.

\begin{figure}[h]
    \centering
    \includegraphics[width=1\linewidth]{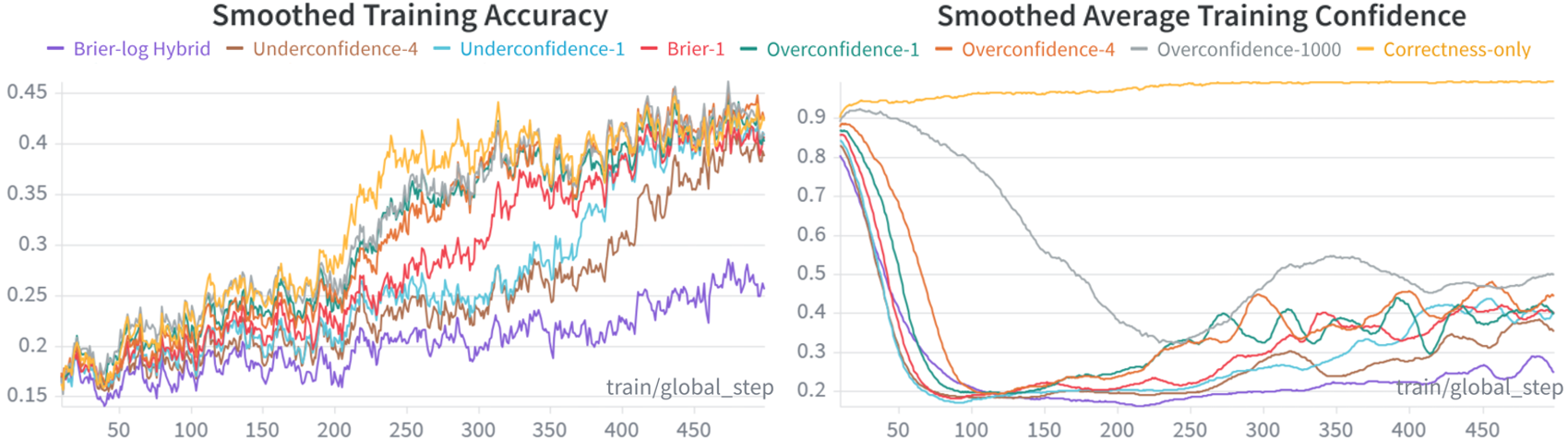}

    \caption{Training accuracy (left) and training confidence (right) statistics for Qwen 2.5 (3B) Instruct when trained with DeepMath-103K from Step 10 to Step 499 (0-indexed), smoothed using an exponential moving average \citep[p. 15]{brown1956exponential}, both initialized at 0.5 with smoothing factor of 0.2.}
    \label{fig:deepmath-103k-training}
\end{figure}

Figure \ref{fig:deepmath-103k-training} shows the training accuracy and training confidence of Qwen 2.5 (3B) Instruct when trained with different weakly non-hackable confidence reward schemes in DeepMath-103K. The training speed is generally slower towards the underconfident end of the spectrum, while the confidence calibration is generally faster apart from a significant anomaly in Brier-log Hybrid. Speed of confidence calibration is important in resource-limited applications where confidence calibration is prioritized over accuracy.

While the initial LLM, i.e. just after the SFT stage, is overconfident, its AUROC is above 0.5 (see Appendix \ref{appendix:full-results} for full results), suggesting that answers with lower confidence are less likely to be correct. With higher confidence penalty at high confidences towards the underconfident end of the spectrum, the LLM is incentivized to generate responses with lower confidence as these responses have higher expected reward. Such answers are less likely to be correct, leaving the LLM with fewer correct responses to learn from, which slows down the improvements in training accuracy.

The general trend in the confidence calibration speed is likely due to the higher miscalibration penalties towards the underconfident end of the spectrum when initial confidence is high. The anomaly is likely due to relatively low miscalibration penalties at low confidences and low probability of answering correctly in DeepMath-103K when using the Brier-log Hybrid reward scheme. Further investigation of underlying causes of the link between the miscalibration penalty and the speed of confidence calibration will be left to future work.

\section{Experimental Anomalies}

This section highlights experimental anomalies that researchers may face while attempting to reproduce this work. Currently, two have been identified, i.e. for Qwen 2.5 (3B) Instruct, the confidence may converge to approximately 0.1 while performing RL training on confidence calibration, and for some LLMs, training collapse may occur due to instabilities in RL.

\subsection{Confidence Converges to Approximately 0.1}

As shown in Figure \ref{figure:anomalous-runs}, while training Qwen 2.5 (3B) Instruct on the HotpotQA-Modified dataset using RL on the reward schemes Correctness-only, Brier-1 and Overconfidence-1, the average LLM training confidence appeared to converge to approximately 0.1 mid-run, and did not significantly change for hundreds of steps despite improving training accuracy, suggesting poor confidence calibration due to the LLM getting stuck in a suboptimal local maximum. These runs were not taken for evaluation in the main paper.

We reran the anomalous RL experiments starting from the same fine-tuned model at the supervised fine-tuning stage. Upon rerun, the anomalies mostly did not resurface, and we took these runs for the evaluation in the main paper. As the problem was not reproducible across reruns, this was likely due to the stochasticity of the RL algorithm and issues with the reliability of the trust region in RL. 

\begin{figure}[h]
    \centering
    \includegraphics[width=1\linewidth]{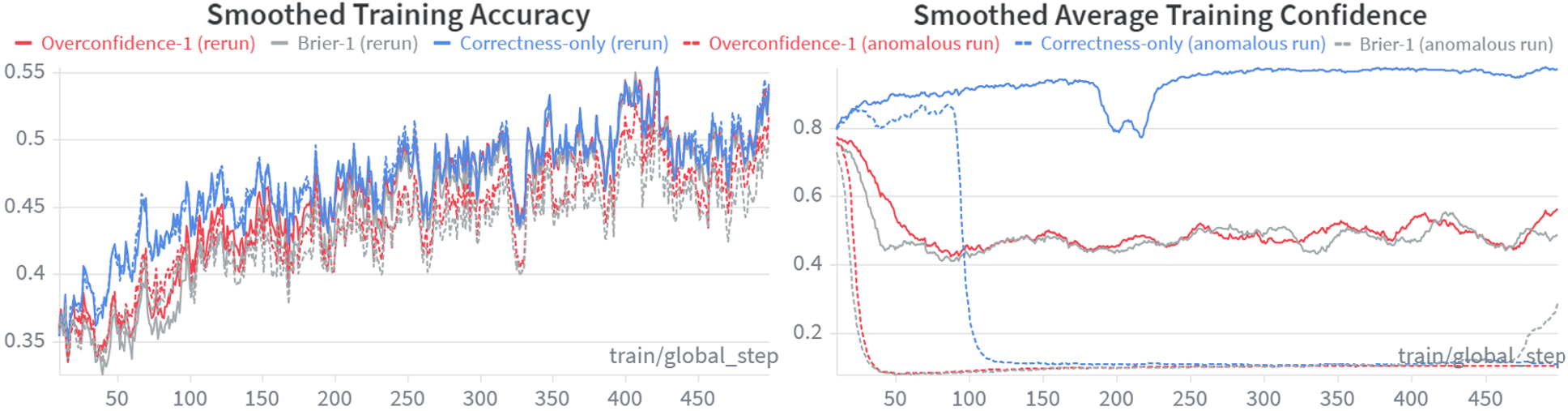}
    \caption{Training accuracy (left) and training confidence (right) statistics for both anomalous runs and their corresponding reruns during RL training Qwen 2.5 (3B) Instruct of the HotpotQA-Modified dataset from Step 10 to Step 499 (0-indexed), smoothed using an exponential moving average \citep[p. 15]{brown1956exponential}, both initialized at 0.5 with smoothing factor of 0.2. Dotted lines represent the anomalous runs while solid lines of the same color represent the corresponding reruns.}
    \label{figure:anomalous-runs}
\end{figure}

The LLM has often shown high confidence on its confidence analysis despite providing a confidence of 10 on the scale from 0 to 100. Therefore, the LLM has likely misinterpreted the instruction to output the confidence on a scale from 0 to 100 as a scale from 0 to 10 instead, and the misinterpretation was likely reinforced due to the lower confidence of 10 resulting in a higher expected reward. To discourage this from happening, there should be checks to ensure that the confidence analysis correctly justifies the confidence value provided. We leave the checking to future work.

In the example below, we provide one question and one corresponding LLM response in the final training step of the anomalous experimental run when RL training on the HotpotQA-Modified dataset and using the Brier-1 reward scheme. The response demonstrates LLM confidence in its answer despite indicating a low confidence of 10 on the scale of 0 to 100, as instructed in the system prompt. The system prompt used is the XML system prompt format in Section \ref{appendix:prompt-formats}.

\begin{mdframed}[linecolor=black, backgroundcolor=gray!10, roundcorner=20pt, innertopmargin=5pt, innerbottommargin=5pt] 

\textbf{Question (excluding system prompt)}

Based on the information sources given below and your existing knowledge, answer the following question: The Holy Trinity, an episode of British motoring series "The Grand Tour" was produced by which  independent television production company founded in 2015 in the UK?

Source 1: James May

James Daniel May (born 16 January 1963) is an English television presenter and journalist.  He is best known as a co-presenter of the motoring programme "Top Gear" alongside Jeremy Clarkson and Richard Hammond from 2003 until 2015.  As of 2016 he is a director of the production company W. Chump \& Sons (founded July 2015) and is also a co-presenter in the television series "The Grand Tour" for Amazon Video as an exclusive for Amazon Prime customers, alongside his former "Top Gear" colleagues, Clarkson and Hammond, as well as former producer Andy Wilman.

Source 2: G-Unit Films and Television Inc.

G-Unit Films and Television Inc. is an American film and television production company founded by rapper 50 Cent and Interscope in 2003.  In 2008, 50 Cent stated in an interview that he has created his own independent film production company with Randall Emmett called Cheetah Vision, technically scrapping G-Unit Films.  In 2010, Jackson revived G-Unit Films, renaming the company to G-Unit Films and Television Inc.  The company has joint ventures with Will Packer’s production company Will Packer Productions and Universal Television.  In over 18 months, Jackson has sold projects to six different networks.  Among them was "Power", a STARZ drama in which he not only co-stars but also serves as co-creator and executive producer.  “Power” debuted in June 2014 and was renewed for a second season after one episode.  “Power’s” August 2 season finale garnered the high ratings through the season, more than doubling the premiere and it generated 71\% of the African-American viewership of any scripted premium series since 2006.  Jackson serves as a co‐star, co-creator and executive television producer of the STARZ network drama where he signed a 2 year contract with representation coming from the Agency for the Performing Arts.  Ratings have been a success for Starz.  with the second season premiere being the highest-ever season with 1.43 million people tuning in live.  Jackson also serves as an executive television producer for "Dream School" for SundanceTV, a series that follows fifteen high school dropouts as they are taught by a series of celebrity "teachers".

Source 3: WestWind Pictures

WestWind Pictures Ltd. is an independent television production company founded in 1989 in Regina, Saskatchewan.  The company, now based in Toronto, Ontario, is co-owned by CEO Mary Darling and President Clark Donnelly.  WestWind currently has programs airing in over 80 countries around the world.  The company is currently involved in scripted television series, documentaries, lifestyle programming and feature films.

Source 4: W. Chump and Sons

W. Chump \& Sons Limited is an independent television production company founded in July 2015 in the United Kingdom.  The company is headquartered on Power Road , Chiswick, London, and owned by Jeremy Clarkson, Richard Hammond, James May and Andy Wilman.  The directors of the company are the former presenters and executive producer of the MPH Show and "Top Gear" Live exhibition tours, as well as the former "Top Gear" series produced for BBC Television from October 2002 until March 2015.  s of 2016 the company was engaged in producing "The Grand Tour," a motoring television series that debuted on Amazon Video on 18 November 2016.

Source 5: The Automobile Association

AA plc (The AA, originally The Automobile Association) is a British motoring association founded in 1905, which currently provides car insurance, driving lessons, breakdown cover, loans, motoring advice, road maps and other services.  The association demutualised in 1999, to become a private limited company, and in 2002, the AA Motoring Trust was created to continue its public interest and road safety activities.  The company is listed on the London Stock Exchange, and is a constituent of the FTSE 250 Index.

Source 6: Kapital Entertainment

Kapital Entertainment is an American independent television production company founded by Aaron Kaplan in 2009.  Prior to founding Kapital, Kaplan was the worldwide head of scripted television at William Morris Agency to which he joined in 1991.  He exited the company in 2009 after 18 years when WMA was going through the regulatory process to finalize their merger with Endeavor to form WME Entertainment.

Source 7: The Grand Tour (TV series)

The Grand Tour is a British motoring television series for Amazon Video presented by Jeremy Clarkson, Richard Hammond and James May and produced by Andy Wilman.

Source 8: Desilu Productions

Desilu Productions ( ) was an American production company founded and co-owned by husband and wife Desi Arnaz and Lucille Ball, best known for shows such as "I Love Lucy", "", and "The Untouchables".  Until 1962, Desilu was the second-largest independent television production company in the U.S. behind MCA's Revue Productions until MCA bought Universal Pictures, and Desilu became and remained the number-one independent production company until being sold in 1967.  Ball and Arnaz jointly owned the majority stake in Desilu from its inception until 1962, when Ball bought out Arnaz and ran the company by herself for several years.  Ball had succeeded in making Desilu profitable again by 1967, when she sold her shares of Desilu to Gulf+Western for \$17 million (\$ in 2016 dollars).  Gulf+Western then transformed Desilu into the television production arm of Paramount Pictures, rebranding the company as the original Paramount Television.

\nointerlineskip

\rule{\linewidth}{0.5pt}

\nointerlineskip

\textbf{Response}

<reasoning>

The question asks for the independent television production company founded in 2015 in the UK producing The Holy Trinity, an episode of the British motoring series 'The Grand Tour'. 'The Grand Tour' is produced by W. Chump and Sons, which is stated to be an independent television production company founded in July 2015 in the United Kingdom.

</reasoning>

<answer>

W. Chump and Sons

</answer>

<confidence\_analysis>

The answer is directly derived from Source 4, which states that W. Chump \& Sons Limited is an independent television production company founded in July 2015 in the United Kingdom and is the producer of 'The Grand Tour'. All other sources either do not mention 'The Grand Tour' or are unrelated to the production company question.

</confidence\_analysis>

<confidence>

10

</confidence>

\end{mdframed}

\subsection{RL Training Collapse}

In addition to Qwen 2.5 (3B) Instruct, Qwen 3 (4B) and Gemma 4 (E2B) Instruct, we ran our experiments on Llama 3.2 (3B) Instruct and Llama 3.1 (8B) Instruct \citep{grattafiori2024llama3herdmodels} on the BigMath dataset. Unless otherwise stated, the experiment settings of Llama 3.2 (3B) Instruct and Llama 3.1 (8B) Instruct are the same as in Qwen 2.5 (3B) Instruct. For these experiments, the LLM has been observed to initially learn before the collapse occurred. 

\subsubsection{Llama 3.2 (3B) Instruct}

For Llama 3.2 (3B) Instruct, we ran RL on BigMath on 1000 steps and the learning rate lowered to \(10^{-6}\). The other hyperparameters are the same as in Qwen 2.5 (3B) Instruct. We tested the LLM on Brier-1, Log-1, Brier-2 and Brier-log Hybrid reward schemes, of which only Brier-1 and Brier-log Hybrid are weakly non-hackable confidence. The training accuracies and the training losses for each experiment are as shown in Figure \ref{fig:llama3.2-3b-training-collapse}.

\begin{figure}
    \centering
    \includegraphics[width=1\linewidth]{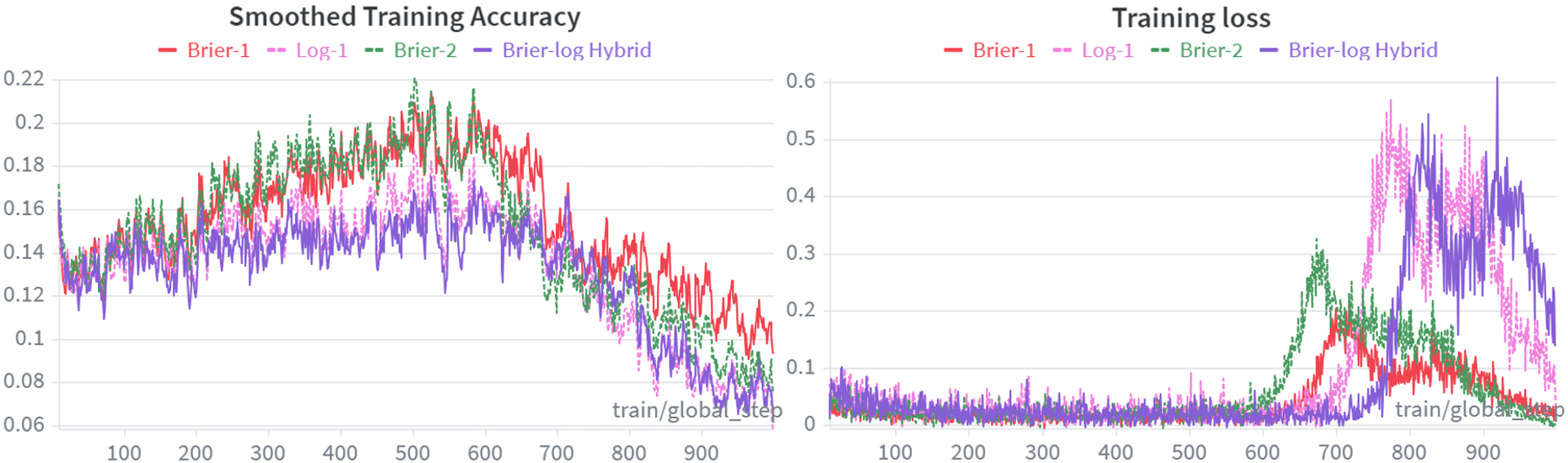}
    \caption{Training accuracy (left) and training loss (right) statistics for BigMath in Llama 3.2 (3B) Instruct from Step 10 to Step 999 (0-indexed). The training accuracy has been smoothed using an exponential moving average \citep[p. 15]{brown1956exponential}, initialized at 0.5 with smoothing factor of 0.2. Dotted lines represent hackable confidence reward schemes.}
    \label{fig:llama3.2-3b-training-collapse}
\end{figure}

Dr GRPO training collapsed, as evidenced by the sudden spike in training loss in Figure \ref{fig:llama3.2-3b-training-collapse}. As this empirically holds independent of the reward scheme used, it is likely a limitation of Dr GRPO RL rather than the reward scheme. We also hypothesize that overly restrictive output token limits are a contributing factor, with further study required to understand the phenomenon. Notably, Brier-2 was the first and Brier-log Hybrid was the last to experience training collapse as evidenced by the training loss in Figure \ref{fig:llama3.2-3b-training-collapse}.

\subsubsection{Llama 3.1 (8B) Instruct}

\begin{figure}
    \centering
    \includegraphics[width=1\linewidth]{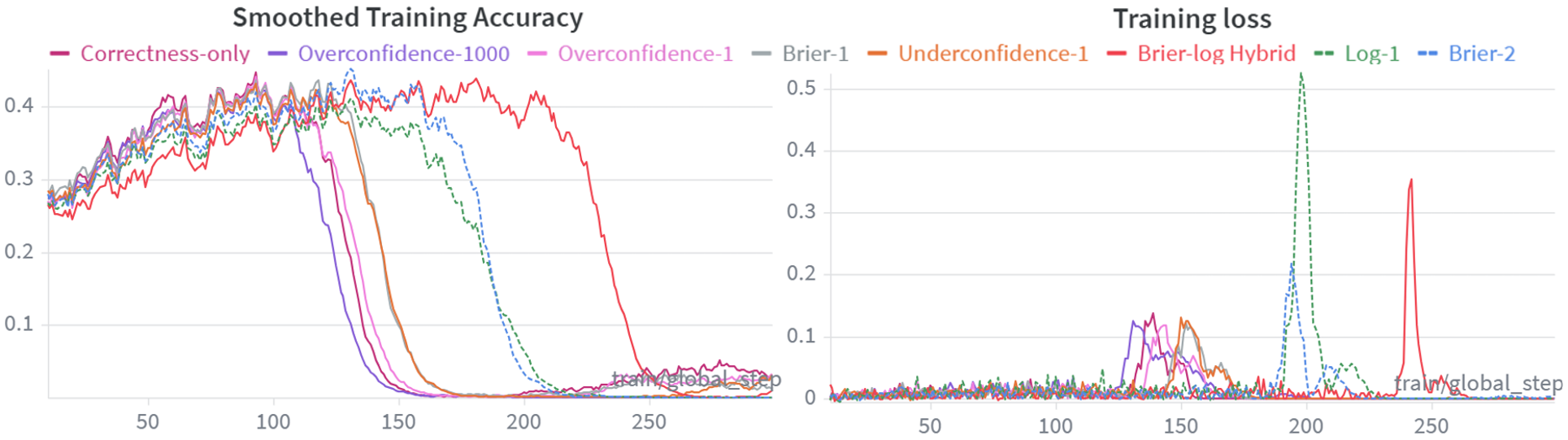}
    \caption{Training accuracy (left) and training loss (right) statistics for BigMath in Llama 3.1 (8B) Instruct from Step 10 to Step 299 (0-indexed). The training accuracy has been smoothed using an exponential moving average \citep[p. 15]{brown1956exponential}, initialized at 0.5 with smoothing factor of 0.2. Dotted lines represent hackable confidence reward schemes.}
    \label{fig:llama3.1-8b-training-collapse}
\end{figure}

We ran Dr GRPO \citep{liu2025understandingr1zeroliketrainingcritical} RL on BigMath with the same hyperparameters as Qwen 2.5 (3B) Instruct except that the maximum gradient norm has been conservatively set to 0.01 to rule out issues regarding too large a gradient. We ran our experiments on Correctness-only, Overconfidence-1000, Overconfidence-1, Brier-1, Underconfidence-1, Brier-log Hybrid, Log-1 and Brier-2 reward schemes. All the reward schemes mentioned are weakly non-hackable confidence except Log-1 and Brier-2. Figure \ref{fig:llama3.1-8b-training-collapse} shows the results of our experiments.

Likewise, the model initially trains well before a sharp decrease to training accuracy, likely due to conflicting signals in the Dr GRPO loss. This occurred regardless of the reward scheme used. The peaks in the training loss occur during the significant decrease in training accuracy.

Similarly, Brier-log Hybrid was the last to suffer the training collapse. However, unlike the experiment in Llama 3.2 (3B) Instruct, the training collapse occurs when the LLM has reached approximately the same reasoning accuracy apart from the hackable confidence reward scheme Log-1, which suggests a learning problem. A more detailed investigation on the phenomenon of training collapse during confidence calibration of RL will be left to future research.

As Brier-log Hybrid is the furthest along the underconfident end of the spectrum, based on the results in Figure \ref{fig:deepmath-103k-training} of Section \ref{section:results-rq2}, Brier-log Hybrid can be expected to take the longest time to train to the approximate accuracy level threshold, after which, RL training collapse would occur. Indeed, as shown in Figure \ref{fig:llama3.1-8b-training-collapse}, Brier-log Hybrid takes the longest time to do so among the examined weakly non-hackable confidence reward schemes.

\subsubsection{Gemma 4 (E2B) Instruct}

\begin{figure}[h]
    \centering
    \includegraphics[width=1\linewidth]{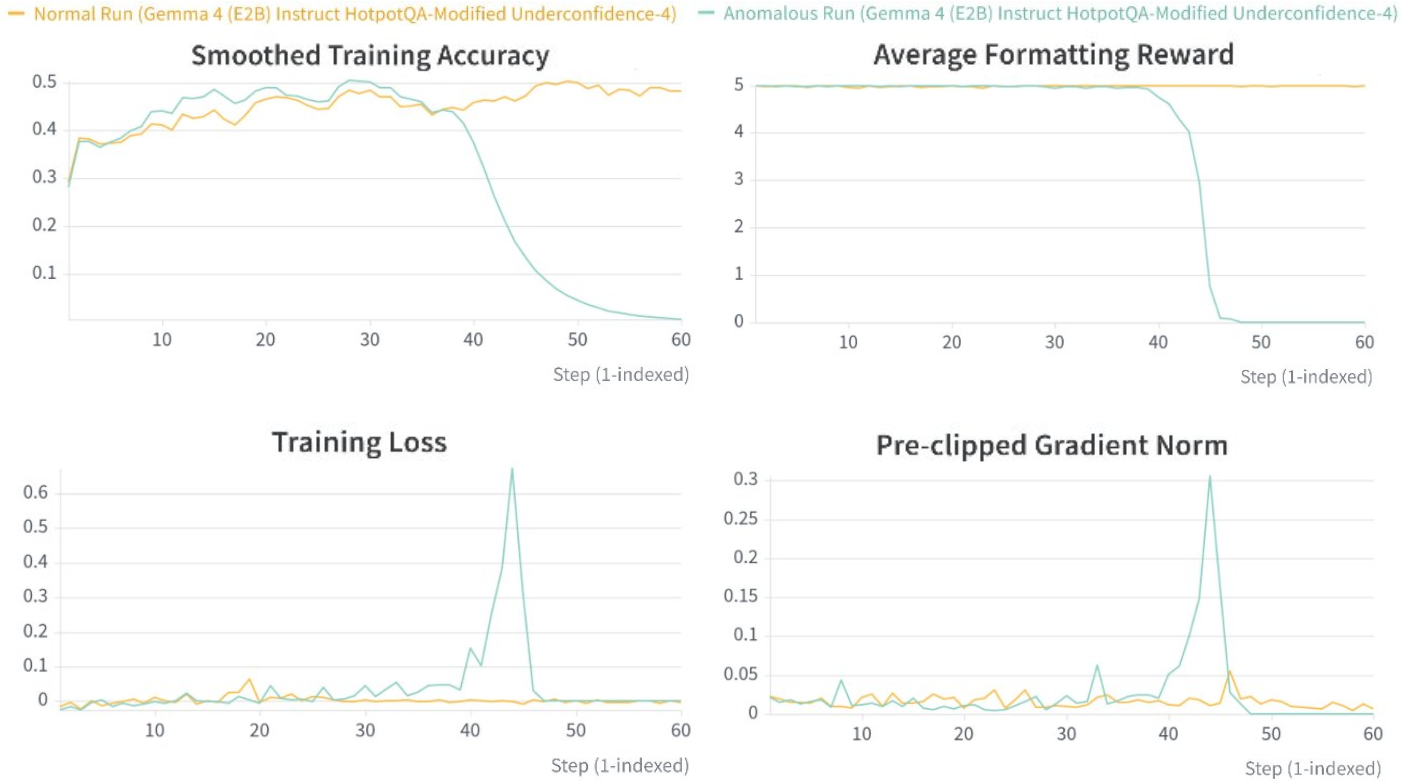}
    \caption{Training accuracy (top left), formatting reward (top right), training loss (bottom left) and pre-clipped gradient (bottom right) statistics for HotpotQA-Modified in Gemma 4 (E2B) Instruct from Step 1 to Step 60 (1-indexed). The training accuracy has been smoothed using the arithmetic mean for the first 5 steps and exponential moving average \citep[p. 15]{brown1956exponential} with smoothing factor of 0.2 thereafter.}
    \label{fig:gemma4-e2b-it-training-collapse}
\end{figure}

Gemma 4 (E2B) Instruct RL training with DAPO \citep{yu2025dapoopensourcellmreinforcement} and HotpotQA-Modified dataset has been observed to occasionally collapse, with all other hyperparameters unrelated to the random seed left unchanged. In extreme cases, the reward can fall to zero. As shown in Figure \ref{fig:gemma4-e2b-it-training-collapse}, during our anomalous attempt RL training using Underconfidence-4 reward scheme, the LLM training accuracy decreased to 0 and there was a spike in training loss and the pre-clipped gradient. Moreover, the formatting reward also decreased to 0. We reran the experiments with the same LLM seed and the same data preprocessing seed, and the training collapse did not happen, suggesting that the incidence of such training collapse is inconsistent. The second attempt was taken for the experimental results.

\endgroup
\end{document}